\documentclass{article}

\usepackage[final,main]{neurips_2026}

\usepackage{graphicx}  
\usepackage{booktabs}
\usepackage{multirow}
\usepackage{booktabs}
\usepackage{graphicx}
\usepackage{xcolor}
\usepackage{caption}
\usepackage[section]{placeins}
\usepackage[table]{xcolor}
\usepackage{arydshln}
\usepackage{url}
\usepackage{xstring}
\usepackage{amsfonts}
\usepackage{amsmath}
\usepackage{amssymb}
\usepackage{upgreek}
\usepackage{booktabs} 
\usepackage{amsthm}
\usepackage{mathtools}
\usepackage{adjustbox}
\mathtoolsset{showonlyrefs}
\usepackage{bm}
\newcommand{\Prob}{\mathcal{P}}

\newcommand{\Id}{\operatorname{id}}
\newcommand{\ProbTwo}{\mathcal{P}_2}

\newcommand{\W}{\operatorname{W}}

\newcommand{\overboldarrow}[1]{\bm{\vbox{\offinterlineskip
  \ialign{##\cr
  \scalebox{1}[0.9]{$\overrightarrow{\hphantom{#1}}$}\cr
  $#1$\cr}}}}

\renewcommand{\d}{\,\mathrm{d}}

\newcommand{\pr}{\mathrm{proj}}
\newcommand{\bproj}{{\rm \bf{Proj}}}

\newcommand{\Perm}{{\operatorname{Perm}}}

\newcommand{\Rd}{\mathbb{R}^d}

\newcommand{\R}{\mathbb{R}}

\renewcommand{\d}{\, \mathrm{d}}

\newcommand{\bmu}{\boldsymbol{\mu}}

\newcommand{\bnu}{\boldsymbol{\nu}}

\newcommand{\bPi}{{\mathbf{\Pi}}}
\newcommand{\bSW}{{\operatorname{\mathbf{SW}}}}

\newcommand{\bW}{{\operatorname{\mathbf{W}}}}

\newcommand{\SW}{\operatorname{SW}}

\newcommand{\supp}{\operatorname{supp}}

\newcommand{\proj}{\pr}  

\newcommand{\mP}{\mathcal{P}}

\usepackage{amsthm}
\usepackage{graphicx}    
\usepackage{wrapfig}

\usepackage{subcaption}  
\DeclareMathOperator*{\argmin}{arg\,min}

\newtheorem{theorem}{Theorem}[section]

\newtheorem{definition}[theorem]{Definition}

\newtheorem{corollary}[theorem]{Corollary}

\newcounter{proproman}

\usepackage{thmtools, thm-restate} 

\usepackage{xcolor}

\newcommand{\revise}[1]{{\color{black}#1}}

\usepackage{hyperref}
\hypersetup{hidelinks,
  pdftitle={Generalized Wasserstein Flow Matching: Nested Couplings for Point Cloud Generation},
  pdfauthor={Moritz Piening, Richard Duong, Gabriele Steidl}}
\usepackage{orcidlink}

\title{\revise{Generalized Wasserstein Flow Matching:\\Nested Couplings for Point Cloud Generation}}
\author{%
  \revise{Moritz Piening$^{1,2}$ \quad Richard Duong$^1$ \quad Gabriele Steidl$^1$} \\
  \revise{$^1$Department of Mathematics, Technische Universität Berlin, Germany} \\
  \revise{$^2$Institute of Computer Science, Faculty of Mathematics and Computer Science} \\
  \revise{University of Göttingen, Germany} \\
  \revise{{\small\{\texttt{piening,duong,steidl}\}\texttt{@math.tu-berlin.de} \quad \texttt{moritz.piening@uni-goettingen.de}}} \\
}

\begin{document}

\maketitle
\renewcommand{\d}{\,\mathrm{d}}

\begin{abstract}
\revise{We formulate flow matching on probability measures over probability measures. Directly from the Wasserstein-on-Wasserstein (WoW) continuity equation, we show that admissible nested outer and inner couplings induce a valid metameasure flow and non-local velocity field. Validity requires neither optimal plans nor a smooth manifold. Practically, this opens up the design space of multisample flow matching: exact Wasserstein plans are distinguished by energy and straightness, while sliced and lazy linear plans yield scalable alternatives.
Point-cloud experiments, including controlled Bures--Wasserstein and Euclidean minibatch baselines, demonstrate the resulting geometry--cost trade-off.}\footnote{\revise{Code: \url{https://github.com/MoePien/wow-flow-matching}}}
\end{abstract}

\let\origFloatBarrier\FloatBarrier\let\FloatBarrier\relax
\section{Introduction}
\let\FloatBarrier\origFloatBarrier

Among the wide range of generative modelling approaches, {flow matching} (FM) has emerged as a particularly flexible and conceptually simple framework \citep{albergo2023building_fm,lipman2023flowmatching_fm,liuflow_fm_recti}. 
Given a fixed target probability measure, FM learns an implicitly defined time-dependent vector field that transports samples from a simple source distribution (e.g., a standard Gaussian) to the target. 
In contrast to stochastic generative models such as diffusion models \citep{ho2020denoising,yang2023diffusion_diffusion_review}, flow matching constructs a {deterministic} continuous normalizing flow. 
This is achieved by regressing the vector field onto a conditional probability flow induced by a prescribed coupling between source and target measures. 
Instead of costly stochastic simulations, this results in solving a deterministic ordinary differential equation at sampling time.

While originally introduced in Euclidean spaces, the flow matching framework extends naturally to more general geometric settings. 
Notably, {Riemannian flow matching} replaces linear interpolations with geodesic paths, enabling generative modelling on manifold-valued data \citep{chen2024flow_riemannian_flow_matching}. 
Beyond the Riemannian setting, recent developments have further generalized the framework to function spaces \citep{kerrigan2024functional_fm,holderriethgeneratormatching}, discrete spaces \citep{gat2024discrete,yimingdefog_discrete_fm} or quotient spaces \citep{klein2023equivariant_flow_matching,ruscelli2026flow}. In this work, we are particularly interested in the case of \emph{Wasserstein} flow matching on the space of probability measures \citep{atanackovic2025meta,haviv2025wasserstein}. Equipping this space with the Wasserstein metric results in the so-called Wasserstein manifold. 
Importantly, this terminology is somewhat misleading, as the space fails to satisfy the exact manifold definition due to the presence of singularities. Particularly, the subset of fixed-size empirical measures is rather a Euclidean quotient space 
containing
permutation-based equivalence classes \citep{boumal2023}.

From a theoretical perspective, flow matching can be interpreted as learning {curves of probability measures} connecting a tractable source measure to a target distribution given by samples \citep{lipman2023flowmatching_fm,wald2025flow}. 
This viewpoint naturally places flow matching within the framework of optimal transport (OT) and Wasserstein geometry, where probability measures are treated as points in the metric Wasserstein space. In this context, our learned transport field corresponds to an interpolation between measures in the Wasserstein space. 
Since the coupling (or transport plan) between the source and target can be arbitrarily chosen, 
this theoretical background further provides a natural design choice via \textcolor{black}{\textit{multisample} flow matching \citep{pooladian2023multisample_couplings}}. 
Since straighter velocity fields improve numerical sampling performance \citep{liuflow_fm_recti}, it is possible to leverage the dynamic OT formulation to learn straight OT velocity fields \citep{chemseddine2025conditional,mousavi2025flow_semidiscre,tong2023improving_minibatch}.

In this work, we build on this perspective and extend flow matching beyond classical probability measures to {metameasures}, i.e., probability measures over probability measures. 
Using the Wasserstein structure of the base space, we equip the space of metameasures with a second-order Wasserstein geometry and study flow-based interpolations in this {Wasserstein-on-Wasserstein} (WoW) space. 
This provides a natural bridge between recent work on the theory of the WoW space \citep{beiglbock2025brenier_wow_theory,bonet2025flowing,emami2025optimal_wow_theory,huesmann2025benamou_wow_theory,pinzi2025nested_wow},
computational methods for hierarchical OT \citep{bonet2025busemann,nguyen2025sotdd, piening2025slicing_gaussian,piening2025slicing_wow}
and
generative methods on the Wasserstein manifolds \citep{atanackovic2025meta,haviv2025wasserstein,jiang2025bureswasserstein_flow_matching,tang2026generative_wowtype_humans}. 
Notably, once we employ the natural discretization of the Wasserstein space via empirical measures, these methods are closely related to general point cloud generative models for unordered sets \citep{li2025generative_set_generation,ludke2025unlocking_set_generation} and point clouds \citep{luo2021diffusion,yang2019pointflow_ot_nn}. Similarly to the use of OT couplings in Euclidean flow matching, we propose to leverage the tools of computational OT to balance efficient training with numerically stable velocity fields.

\revise{We make a conceptual contribution with practical consequences: the WoW continuity equation shows that arbitrary admissible nested plans define valid flows, without invoking a smooth ``Wasserstein manifold.'' Optimal plans minimize transport energy and give straight WoW geodesics. Sliced, lazy linear, and independent plans remain valid with different accuracy--cost trade-offs.}

\revise{WFM \citep{haviv2025wasserstein} uses optimal inner maps and, in its multisample variant, Gaussian Fr\'echet/Bures outer matching. Our framework recovers these cases and permits full-Wasserstein, sliced, and lazy linear outer geometry. NSOT \citep{hui2025notsooptimal_point_cloud_flow_matching} instead constructs an offline coupling from high-resolution noise supersets. Appendix~\ref{sec:related_work} discusses the distinction and an adapted reproduction. Empirically, Bures and flattened Euclidean minibatch OT test whether reordering alone suffices, while our geometry-aware plans target few-step generation on point clouds and finite molecular sets.}
Summarizing our \textbf{contributions},
\begin{itemize}
    \item \revise{we prove that nested outer and inner transport plans induce metameasure flows directly through the WoW continuity equation, and identify their non-local velocity fields as least-squares minimizers (Theorem~\ref{prop:wow_flow_matching}).}
    \item \revise{we show that coupling validity is more general than WoW optimality and design sliced and lazy linear constructions that trade exact geometry for scalability.}
    \item \revise{we validate the coupling principle through controlled Bures--Wasserstein and Euclidean-minibatch baselines on 2D point clouds, 3D shapes, and finite molecular sets.}
\end{itemize}
All novel proofs are given in Appendix \ref{sec:theory_appendix_supp}.

\section{Curves between Probability Measures}
In flow matching generative models, parametrized measure curves are constructed that connect an easy-to-sample source measure with a target measure which is available only through samples \citep{lipman2023flowmatching_fm,wald2025flow}. Notably, this description is not necessarily limited to Euclidean probability measures \citep{chen2024flow_riemannian_flow_matching}. After a short recap on the Wasserstein geometry \citep{AmbrosioGigliSavare2005,Santambrogio2015,wald2025flow}, we will indeed see how we can lift this construction to stochastic measures, i.e., metameasures.

\subsection{Wasserstein Geometry}
Let $(X,d)$ be a Polish metric space and $\mP_2(X)$ the set of Borel probability measures on $X$ with finite second moments, i.e.
$
\int_X d(x,x_0)^2 \,\mathrm d\mu(x) < \infty
$
for some $x_0 \in X$.
The set $\mP_2(X)$ becomes a complete metric space with the \emph{Wasserstein distance} $\W_2$ which is given for any $\mu,\nu \in \mP_2(X)$ by
\begin{equation}
\label{eq:wasserstein}
\W_2^2(\mu,\nu)
\coloneqq
\min_{\pi \in \text{c}(\mu,\nu)}
\int_{X \times X} d(x,x')^2 \, \mathrm d\pi(x,x').
\end{equation}
The \emph{couplings} or \emph{transport plans} are given by 
$\text{c}(\mu,\nu) = \{ \pi \in \mP_2(X \times X): \text{proj}^0_\sharp \pi = \mu,\text{proj}^1_\sharp \pi = \nu\}
$
with $ \text{proj}^i: X \times X \to X$, $(x_0,x_1) \mapsto x_i$. 
By $\text{c}^{{\rm opt}}(\mu,\nu)$
we denote \emph{optimal} couplings, where the minimum in \eqref{eq:wasserstein} is attained.
For $I=[0,1]$, let $C_I(X)$ denote the set of continuous curves mapping from $I$ to
$X$, and $AC_I(X)$ the set of absolutely continuous ones. 

We are interested in two examples for $(X,d)$, namely the Euclidean setting
$(\R^d,\| \cdot \|)$,  and the \revise{so-called}
\emph{Wasserstein-on-Wasserstein} (WoW) case
$(\ProbTwo(\R^d), W_2)$. For visual distinction, we use capital/bold letters in the latter case $X = \ProbTwo(\R^d)$,
i.e. for any $\bmu,\bnu \in \mP_2(\mP_2(\R^d))$ we define
\begin{equation}     \label{eq:wow_what_is}
    \bW_2^2(\bmu, \bnu)
    \coloneqq 
    \min_{\Pi \in \text{C} (\bmu, \bnu)}\int_{\ProbTwo(\R^d) \times \ProbTwo(\R^d)} \W_2^2(\mu, \nu) \d \Pi(\mu, \nu),
\end{equation}
where the \emph{outer plans} are given by
\begin{equation}
\text{C}(\bmu,\bnu) \coloneqq \{ \Pi \in \mP_2(\mP_2(\R^d) \times \mP_2(\R^d)): \text{Proj}^0_\sharp \Pi = \bmu, \allowbreak ~ \text{Proj}^1_\sharp \Pi = \bnu\}.
\end{equation}
In contrast, we address elements $\pi \in \text{c}(\mu,\nu)$ in the case $X=\R^d$ as \emph{inner plans}. 
By $\text{C}^{{\rm opt}}(\bmu,\bnu)$
we again denote the couplings which minimize \eqref{eq:wow_what_is}.

\subsection{Wasserstein-on-Euclidean}
In the following, we briefly recall the basic idea of flow matching in the Euclidean case, see, e.g. \citep{lipman2023flowmatching_fm,wald2025flow} 
and  generalize it to the WoW setting in the next subsection.

\begin{theorem}
Let $(\mu_t)_{t \in I} \in C_I(\mP(\R^d))$ 
and $v: I \times \R^d \to \R^d$ be a 
sufficiently smooth\footnote{For concrete regularity conditions, see \cite[Proposition 8.1.8]{AmbrosioGigliSavare2005}}
Borel vector field 
satisfying $\int_I \int_{\R^d} \|v \| \d \mu_t(x) \d t < \infty$. 
Assume that
$(\mu_t,v_t)$ solves the continuity equation
\begin{equation}\label{ce1}
\partial_t \mu_t+ {\rm{div}} (v_t\mu_t)=0
\end{equation}
in the sense of distributions. 
Then, for $\mu_0$-a.e.\ $x \in \R^d$ the ODE
\begin{equation} \label{flow-ode}
\dot \gamma(t,x) = v(t, \gamma(t,x)), ~ t \in (0,1), \quad \gamma(0,x) = x,
\end{equation}
admits a unique solution $\gamma(\cdot, x)$ on $I$, and with $\gamma_t : \R^d \to \R^d, x \mapsto \gamma(t, x)$, it holds 
$\mu_t = \gamma_{t,\sharp}  \mu_0$.
\end{theorem}
Once such a velocity field $v: I \times \R^d \to \R^d$ is known, we can use any ODE solver to transport samples from $\mu_0$ to $\mu_1$ via the relation $\mu_t = \gamma_{t,\sharp}  \mu_0$.

In flow matching, this velocity field is learned from samples of $\mu_1$ using a neural network.
A simple construction is based on the following observation.

\begin{theorem}\label{vr}
For any coupling $\pi \in \text{c}(\mu,\nu)$,  the curve $(\mu_t)_{t \in I}$ defined by
$
\mu_t \coloneqq \pr^t_\sharp \pi, \allowbreak \pr^t(x, x') \coloneqq (1-t) x + tx',
$
fulfills the continuity equation \eqref{ce1} 
together with a vector field $v_t^\pi$ which minimizes 
\begin{equation} \label{minvr}
J(v) \coloneqq 
\int_I \int_{\R^d \times \R^d} \| v_t( \pr^t(x,x')) - (x'-x) \|^2 \d \pi(x,x') \d t.
\end{equation} 
If $\pi \in {\rm c}^{{\rm opt}}(\mu,\nu)$, then it holds 
$\int_I \int_{\R^d} \|v_t^\pi (x)\|^2 \d \mu_t(x) \d t = W_2^2(\mu,\nu)$.
\end{theorem}
Therefore, a neural network
$v^\theta$ is trained to minimize the empirical loss $\mathcal J(\theta) \coloneqq J(v^\theta)$.

\subsection{Wasserstein-on-Wasserstein}
Our goal is to extend the flow matching framework to the \emph{WoW space} $\ProbTwo(\ProbTwo(\R^d))$, the space of \emph{metameasures}.
This space can be equipped with a meaningful geometry via \eqref{eq:wow_what_is}.
The following nested superposition result was proved in \citep{pinzi2025nested_wow}.

\begin{theorem}\label{thm:wor}
Let \revise{$(M_t)_{t \in I} \in C_I(\mathcal P(\mathcal P(\R^d)))$} and $V: I \times \R^d \times \mP(\R^d) \to \R^d$ be a sufficiently smooth\footnote{For details, see \cite[Theorem 6.2]{pinzi2025nested_wow}} Borel non-local vector field satisfying $\int_I \int_{\mP(\R^d)} \int_{\R^d} \|V(t,x,\mu)\|  \d \mu(x) \d M_t(\mu) \d t < \infty$. Assume that $(M_t, V_t)$ satisfies the continuity equation
\begin{equation}\label{eq:wow_continuity}
\revise{\partial_t M_t} + \operatorname{div}_{\Prob}(V_t M_t) = 0 
\end{equation}
in the sense of duality with 
cylindric functions, see Definition \ref{def1} in Appendix \ref{sec:theory_appendix_supp}.
Then, for $M_0$-a.e.\ $\mu_0 \in \mathcal P(\R^d)$ 
the collection of ODEs, indexed by $x \in {\rm supp} \, \mu_0$,
\begin{align}\label{nonlocal-flow-ode}
\dot \gamma(t,x,\mu_0) = V(t, \gamma(t,x,\mu_0), \gamma_{t, \mu_0, \sharp}  \mu_0), ~ t \in (0,1), ~
\gamma(0,x,\mu_0) = x,
\end{align}
admits a unique collection of solutions $\gamma(\cdot, x, \mu_0)$ on $I$, where we set $\gamma_{t, \mu_0} : {\rm supp} \, \mu_0 \to \R^d$, $x \mapsto \gamma(t, x, \mu_0)$. It holds that for $M_0$-a.e.\ $\mu_0$, the curve
$\mu_t \coloneqq \gamma_{t, \mu_0, \sharp}  \mu_0$
solves 
$\partial_t \mu_t+ {\rm{div}} (V_t(\cdot, \mu_t) \mu_t) = 0$
in the sense of distributions. 
Furthermore, with $\Gamma_t : \mathcal P(\R^d) \to \mathcal P(\R^d),\allowbreak \mu \mapsto \gamma_{t, \mu, \sharp}  \mu,$ we have $M_t = \Gamma_{t,\sharp} M_0$.
\end{theorem}
Once such \revise{a} non-local vector field $V: I \times \R^d \times \mP(\R^d) \to \R^d$ is known, we can  transfer a sample $\mu_0$ from $\bmu = M_0$ to one of $\bnu = M_1$ by solving the above system of ODEs \eqref{nonlocal-flow-ode} starting in a collection of samples $x \sim \mu_0$. Importantly, during integration, each single curve $\gamma(t_n, x, \mu_0)$ at time step $t_n$ depends (non-locally) on the collective information of all curves $\gamma_{t_{n-1}, \mu_0, \sharp}  \mu_0$ at the prior time step. 

In order to formulate a result on how to approximate the non-local vector field similar to Theorem \ref{vr}, we fix $\bmu, \bnu \in \ProbTwo(\ProbTwo(\R^d))$ and introduce the set of \emph{random couplings} by 
\begin{equation}
{\rm RC}(\bmu, \bnu) \coloneqq 
\{
\bPi \in \mP(\mP(\R^d \times \R^d)):
{\rm \bf{Proj}}^0_\sharp \bPi = \bmu, \, {\rm \bf{Proj}}^1_\sharp \bPi = \bnu
\}
\end{equation}
where
$
\bproj^i:\mP(\R^d \times \R^d) \to \mP(\R^d), ~ \pi \mapsto \pr^i_\sharp \pi.
$
Such random couplings are 
easily constructed in a two-step procedure based on a given outer plan and inner plan selection, i.e., mappings
$
\operatorname{OP}: (\bmu, \bnu) \mapsto \operatorname{OP}_{\bmu, \bnu} \in \text{C}(\bmu,\bnu),
$
and
$
\operatorname{IP}: (\mu, \nu) \mapsto \operatorname{IP}_{\mu, \nu} \in \text{c}(\mu,\nu).
$

This procedure as well as the connection to the WoW distance \eqref{eq:wow_what_is} is described in the following lemma, extending the results of \citep{pinzi2025nested_wow}. 

\begin{restatable}{lemma}{IPOPlem}
\label{prop:wow_distance_equiv_main}
Let $\bmu, \bnu \in \ProbTwo(\ProbTwo(\R^d))$. Then,
for any $\operatorname{OP}_{\bmu, \bnu} \in \text{C}(\bmu,\bnu) $ and any $\operatorname{OP}_{\bmu, \bnu}$-measurable selection $\operatorname{IP}: (\mu, \nu) \mapsto \operatorname{IP}_{\mu,\nu} \in c(\mu,\nu)$, it holds that $\bPi \coloneqq \operatorname{IP}_\sharp \operatorname{OP}_{\bmu, \bnu} \in {\rm RC}(\bmu,\bnu)$.
Furthermore, the WoW distance \eqref{eq:wow_what_is} satisfies
\begin{equation}
    \label{eq:wow_reformulated}
        \bW^2_2(\bmu, \bnu) =\inf_{\bPi \in {\rm RC}(\bmu,\bnu)} \int_{\ProbTwo(\R^d \times \R^d)} \int_{\R^d \times \R^d} \|x- x'\|^2 \, \d \pi(x, x') \, \d \bPi(\pi).
\end{equation}
\end{restatable}


Now, we construct a metameasure curve $(M_t)_{t \in [0, 1]}$ in $\ProbTwo(\ProbTwo(\R^d))$ that connects given metameasures $\bmu$ and $\bnu$, and the corresponding non-local vector field $V(t,x,\mu)$.
In particular, we need to make sure that $(M_t, V_t)$ fulfills the WoW continuity equation \eqref{eq:wow_continuity}.
To this aim, we employ the established flow matching framework similar to Theorem \ref{vr}, for the proof see Appendix \ref{sec:theory_appendix_supp}.

\begin{restatable}{theorem}{WOWthm}
\label{prop:wow_flow_matching}
Let $\bmu, \bnu \in \ProbTwo(\ProbTwo(\R^d))$ and $\bPi \in {\rm RC}(\bmu,\bnu)$, and 
define the curve
$(M_t)_{t\in[0,1]}$
by
$
M_t \coloneqq (\bproj^t_\sharp\bPi), \, \bproj^t  \pi \coloneqq \pr^t_\sharp \pi. 
$
Then, $\bPi \in \ProbTwo(\ProbTwo(\R^d \times \R^d))$, i.e.\ it holds the integrability
\begin{equation}\label{IP-OP-int}
\int_{\ProbTwo(\R^d\times\R^d)}
\int_{\R^d\times\R^d} \|x\|^2 + \|x'\|^2\,\d\pi(x,x')\,\d\bPi(\pi)
< \infty.
\end{equation}
Furthermore, $(M_t)_{t \in [0,1]} \in AC_I(\ProbTwo(\ProbTwo(\R^d)))$ and the functional
\begin{equation}\label{minWoW}
J(V) \coloneqq
\int_I \int_{\ProbTwo(\R^d\times\R^d)} \int_{\R^d\times\R^d}
\|V(t,\pr^t(x,x'),\bproj^t(\pi)) - (x'-x)\|^2 
\,\d\pi(x,x')\,\d\bPi (\pi)\,\d t
\end{equation}
admits a unique minimizer $V_t$, such that $(M_t, V_t)$
fulfills the continuity equation \eqref{eq:wow_continuity}. \revise{Moreover, the kinetic energy of $V$ satisfies}
\begin{equation}\label{eq:wow_flow_matching_formulation_main}
\revise{\bW_2^2(\bmu,\bnu)} \mathrel{\revise{\le}}
\int_I \int \int
\|V(t,x,\mu)\|^2 \,\d\mu(x)\,\d M_t(\mu)\,\d t
\mathrel{\revise{\le}} \revise{\int\int \|x'-x\|^2\,\d\pi(x,x')\,\d\bPi(\pi),}
\end{equation}
\revise{with equality in both bounds if} $\bPi = \operatorname{IP}^*_\sharp \operatorname{OP}^*$\revise{ with $\operatorname{OP}^*\in {\rm C}^{{\rm opt}}(\bmu,\bnu)$ and $\operatorname{IP}^*_{\mu,\nu}\in {\rm c}^{\rm{opt}}(\mu,\nu)$.}
\end{restatable}
Together with Lemma \ref{prop:wow_distance_equiv_main}, we can construct metameasure curves between $\bmu$ and $\bnu$. \revise{By \eqref{eq:wow_flow_matching_formulation_main}, cost-efficient couplings guarantee low kinetic energy, easing numerical integration.}
The \emph{Wasserstein flow matching}
approach
\citep{haviv2025wasserstein} 
is recovered for appropriate \revise{IP and OP (Section~\ref{sec:couplings}).}

\section{Wasserstein-on-Wasserstein Flow Matching}
\label{subsec:gen_w_fm}
Theorem \ref{prop:wow_flow_matching} yields a principled training objective for learning the WoW velocity field similar to the classical Euclidean flow matching approach. However, enabling the generation of arbitrary, non-parametric probability measures requires discretizing them. Here, we employ a scalable free-support approach \citep{lindheim2023simple_free_support} that uses fixed uniform weights for parametric locations, as opposed to fixed-support approaches with parametric weights for fixed locations \citep{lin2020fixed} or approaches based on a parametric family of distributions \citep{jiang2025bureswasserstein_flow_matching}.

We aim to parametrize our velocity field via a network
$
V_\theta: [0,1]\times \R^{d\times N}\to \R^{d\times N}, \allowbreak N \in \{\underline{N}, \ldots, \overline{N}\},
$
where $V_\theta$ is potentially able to process data for non-fixed $N$, and the network is permutation-equivariant (or $S_N$-equivariant), i.e.,
$
V_\theta(t, P_\sigma \mathbf{x}) = P_\sigma\, V_\theta(t,\mathbf{x})
$ 
for all $\sigma \in S_N.$
Here, \(\mathbf{x}=(x_1,\dots,x_N)\), $S_N$ is the group of $N$-permutations and \(P_\sigma\) is the permutation matrix associated with \(\sigma\).
This corresponds to encoding empirical measures as elements of the quotient space 
$\R^{d \times N}/S_N$ composed of equivalence classes of permutable vectors \cite[Ch. 9]{boumal2023}. Denoting equality up to reparametrization via $\simeq$, we obtain the parametrization
\begin{align}
\label{eq:empirical_matrices_and_vector_representation}
[\mathbf{x}] \in \R^{d \times N}/S_N  \simeq  \hat{\mu}=\frac{1}{N}\sum_{j=1}^{N}\delta_{x_j} \in \ProbTwo^N(\R^d), \quad
V_\theta(t, x_i, \hat{\mu}) \in \R^d  \simeq V_\theta(t, \mathbf{x})_i \in \R^d. 
\end{align}
Hence, the network predicts the velocity of particle $x_i \in \R^d$ as $V_\theta(t, \mathbf{x})_i \in \R^d$ at time $t$ based on a vector-encoded $\hat{\mu}$. \revise{Throughout, $\|\cdot\|$ on $\R^{d\times N}$ denotes the normalized Frobenius norm $\|\mathbf{x}\|^2 \coloneqq \frac1N\sum_{j=1}^N\|x_j\|^2$.} Notably, natural choices for the parametrization are transformer architectures without positional encodings \citep{guo2021pct_point_cloud_transformer,lee2019settransformer}
or a shared feedforward network for all points in combination with global information \citep{qi2017pointnet}.

Given such a parametrization, our training objective is then the minimization of the loss 
$
{\mathcal{L}}_{\text{WoW-FM}}(\theta) = J(V_\theta)
$
in \eqref{minWoW}.
Since the actual integrals appearing in this loss are intractable, we follow the standard flow matching procedure and train the network using a stochastic Monte Carlo approximation $\widehat{\mathcal{L}}_{\text{WoW-FM}}(\theta)$, where we rely on  two batches of $B$ source and target measures with (potentially random) discretization $N$. 
Given estimators for the outer plan via $\widehat{\operatorname{OP}}$ and the inner plans via $\widehat{\operatorname{IP}}$,
we compute a Monte Carlo loss estimator for every gradient step via 
$
t\sim \mathrm{Unif}[0,1], ~ N \sim  \mathrm{Unif}(\{\underline{N}, \ldots, \overline{N}\}),
~ \hat{\bmu}=\frac{1}{B}\sum_{i=1}^{B}\delta_{\hat{\mu}_i},
~
\hat{\bnu}=\frac{1}{B}\sum_{i'=1}^{B}\delta_{\hat{\nu}_{i'}}
$
and 
\begin{align}
&\{(\hat{\mu}_b , \hat{\nu}_b)\}_{b=1}^B \stackrel{\text{iid}}{\sim} \widehat{\operatorname{OP}}_{\hat{\bmu},\hat{\bnu}}
=
\sum_{i=1}^{B}\sum_{i'=1}^{B}\hat{\Pi}^{\operatorname{OP}}_{ii'}\,\delta_{(\hat{\mu}_i,\hat{\nu}_{i'})},
&&\hat{\Pi}^{\operatorname{OP}}_{ii'} \in \left[0, \frac{1}{B}\right],\\
&(\mathbf{x}_b,\mathbf{x}'_b) = \bigl((x_{b,j}, x'_{b,j})\bigr)_{j=1}^N \stackrel{\text{iid}}{\sim}
\widehat{\operatorname{IP}}(\hat{\mu}_b,\hat{\nu}_b)
=
\sum_{i=1}^{N}\sum_{i'=1}^{N}\hat{\pi}^{\operatorname{IP}}_{b, ii'}\,\delta_{(x_{b,i},x'_{b, i'})},
&& \hat{\pi}^{\operatorname{IP}}_{b,ii'} \in \left[0, \frac{1}{N}\right],\\
&\widehat{\mathcal{L}}_{\text{WoW-FM}}(\theta) =
\frac{1}{B}\sum_{b=1}^{B}
\left\|
V_\theta\bigl(t,\pr^t(\mathbf{x}_b,\mathbf{x}'_b)\bigr)-(\mathbf{x}'_b-\mathbf{x}_b)
\right\|^2.
\end{align}

Note that the augmented vector $\pr^t(\mathbf{x}_b, \mathbf{x}'_b)$ represents $\pr^t(x, x')$ and $\bproj^t(\pi)$ simultaneously, with $(\mathbf{x}_b, \mathbf{x}'_b)$ given by $N$ i.i.d.\ draws from $\widehat{\operatorname{IP}}(\hat{\mu}_b,\hat{\nu}_b)$. For permutation-valued inner plans this reduces to repeated sampling among the matched pairs $(x_i, x'_{\sigma(i)})$.
Under the assumption of a parametric source sampling procedure for $\hat{\bmu}$ similar to the classical Gaussian sampling in generative modelling, 
we can produce new target samples by encoding $\hat{\mu} \sim \hat{\bmu}$ as a vector $\mathbf{x}_{\hat{\mu}}$  and numerically integrating the flow ODE \eqref{nonlocal-flow-ode} produced by $V_\theta$. 
In particular, we may directly parametrize a discretized metameasure via 
a vectorized measure $\overboldarrow{\mu} \simeq \hat{\bmu}$, i.e., $\mathbf{x}_{\hat{\mu}} \sim \overboldarrow{\mu} \in \ProbTwo(\R^{d \times N})$.

Importantly, any admissible choice for our transport plan estimators would result in a valid flow, mirroring the flexibility of Euclidean flow matching \citep{pooladian2023multisample_couplings,wald2025flow}. 
\revise{This design space contains standard point-cloud WFM $(\widehat{\operatorname{OP}}_{\mathrm{ind}},\widehat{\operatorname{IP}}_{\W})$ \citep{haviv2025wasserstein}, multisample Gaussian/Bures matching $(\widehat{\operatorname{OP}}_{\mathrm{BW}},\widehat{\operatorname{IP}}_{\W})$, and flattened Euclidean minibatch OT. The latter tests reordering alone. W, SW, and LLW instead use permutation-invariant cloud geometry.}

\subsection{Constructing Transport Plans}\label{sec:couplings}
Even though flow matching is theoretically valid for any transport plan, the specific choice affects numerical performance. 
Besides the standard independent coupling $\mu \otimes \nu$, 
a common alternative is the Wasserstein coupling $\pi^* = \argmin_{\pi \in {\rm c}(\mu, \nu)} \int \| x-x'\|^2 \d \pi$ \citep{chemseddine2025conditional,tong2023improving_minibatch}. 
By construction, the latter approximates the optimal-transport straight-line interpolation and often yields more stable ODE integration, at the cost of more expensive training due to the complexity of OT solvers. \revise{For nested couplings, \eqref{eq:wow_flow_matching_formulation_main} makes this precise: the kinetic energy of the induced velocity field is at most the coupling cost.}

This issue is even more pronounced in our setting, where both inner and outer transport plans must be chosen. 
While our framework allows for straight-line WoW interpolations for metameasures via 
\eqref{eq:wow_flow_matching_formulation_main},
approximating the loss requires solving the (batched) WoW 
problem \eqref{eq:wow_what_is} at each step.
For batch size $B$, this requires $B^2$ Euclidean Wasserstein distances and solving a linear program of complexity $\mathcal{O}(B^3)$.
With exact Wasserstein solvers of complexity $\mathcal{O}(N^3)$ for discretization size $N$, this yields an overall complexity of $\mathcal{O}(B^2N^3 + B^3)$. 
Even for Sinkhorn-based solvers, the resulting $\mathcal{O}(B^2N^2 \log N + B^2 \log B)$ complexity remains substantial. 
As an alternative to WoW solvers, we therefore consider practical alternative constructions for inner and outer transport plans.

\paragraph{Independent Couplings.}
The simplest admissible choice is given by independent couplings based on the product measure.
At the outer level, this corresponds to
$\widehat{\operatorname{OP}}_{\text{ind}}
=\hat{\bmu}\otimes \hat{\bnu}$
and at the inner level to
$
\widehat{\operatorname{IP}}_{\mathrm{ind}}(\hat{\mu}, \hat{\nu})=\hat{\mu}\otimes \hat{\nu}.
$
This choice is computationally trivial, but may lead to curved and unstable interpolations. 
Choosing ${\widehat{\operatorname{IP}}}_{\mathrm{ind}}$ corresponds to standard Euclidean flow matching on randomly ordered point clouds with permutation-based data augmentation, see \citep{hui2025notsooptimal_point_cloud_flow_matching}.

\paragraph{Wasserstein Couplings.}
A straightforward alternative is the classical use of standard OT solvers. 
For the outer plan, this means choosing
$
\widehat{\operatorname{OP}}_{\bW}
$
according to 
$
\argmin_{{\operatorname{OP}_{\hat \bmu, \hat \bnu}} \in {\rm C}(\hat \bmu, \hat \bnu)} \int \W^2_2(\mu,\nu)  \d {\operatorname{OP}_{\hat \bmu,\hat  \bnu}}$
based on a precomputed $B \times B$ Euclidean Wasserstein cost matrix.
As the inner coupling, we set $
\widehat{\operatorname{IP}}_{\W}(\hat{\mu}, \hat{\nu})$
according to 
$
\argmin_{{\operatorname{IP}_{\hat \mu, \hat \nu}} \in {\rm c}(\hat \mu, \hat \nu)} \int \|x-x'\|^2  \d {\operatorname{IP}_{\hat \mu, \hat \nu}}$. Combining 
$\widehat{{\operatorname{OP}}}_{\bW}$ and $\widehat{{\operatorname{IP}}}_{\W}$ allows for approximating the dynamic WoW formulation, see Theorem~\ref{prop:wow_flow_matching}.
Inner Wasserstein couplings appear in $S_N$-equivariant flow matching \citep{klein2023equivariant_flow_matching} for point clouds \citep{hui2025notsooptimal_point_cloud_flow_matching} and the choice of $\widehat{\operatorname{IP}}_{\W}$ and $\widehat{\operatorname{OP}}_{\text{ind}}$ reflects the training setup of Wasserstein flow matching \citep{haviv2025wasserstein}.

\paragraph{\revise{Euclidean Minibatch Coupling.}}
\revise{This performs outer minibatch OT using squared Euclidean costs between vectorized point clouds, followed by exact $\widehat{\operatorname{IP}}_{\W}$. While order-dependent, its expected cost under independent random orderings is the squared mean distance plus both variances.}

\paragraph{\revise{Bures--Wasserstein Coupling.}}
\revise{Let $\hat m_\mu$ and $\hat\Sigma_\mu$ denote the empirical mean and covariance of $\hat\mu$. Outer minibatch OT uses $c_{\mathrm{BW}}(\hat\mu,\hat\nu)=\|\hat m_\mu-\hat m_\nu\|^2+\operatorname{tr}(\hat\Sigma_\mu+\hat\Sigma_\nu-2(\hat\Sigma_\mu^{1/2}\hat\Sigma_\nu\hat\Sigma_\mu^{1/2})^{1/2})$, followed by exact $\widehat{\operatorname{IP}}_{\W}$. Thus, outer matching only considers the first two moments.}

\paragraph{Sliced Wasserstein Couplings.}
Since the computational burden of $\widehat{{\operatorname{OP}}}_{\bW}$ is mainly the result of the cost matrix computation, we can simply replace it via a related, but more efficient distance estimator.
A popular Wasserstein alternative is the sliced Wasserstein distance \citep{bonnotte2013unidimensional} given via
$
\SW_2^2(\mu,\nu)
=
\int_{\mathbb{S}^{d-1}}
\W_2^2\bigl(\pr^\theta_\#\mu,\pr^\theta_\#\nu\bigr)\,\d S(\theta),$ where 
$
\pr^\theta(x)=\langle \theta,x\rangle,
$
and $S$ is the uniform probability measure on $\mathbb{S}^{d-1}$.
Now, we choose
$
\widehat{\operatorname{OP}}_{\text{SW}}
$
according to
$
\argmin_{\operatorname{OP}_{\hat{\bmu}, \hat{\bnu}} \in {\rm C}(\hat{\bmu}, \hat{\bnu})} \int \SW_2^2(\mu,\nu)\,\d \operatorname{OP}_{\hat{\bmu}, \hat{\bnu}}$
based on a $B \times B$ sliced Wasserstein cost matrix \citep{piening2025slicing_gaussian}.
Since each $\theta \in \mathbb{S}^{d-1}$ gives a 1D transport plan between projected measures, canonically lifting to a plan $\operatorname{IP}_\theta(\hat{\mu}, \hat{\nu}) \in {\rm c}(\hat{\mu}, \hat{\nu})$, we can similarly employ inner sliced transport plans $\operatorname{IP}_{\text{SW}}(\hat{\mu}, \hat{\nu}) = \int_{\mathbb{S}^{d-1}} \operatorname{IP}_\theta(\hat{\mu}, \hat{\nu}) \, \d S(\theta)$.

\paragraph{Lazy Linear Wasserstein Couplings\revise{.}}
\revise{Alternatively}, 
linearized Wasserstein distances have been introduced precisely to accelerate pairwise Wasserstein distance calculation \citep{wang2013linear_lot}, \revise{as needed for computing} $\widehat{\operatorname{OP}}_{\bW}$. 
Focusing on our use case of empirical measures with $N$ fixed, this technique replaces the Wasserstein distance with the linearized divergence
\begin{align}
\label{eq:linear_wasserstein}
\operatorname{LW}_2^2(\hat{\mu},\hat{\nu})=
\int \| x - x'\|^2 \,\d \bigl((\Id, T_{\hat{\nu},\hat{\rho}}^{-1} \circ T_{\hat{\mu},\hat{\rho}})_\sharp \hat{\mu}\bigr)(x, x')=
\|P_{\mathbf{x}_{\hat{\mu}},\mathbf{x}_{\hat{\rho}}} \mathbf{x}_{\hat{\mu}} - 
P_{\mathbf{x}_{\hat{\nu}},\mathbf{x}_{\hat{\rho}}} 
\mathbf{x}_{\hat{\nu}}\|^2
\end{align}
based on precomputed transport maps $T_{\hat{\mu},\hat{\rho}},  T_{\hat{\nu},\hat{\rho}}: \R^d \to \R^d$ and permutation matrices $P_{\mathbf{x}_{\hat{\mu}},\mathbf{x}_{\hat{\rho}}},  P_{\mathbf{x}_{\hat{\nu}},\mathbf{x}_{\hat{\rho}}} \in \R^{N \times N}$ characterized via
\begin{align}
     \W_2^2(\hat{\mu}, \hat{\rho}) \allowbreak &=
\int \| x - x'\|^2 \,\d \bigl((\Id,T_{\hat{\mu},\hat{\rho}})_\sharp \hat{\mu}\bigr)(x, x')
\allowbreak=
\|\mathbf{x}_{\hat{\rho}} - P_{\mathbf{x}_{\hat{\mu}},\mathbf{x}_{\hat{\rho}}} \mathbf{x}_{\hat{\mu}}\|^2,\\
 \W_2^2(\hat{\nu}, \hat{\rho}) \allowbreak &=
\int \| x - x'\|^2 \,\d \bigl((\Id,T_{\hat{\nu},\hat{\rho}})_\sharp \hat{\nu}\bigr)(x, x')
\allowbreak=
\|\mathbf{x}_{\hat{\rho}} - P_{\mathbf{x}_{\hat{\nu}},\mathbf{x}_{\hat{\rho}}} \mathbf{x}_{\hat{\nu}}\|^2.
\end{align}

Importantly, the existence of such maps 
that (i) realize the discrete Wasserstein distance and (ii) are represented as 
invertible permutations, as well as the vector representations 
$\mathbf{x}_{\hat{\mu}_{b}}$, $\mathbf{x}_{\hat{\nu}_{b'}}$, and $\mathbf{x}_{\hat{\rho}} \in \R^{d \times N}$, 
relies on the formulation of empirical measures $\hat{\mu}_b$, $\hat{\nu}_{b'}$, 
and $\hat{\rho} \in \ProbTwo^N(\R^d)$ of fixed size as in 
\eqref{eq:empirical_matrices_and_vector_representation}.
While linearization allows efficient approximation of pairwise 
Wasserstein distances for a suitably chosen reference $\hat{\rho}$, employing 
\eqref{eq:linear_wasserstein} as an inner plan estimator would require solving 
two OT problems per gradient step: Instead of the single solve required by 
$\widehat{\operatorname{IP}}_{\W}$, we would solve for $T_{\hat{\mu},\hat{\rho}}$ and 
$T_{\hat{\nu},\hat{\rho}}$.
To avoid this overhead, we instead exploit the local stability of discrete OT maps 
to construct a {lazy} linearization strategy based on a suitable source 
metameasure $\hat{\bmu}$.

\begin{restatable}{proposition}{LAZYprop}
\label{prop:lazy_linear_justification}
Given
the mapping $\operatorname{Vec2M}: \R^{d \times N} \to \ProbTwo^N(\R^d), \quad \mathbf{x} \mapsto \hat{\mu}=\frac{1}{N}\sum_{j=1}^{N}\delta_{x_j},$  
a reference measure $\hat{\rho} \in \ProbTwo^N(\R^d)$,
and some $\mathbf{x}_{\hat{\rho}} \in \R^{d \times N}$ with $\mathbf{x}_{\hat{\rho}, i} \neq \mathbf{x}_{\hat{\rho}, i'}$ for all $i \neq i'$ and $\operatorname{Vec2M}(\mathbf{x}_{\hat{\rho}}) = \hat{\rho}$,
    then for (Lebesgue-)almost every $\mathbf{x}' \in \R^{d \times N}$, there exists  $R_{\mathbf{x}'} > 0$ such that for all $\mathbf{x}$ with
    $
    \|\mathbf{x}  - \mathbf{x}_{\hat{\rho}}\| < R_{\mathbf{x}'},
    $
    it holds that
    \begin{equation}
    \label{eq:Wasserstein_linear_for_subset}
        \W_2^2(\operatorname{Vec2M}(\mathbf{x}), \operatorname{Vec2M}(\mathbf{x}'))
        =
        \|\mathbf{x} - P_{ \operatorname{Vec2M}(\mathbf{x}'),\mathbf{x}_{\hat{\rho}}} \mathbf{x}'\|^2.
    \end{equation}
\end{restatable}
The elementary proof is contained in Section~\ref{sec:theory_appendix_supp} and is based on the fact that all Wasserstein distances in $\ProbTwo^N(\R^d)$ are realized via permutations \citep{PeyreCuturi2019}. 
Since the set of permutations is finite, the optimal permutation is almost everywhere unique and therefore locally stable around the reference. This statement implies a simple corollary motivating an accelerated linear OT regime.
Note that it follows directly from Proposition~\ref{prop:lazy_linear_justification}, since the minimum of $B$ positive constants is again a positive constant. 

\begin{corollary}
\label{corr:lazy_linear_justification}
Under the setting of Proposition~\ref{prop:lazy_linear_justification}, fix a right-inverse $\operatorname{M2Vec}: \ProbTwo^N(\R^d) \to \R^{d \times N}$ of $\operatorname{Vec2M}$, i.e., $\operatorname{Vec2M} \circ \operatorname{M2Vec} = \Id$, and define the $\hat\rho$-alignment map
$
    \Phi_{\hat\rho}: \R^{d \times N} \to \R^{d \times N}, \allowbreak
    \Phi_{\hat\rho}(\mathbf{x}) \coloneqq P_{ \mathbf{x},\mathbf{x}_{\hat\rho}}\, \mathbf{x},
$
which re-orders $\mathbf{x}$ so that its entries are optimally matched against $\mathbf{x}_{\hat\rho} \coloneqq \operatorname{M2Vec}(\hat\rho)$.
Then, for almost every target metameasure 
$\hat{\bnu} = \frac{1}{B}\sum_{i=1}^{B}\delta_{\frac{1}{N}\sum_{j=1}^{N}\delta_{x_{i,j}}}  \allowbreak \in \ProbTwo(\ProbTwo^N(\R^d))$
(in the Lebesgue sense on $\R^{d\times N\times B}$), there exists $R_{\hat\bnu} > 0$ such that whenever
$
    \supp(\operatorname{M2Vec}_\sharp \hat\bmu) \subset B(\mathbf{x}_{\hat\rho}, R_{\hat\bnu}) \allowbreak \coloneqq \{\mathbf{x} \in \R^{d\times N} : \|\mathbf{x} - \mathbf{x}_{\hat\rho}\| \leq R_{\hat\bnu}\},
$
it holds
\begin{equation}
    \bW_2^2(\hat\bmu, \hat\bnu) 
    = 
    \W_2^2\bigl(\operatorname{M2Vec}_\sharp \hat\bmu,\; (\Phi_{\hat\rho})_\sharp \operatorname{M2Vec}_\sharp \hat\bnu\bigr).
\end{equation}
\end{corollary}
This statement applies directly to the practical setting, where 
$\hat{\bnu}$ is our full training dataset. It says that for a source metameasure obtained by sampling pre-ordered vectors from $\operatorname{M2Vec}_{\sharp}\hat{\bmu}$ centred around $\operatorname{M2Vec}(\hat{\rho})$ (in the Euclidean sense), we \emph{only} need to permute our empirical measure training samples with respect to $\operatorname{M2Vec}(\hat{\rho})$, once before training. Notably, $\operatorname{M2Vec}_{\sharp}\hat{\bmu}$ simply describes sampling via a vectorized $\overboldarrow{\mu} \in \ProbTwo(\R^{d \times N})$ and 
$\mathbf{x}_{\hat{\rho}} = \operatorname{M2Vec}(\hat{\rho}) \simeq \hat{\rho}$
as described in Section \ref{subsec:gen_w_fm}.

Once we employ this discretization for $\overboldarrow{\mu}$ with 
$\supp (\overboldarrow{\mu}) \subset B(\mathbf{x}_{\hat{\rho}}, \revise{R_{\hat{\bnu}}})  $, 
we are in a classical Euclidean flow matching regime on $\R^{d \times N}$. If we choose the standard independent coupling here, this corresponds to the Wasserstein flow matching regime of employing $\widehat{\operatorname{OP}}_{\text{\textbf{ind}}}$
and $\widehat{\operatorname{IP}}_{\W}$. Opting for classical OT flow matching here, results in approximating the WoW flow via $\widehat{\operatorname{OP}}_{\bW}$ and $\widehat{\operatorname{IP}}_{\W}$.

Inspired by this observation, we design our lazy linear plans.
Based on a reference measure $\hat{\rho} \in \ProbTwo^N(\R^d)$ with vectorization $\mathbf{x}_{\hat{\rho}} \simeq \hat{\rho}$ satisfying $\mathbf{x}_{\hat{\rho}, i} \neq \mathbf{x}_{\hat{\rho}, i'}$ for $i \neq i'$, and 
a reference metameasure $\hat{\bmu}_{\hat{\rho}} \coloneqq \operatorname{Vec2M}_{\sharp} \overboldarrow{\mu}$ with $\overboldarrow{\mu} \in \ProbTwo(\R^{d \times N})$ heavily centred around $\mathbf{x}_{\hat{\rho}}$, we introduce our lazy linear \revise{plan and} divergence as
\begin{align}
\revise{\widehat{\operatorname{IP}}_{\mathrm{LLW}}(\hat{\mu},\hat{\nu})}
&\coloneqq
\revise{\frac{1}{N}\sum_{j=1}^N \delta_{\left((\mathbf{x}_{\hat{\mu}})_j,\, (P_{\mathbf{x}_{\hat{\nu}},\mathbf{x}_{\hat{\rho}}}\mathbf{x}_{\hat{\nu}})_j\right)} \in {\rm c}(\hat{\mu},\hat{\nu}),}
\nonumber\\
\operatorname{LLW}_2^2(\hat{\mu},\hat{\nu})
&\coloneqq
\int \| x - x'\|^2 \,\d \revise{\widehat{\operatorname{IP}}_{\mathrm{LLW}}(\hat{\mu},\hat{\nu})}(x, x')
=
\|\mathbf{x}_{\hat{\mu}} - P_{\mathbf{x}_{\hat{\nu}},\mathbf{x}_{\hat{\rho}}}\, \mathbf{x}_{\hat{\nu}}\|^2,
\end{align}
where again $P_{\mathbf{x}_{\hat{\nu}},\mathbf{x}_{\hat{\rho}}}$ is the permutation aligning $\mathbf{x}_{\hat{\nu}}$ to $\mathbf{x}_{\hat{\rho}}$\revise{, and the source vector $\mathbf{x}_{\hat{\mu}}$ is already ordered like $\mathbf{x}_{\hat{\rho}}$ by construction of $\overboldarrow{\mu}$}. 
Crucially, this permutation depends only on $\hat{\rho}$ and $\hat{\nu}$ and not on $\hat{\mu}$. Thus, it can be precomputed once per sample $\hat{\nu}$ before training.

Due to Corollary~\ref{corr:lazy_linear_justification}, 
we know that $\operatorname{LLW}^2_2(\hat{\mu},\hat{\nu}) = \W^2_2(\hat{\mu},\hat{\nu})$ for $\hat{\bmu}_{\hat{\rho}} \coloneqq\operatorname{Vec2M}_{\sharp} \overboldarrow{\mu}$ with `high' probability if $\overboldarrow{\mu}$ is sufficiently centred around $\mathbf{x}_{\hat{\rho}}$.
Now, we may use this to approximate WoW transport plans via optimal outer and inner plans.
Consequently for suitably concentrated $\hat{\mu}_i \sim \hat{\bmu}_{\hat{\rho}}$
with randomly batched $\hat{\bmu} = 1/B \sum_{i=1}^B \delta_{ \hat \mu_i}$, we set our transport plan estimator $
\widehat{{\operatorname{OP}}}_{\text{\textbf{LLW}}}
$
according to 
$
\argmin_{{\operatorname{OP}_{\hat \bmu, \hat \bnu}} \in {\rm C}(\hat \bmu, \hat \bnu)}\int \operatorname{LLW}^2_2(\mu,\nu)  \d {\operatorname{OP}_{\hat{\bmu}, \hat \bnu}}$
and \revise{use the inner plan $\widehat{\operatorname{IP}}_{\mathrm{LLW}}(\hat{\mu}, \hat{\nu})$ defined above}. \revise{Unlike previous estimators, the lazy linear estimator quality depends on the source distribution, see Appendix \ref{sec:ablation_exp}. }

\section{Experiments}
Note that our general approach already covers the Wasserstein flow matching model \citep{haviv2025wasserstein}, which has recently demonstrated SOTA performance in point cloud generation.
Throughout our experiments, we therefore focus on explicit `apples-to-apples' comparison: we aim to compare our couplings based on a fixed architecture and training setup. Mainly presenting experiments with fixed $N$, a multi-resolution experiment is deferred to Appendix~\ref{sec:further_exp} and ablations to Appendix~\ref{sec:ablation_exp}.
\subsection{Implementation and Evaluation}
\label{subsec:main_text_implementation}
\paragraph{Networks.} 
To reduce any bias induced by a specific architecture, we employ two distinct architectural choices. 
As a simple \textbf{baseline model}, we employ a shared $\operatorname{MLP}_\theta$
conditioned on per-point local features $F_\ell(\mathbf{x})_i$ \citep{piening2025novel,sato2020fast}, i.e.,
sorted pairwise distances and neighbour positions,
and global features $F_g(\mathbf{x})$, i.e., mean and covariance of $(\mathbf{x}_1, \ldots, \mathbf{x}_N)$,
together with a sinusoidal time embedding $\gamma(t)$, followed by
a global self-attention layer $\operatorname{A}_\theta$. This results in
$
V_\theta(t, \mathbf{x}_t)
= \operatorname{A}_\theta(M_1, \ldots, M_N)$
with $M_i = \operatorname{MLP}_\theta((\mathbf{x}_t)_i,\,F_\ell(\mathbf{x}_t)_i,\,F_g(\mathbf{x}_t),\,\gamma(t))$.
As a more standard architecture, we additionally employ a \textbf{transformer} \citep{vaswani2017attention} without positional encodings.
As highlighted in Section~\ref{subsec:gen_w_fm}, both architectures are $S_N$-equivariant. More details can be found in Appendix~\ref{sec:supp_implementation_details}.

\paragraph{Reference Measure.}
Similarly to most existing linearized OT methods \citep{beier2021linear,nguyen2023linearfused,sarrazin2024linearized_lot_schmitzer,wang2013linear_lot}, we rely on
established solvers \citep{flamary2021pot} to approximate the
(finite) barycenter
$
    \hat{\rho} \in \argmin_{\rho \in \ProbTwo^N(\R^d)} \sum_{b=1}^{B_{\text{ref}}} \W^2_2(\rho, \hat{\nu}_b),$ where $\hat{\nu}_b \sim \hat{\bnu},
$
and we generally employ $B_{\text{ref}}=8$.

\paragraph{Source Metameasure.}
Due to the importance of the latent source for flow-based models \citep{chemseddine2026adapting_noise_data_quantile,hagemann2021stabilizing,pandey2025heavytailed_t_distribution_prior}, we investigate three potential choices for our source metameasure
$\hat{\bmu}$ parametrized via $\overboldarrow{\mu}$.
\revise{As i) a baseline \textbf{pure noise metameasure},} we employ independent and identically distributed (i.i.d.) entries $\mathbf{x}_{i,j} \sim \mathcal{N}(0, \sigma^2)$, $i=1, \ldots, N$ and $j=1, \ldots, d$, for $\mathbf{x} \sim \overboldarrow{\mu}$.
\revise{Here, $\sigma \sim \mathrm{Unif}[\underline{\sigma}, \overline{\sigma}]$ is sampled independently per cloud,} see supplementary Experiment~\ref{subsec:latent_pure_noise_ablation}.
However, to enable our lazy linear approach, we theoretically require 
concentration around $\mathbf{x}_{\hat{\rho}}$, see Corollary~\ref{corr:lazy_linear_justification}. 
Therefore, we alternatively use ii) our \textbf{barycentric noise metameasure} $\overboldarrow{\mu}$. \revise{In the MNIST and ShapeNet implementations, its entries are sampled as
$\mathbf{x}_{i,j}=(\mathbf{x}_{\hat\rho})_{i,j}+\sigma_0 Z_{i,j}U_{i,j}$, where $Z_{i,j}\sim \mathcal N(0,1)$ and $U_{i,j}\sim\operatorname{Unif}[0.5,1.5]$ independently. Thus, every coordinate receives Gaussian noise with an independently rescaled standard deviation.} See Experiments~\ref{subsec:noise2mnist_64_experiment} and  \ref{subsec:shapenet_experiment}.
In line with neural OT \citep{kim2024unpaired_neuralot_i2i, korotin2023neural_neuralot_i2i}, we further employ iii) 
\textbf{empirical metameasures} to investigate unsupervised couplings between data distributions, see Experiments \ref{subsec:circular_toy_experiment} and \ref{subsec:mnist2usps}.

\paragraph{Evaluation.}
Since we primarily focus on the case $d=2$ or $d=3$, we follow the evaluation protocol in point cloud generation \citep{haviv2025wasserstein,hui2025notsooptimal_point_cloud_flow_matching,yang2019pointflow_ot_nn}.
Thus, we are mainly interested in the accuracy of a nearest neighbour classifier based on the Chamfer (`Chamfer-NNA') or the Wasserstein distance (`OT-NNA').
Here, we generate $\{\mathbf{x}^{b}\}_{b=1}^B$ point clouds and compare them to $\{\mathbf{x}^{b}\}_{b=B+1}^{2B}$ real test set point clouds.
Based on the pairwise distance matrix $(d(\mathbf{x}^b, \mathbf{x}^{b'}))_{b, b'=1}^{2B}$, we then predict whether $\mathbf{x}^b$ is generated or real based on the labels of all other point clouds. Averaging results for all $b=1, \ldots, 2B$ gives the nearest neighbour accuracy, where Chamfer-NNA is obtained via the Chamfer distance $d(\mathbf{x}, \mathbf{x}') = \sum_{i=1}^N \min_{j} \|\mathbf{x}_i - \mathbf{x}'_j\|^2 + \sum_{i=1}^N \min_{j} \|\mathbf{x}'_i - \mathbf{x}_j\|^2$
and OT-NNA via $d(\mathbf{x}, \mathbf{x}')= \W^2_2(\operatorname{Vec2M}(\mathbf{x}), \operatorname{Vec2M}(\mathbf{x}'))$. The best NNA is $0.5$, and the worst one is $1.0$.

We evaluate each NNA with $512$ generated and $512$ real inputs and report the average NNA and its standard deviation over $5$ repetitions. Moreover, we perform this evaluation with $5$, $25$ and $125$ Euler steps for the generation to investigate the stability of numerical ODE integration.

 \subsection{Circular Point Clouds (2D, Baseline Model, Resolution: 30)}
\label{subsec:circular_toy_experiment}
We experiment on a synthetic 2D metameasure transport task to
illustrate the importance of the outer coupling.
Circular source clouds $\hat{\mu}\sim\hat{\bmu}$ are drawn as
$
\hat{\mu}
=
\frac{1}{30}\sum_{i=1}^{30}\delta_{(h,\,0)\,+\,r_s\,\varepsilon_i/\|\varepsilon_i\|}$,
for 
i.i.d.\ 
$
\varepsilon_i{\sim}\mathcal{N}(0,I_2)$
and $h\sim\mathrm{Unif}[-20, 20]$,
i.e., $N=30$ points placed on a ring of radius $r_s=0.5$ centred at the point $(h,0)$.
Target clouds $\hat{\nu}\sim\hat{\bnu}$ follow the same construction but with
radius $r_t=2.0$ and a vertical offset of $10$, i.e., vertically translated circles with bigger radius.
\revise{After $15{,}000$ steps (baseline model, Adam, step size $5\cdot10^{-4}$, $B=8$), only the WoW coupling yields nearly straight paths, see Figure~\ref{fig:gaussian_mixture_trajs}. Inner OT alone does not suffice.}

\begin{figure}[t]
    \centering
    \begin{subfigure}[t]{0.3\textwidth}
        \centering
        \fbox{\includegraphics[width=\textwidth, trim=8mm 45mm 8mm 45mm, clip]{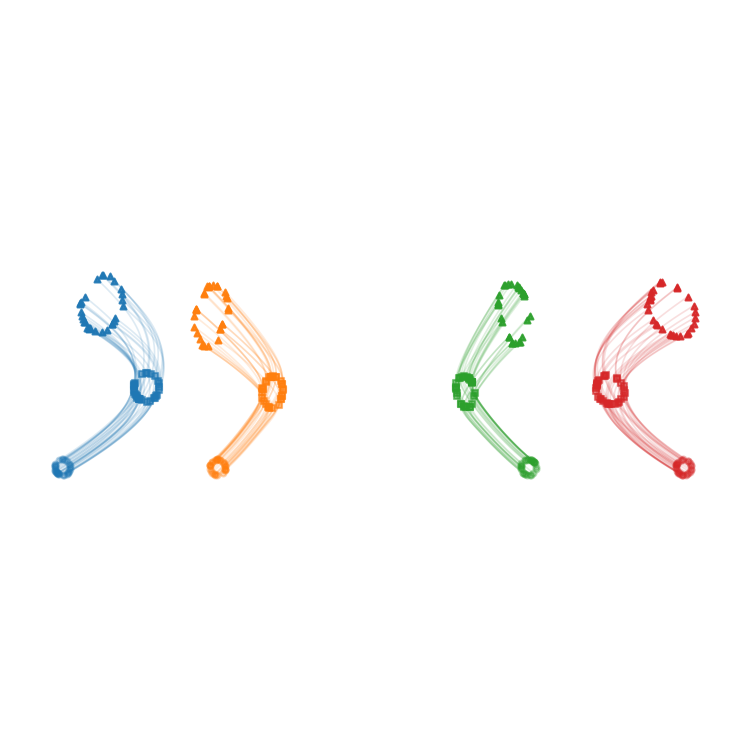}}
        \caption{\revise{$(\widehat{\operatorname{OP}}_{\mathrm{ind}},\,\widehat{\operatorname{IP}}_{\mathrm{ind}})$}}
        \label{subfig:gmm_ind_ind}
    \end{subfigure}
    \hfill
    \begin{subfigure}[t]{0.3\textwidth}
        \centering
        \fbox{\includegraphics[width=\textwidth, trim=8mm 45mm 8mm 45mm, clip]{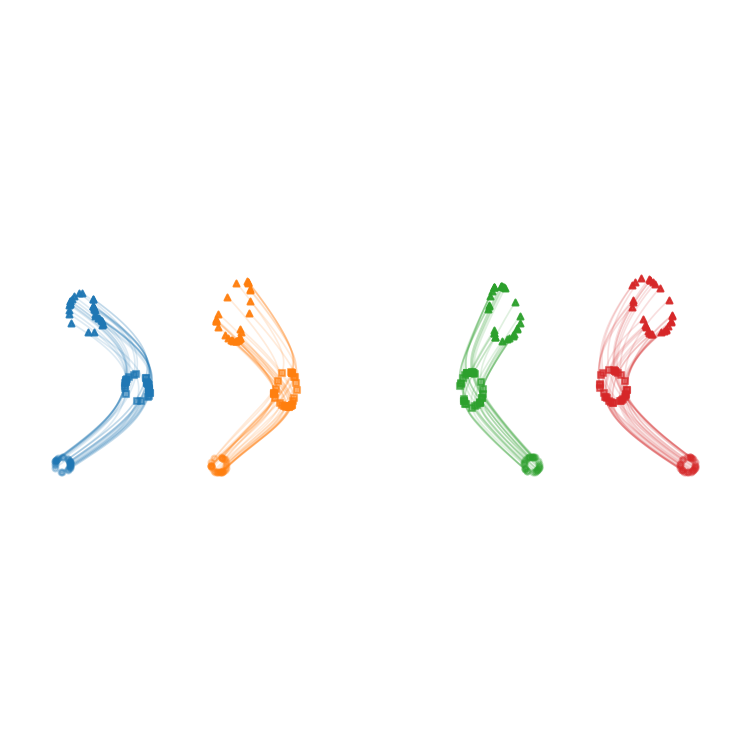}}
        \caption{\revise{$(\widehat{\operatorname{OP}}_{\mathrm{ind}},\,\widehat{\operatorname{IP}}_{\W})$}}
            \label{subfig:gmm_ind_ot}
    \end{subfigure}
    \hfill
    \begin{subfigure}[t]{0.3\textwidth}
        \centering
        \fbox{\includegraphics[width=\textwidth, trim=8mm 45mm 8mm 45mm, clip]{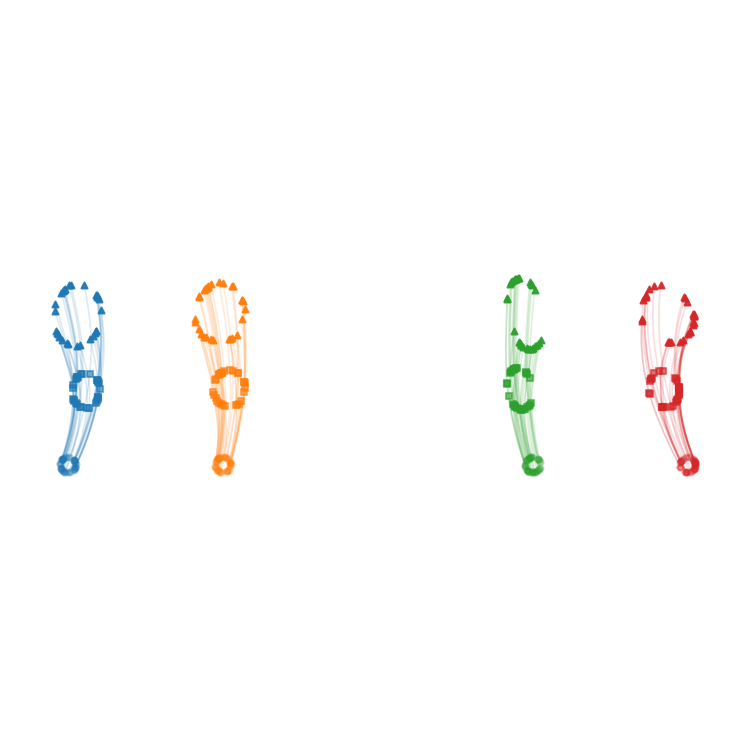}}
        \caption{\revise{$(\widehat{\operatorname{OP}}_{\bW},\,\widehat{\operatorname{IP}}_{\W})$}}
            \label{subfig:gmm_ot_ot}
    \end{subfigure}
    \caption{\revise{Learned paths from source circles (bottom) to larger, offset targets (top), $h\in\{-20,-10,10,20\}$ by colour. Only the WoW coupling yields straight paths.}}
    \label{fig:gaussian_mixture_trajs}
\end{figure}

\subsection{MNIST Point Clouds (2D, Baseline Model, Resolution: 64)}
\label{subsec:noise2mnist_64_experiment}
Similar to \citep{haviv2025wasserstein}, we experiment on point clouds obtained from MNIST images \citep{lecun1998gradient}.
Therefore, we view the normalized pixel values of each MNIST image as the intensities of a 2D histogram on $[0, 1]^2$, allowing us to sample $N=64$ points from it.
This results in fixed-size point clouds encoded as $\mathbf{x} \in \R^{2 \times 64}$. 
We train our baseline model for $50$ epochs using Adam (step size: $5 \cdot 10^{-4}$).
Evaluating our barycentric noise source \revise{with base scale $\sigma_0=0.1$ and coordinate-wise factors $U_{i,j}\in[0.5,1.5]$}, the results displayed in Table~\ref{tab:mnist_barycenter_noise_results} clearly show performance gains from OT-based plans on both levels. Even for this small dataset, we see a computational explosion of the \revise{coupling} time from $(\widehat{\operatorname{OP}}_{\bW}, \widehat{\operatorname{IP}}_{\mathrm{W}})$.
In contrast, our linearized 
$(\widehat{\operatorname{OP}}_{\text{\textbf{LLW}}}, \widehat{\operatorname{IP}}_{\text{LLW}})$
gives comparable or better results, while being significantly faster to train. 
While falling behind the other two OT solvers, $(\widehat{\operatorname{OP}}_{\text{\textbf{SW}}}, \widehat{\operatorname{IP}}_{\text{SW}})$ balances speed and performance, while not being tied to a specific source metameasure.
Notably, the independent coupling has the worst performance, but reduces the performance gap for enough Euler steps. 
\revise{Exact WoW and LLW remain within standard deviation here, consistent with LLW approximating concentrated sources.}

\begin{table}[t]
\centering
\setlength{\tabcolsep}{4.96pt} 
\renewcommand{\arraystretch}{0.72}
\resizebox{\linewidth}{!}{%
\begin{tabular}{llccc|ccc|c}
\toprule
 &  & \multicolumn{3}{c}{Chamfer-NNA $\downarrow$} & \multicolumn{3}{c}{OT-NNA $\downarrow$} & \revise{\shortstack{\textbf{Coupling/step}\\\textbf{(ms)}}} \\
\textbf{OP} & \textbf{IP} & Euler-5 & Euler-25 & Euler-125 & Euler-5 & Euler-25 & Euler-125 & \\
\midrule
$\widehat{\operatorname{OP}}_{\mathrm{ind}}$ 
& $\widehat{\operatorname{IP}}_{\mathrm{ind}}$ 
& \revise{$0.94{\scriptstyle\pm0.00}$} & $0.70{\scriptstyle\pm0.01}$ & \revise{$0.65{\scriptstyle\pm0.01}$}
& $0.95{\scriptstyle\pm0.01}$ & \revise{$0.65{\scriptstyle\pm0.01}$} & \revise{$0.61{\scriptstyle\pm0.01}$} 
& \revise{${<}0.1$} \\

$\widehat{\operatorname{OP}}_{\mathrm{ind}}$ 
& $\widehat{\operatorname{IP}}_{\mathrm{W}}$ 
& $0.77{\scriptstyle\pm0.01}$ & $0.69{\scriptstyle\pm0.03}$ & $0.66{\scriptstyle\pm0.02}$
& $0.71{\scriptstyle\pm0.01}$ & $0.64{\scriptstyle\pm0.01}$ & $0.61{\scriptstyle\pm0.02}$ 
& \revise{$\approx1.2$} \\

$\widehat{\operatorname{OP}}_{\mathrm{ind}}$ 
& $\widehat{\operatorname{IP}}_{\mathrm{SW}}$ 
& $0.79{\scriptstyle\pm0.01}$ & $\underline{0.63{\scriptstyle\pm0.01}}$ & $0.62{\scriptstyle\pm0.01}$
& $0.75{\scriptstyle\pm0.01}$ & $0.61{\scriptstyle\pm0.01}$ & $0.59{\scriptstyle\pm0.02}$ 
& \revise{$\approx0.4$} \\

$\widehat{\operatorname{OP}}_{\mathrm{ind}}$ 
& $\widehat{\operatorname{IP}}_{\mathrm{LLW}}$ 
& $0.75{\scriptstyle\pm0.01}$ & $0.64{\scriptstyle\pm0.01}$ & $0.62{\scriptstyle\pm0.01}$
& $0.69{\scriptstyle\pm0.01}$ & $\underline{0.59{\scriptstyle\pm0.01}}$ & $0.58{\scriptstyle\pm0.02}$ 
& \revise{${<}0.1$} \\
\arrayrulecolor{gray!200}
\hdashline[0.5pt/1.5pt]
\arrayrulecolor{black}
$\widehat{\operatorname{OP}}_{\text{\textbf{W}}}$ 
& $\widehat{\operatorname{IP}}_{\mathrm{W}}$ 
& $\underline{0.72{\scriptstyle\pm0.01}}$ & $0.65{\scriptstyle\pm0.02}$ & \revise{$\underline{0.62{\scriptstyle\pm0.02}}$}
& \revise{$0.66{\scriptstyle\pm0.01}$} & $\mathbf{0.59{\scriptstyle\pm0.03}}$ & $\underline{0.58{\scriptstyle\pm0.01}}$ 
& \revise{$\approx13.7$} \\

\revise{$\widehat{\operatorname{OP}}_{\mathrm{BW}}$}
& \revise{$\widehat{\operatorname{IP}}_{\mathrm{W}}$}
& \revise{$0.73{\scriptstyle\pm0.01}$} & \revise{$0.66{\scriptstyle\pm0.01}$} & \revise{$0.65{\scriptstyle\pm0.01}$}
& \revise{$\underline{0.65{\scriptstyle\pm0.01}}$} & \revise{$0.61{\scriptstyle\pm0.01}$} & \revise{$0.60{\scriptstyle\pm0.02}$}
& \revise{$\approx5.5$} \\

\revise{Euclidean}
& \revise{$\widehat{\operatorname{IP}}_{\mathrm{W}}$}
& \revise{$0.74{\scriptstyle\pm0.00}$} & \revise{$0.64{\scriptstyle\pm0.02}$} & \revise{$\mathbf{0.62{\scriptstyle\pm0.01}}$}
& \revise{$0.68{\scriptstyle\pm0.01}$} & \revise{$0.60{\scriptstyle\pm0.03}$} & \revise{$0.59{\scriptstyle\pm0.01}$}
& \revise{$\approx1.6$} \\

$\widehat{\operatorname{OP}}_{\text{\textbf{SW}}}$ 
& $\widehat{\operatorname{IP}}_{\mathrm{SW}}$ 
& $0.79{\scriptstyle\pm0.02}$ & \revise{$0.65{\scriptstyle\pm0.01}$} & $0.64{\scriptstyle\pm0.01}$
& $0.76{\scriptstyle\pm0.01}$ & $0.63{\scriptstyle\pm0.02}$ & \revise{$0.61{\scriptstyle\pm0.01}$} 
& \revise{$\approx3.1$} \\

$\widehat{\operatorname{OP}}_{\text{\textbf{LLW}}}$ 
& $\widehat{\operatorname{IP}}_{\mathrm{LLW}}$ 
& $\mathbf{0.72{\scriptstyle\pm0.01}}$ & $\mathbf{0.63{\scriptstyle\pm0.02}}$ & \revise{$0.62{\scriptstyle\pm0.02}$}
& $\mathbf{0.65{\scriptstyle\pm0.00}}$ & $0.59{\scriptstyle\pm0.02}$ & $\mathbf{0.57{\scriptstyle\pm0.02}}$ 
& \revise{${<}0.1$} \\

\bottomrule
\end{tabular}
}
\caption{\revise{MNIST results for barycentric noise ($B=8$, $N=64$). Bold and underline mark the two best scores before rounding. Coupling times are CPU-timed per step, see Appendix~\ref{subsec:ablation_runtime}.}}
\label{tab:mnist_barycenter_noise_results}
\end{table}
\revise{
Under the same protocol, Euclidean minibatch OT improves over independent pairing at five steps but trails W, BW, and LLW in OT-NNA. The gap largely vanishes with more Euler steps, locating the main benefit in few-step trajectories. Exact W and LLW differ chiefly in online cost.
While Table~\ref{tab:mnist_barycenter_noise_results} fixes $B=8$, minibatch OT benefits from larger batches \citep{tong2023improving_minibatch}, where exact WoW scales quadratically (Appendix~\ref{subsec:ablation_runtime}) but LLW stays negligible. At $B=64$ (140k steps, otherwise same setting), LLW reaches a five-step OT-NNA of $0.61{\scriptstyle\pm0.01}$, outperforming all $B=8$ couplings and the slower, equally trained $(\text{Euclidean}, \widehat{\operatorname{IP}}_{\W})$ ($0.65{\scriptstyle\pm0.02}$).}

\begin{figure}[t]
    \centering
    \begin{subfigure}[b]{0.32\textwidth}
        \centering
\includegraphics[width=\textwidth]{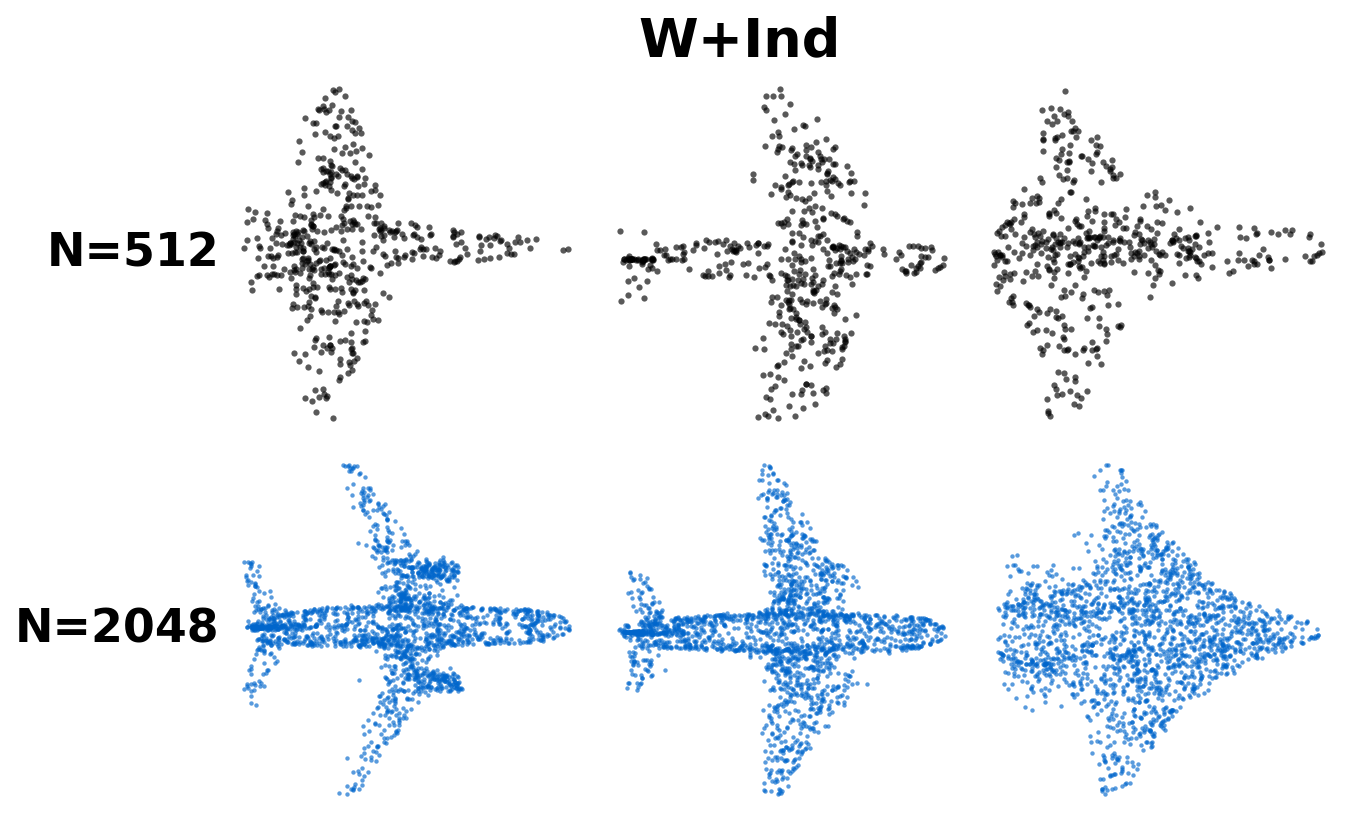}
        \caption{$(\widehat{\operatorname{OP}}_{\text{\textbf{ind}}},\widehat{\operatorname{IP}}_{\mathrm{W}})$}
        \label{subfig:wfm_planes}
    \end{subfigure}
    \hfill
    \begin{subfigure}[b]{0.32\textwidth}
        \centering
\includegraphics[width=\textwidth]{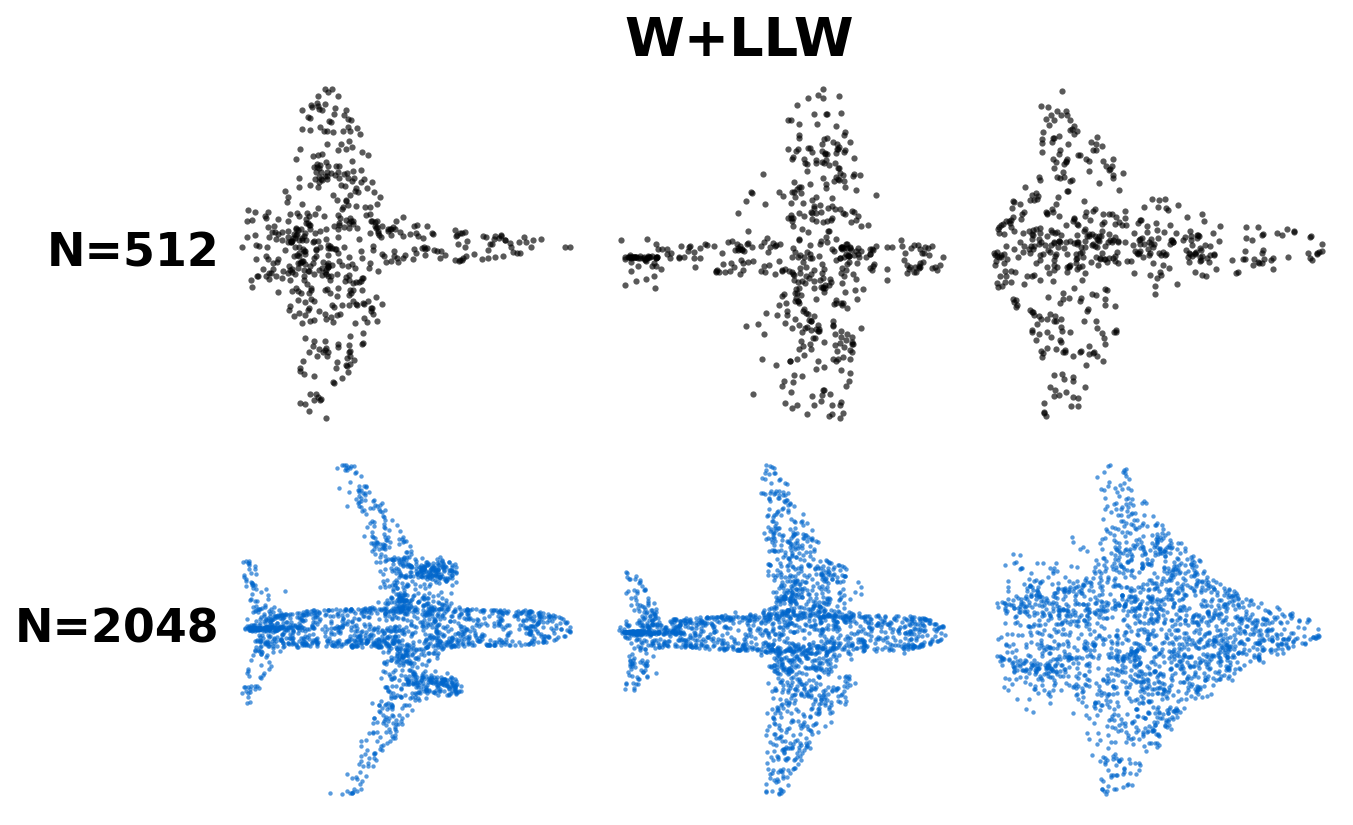}
        \caption{\revise{$(\widehat{\operatorname{OP}}_{\text{\textbf{LLW}}},\widehat{\operatorname{IP}}_{\mathrm{W}})$}}
                \label{subfig:llw_planes}
    \end{subfigure}
    \hfill
    \begin{subfigure}[b]{0.32\textwidth}
        \centering
\includegraphics[width=\textwidth]{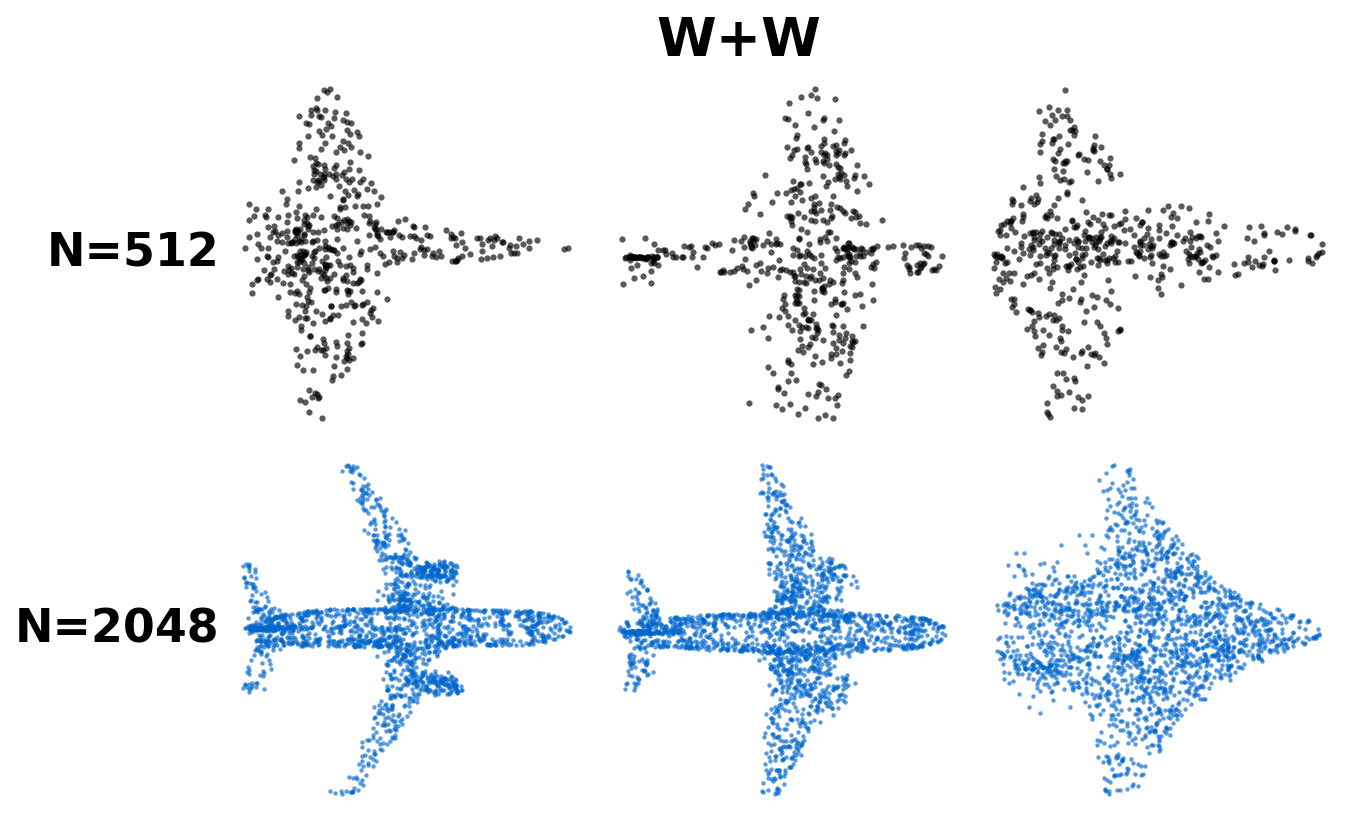}
        \caption{\revise{$(\widehat{\operatorname{OP}}_{\bW},\widehat{\operatorname{IP}}_{\mathrm{W}})$}}
            \label{subfig:w_planes}
    \end{subfigure}
    \caption{\revise{ShapeNet planes from identical barycentric-source samples after $25$ Euler steps. The outer couplings induce similar endpoints.}}
    \label{fig:generated_planes_comparison}
\end{figure}

\subsection{ShapeNet Airplanes (3D, Transformer, Resolution: 512)}
\label{subsec:shapenet_experiment}
\begin{wraptable}{r}{0.47\textwidth}
\vspace{-1.2em}
\centering
\setlength{\tabcolsep}{7pt}
\renewcommand{\arraystretch}{0.8}
\resizebox{\linewidth}{!}{\color{black}%
\begin{tabular}{lcc|c}
\toprule
\textbf{Outer matching} & \textbf{Chamfer-NNA} $\downarrow$ & \textbf{OT-NNA} $\downarrow$ & \textbf{Coupling/step} \\
 & \multicolumn{2}{c|}{Euler-5} & (s) \\
\midrule
$\widehat{\operatorname{OP}}_{\mathrm{ind}}$ & $0.742{\scriptscriptstyle\pm.02}$ & $0.741{\scriptscriptstyle\pm.01}$ & $\approx0.14$ \\
Euclidean & $0.738{\scriptscriptstyle\pm.01}$ & $0.715{\scriptscriptstyle\pm.01}$ & $\approx0.15$ \\
$\widehat{\operatorname{OP}}_{\mathrm{BW}}$ & $\underline{0.702{\scriptscriptstyle\pm.01}}$ & $\underline{0.671{\scriptscriptstyle\pm.01}}$ & $\approx0.16$ \\
$\widehat{\operatorname{OP}}_{\bW}$ & $\mathbf{0.695{\scriptscriptstyle\pm.02}}$ & $\mathbf{0.661{\scriptscriptstyle\pm.01}}$ & $\approx1.31$ \\
$\widehat{\operatorname{OP}}_{\mathrm{SW}}$ & $0.739{\scriptscriptstyle\pm.01}$ & $0.724{\scriptscriptstyle\pm.01}$ & $\approx0.15$ \\
$\widehat{\operatorname{OP}}_{\mathrm{LLW}}$ & $0.713{\scriptscriptstyle\pm.01}$ & $0.685{\scriptscriptstyle\pm.01}$ & $\approx0.15$ \\
\bottomrule
\end{tabular}%
}
\caption{ShapeNet-Airplane controlled outer-matching comparison with $\widehat{\operatorname{IP}}_{\W}$ fixed. Training and coupling timings use $B=8$ and $N=512$. NNA uses five Euler steps.}
\label{tab:shapenetpart_Airplane_pt_barycenter_noise_std0.3_N512}
\vspace{-1em}
\end{wraptable}
\leavevmode Next, we investigate 3D plane-shaped point clouds in the form of planes obtained from the ShapeNet dataset \citep{yi2016scalable_shapenet}. 
We use a standard transformer and resolution $N=512$ 
resulting in min-max normalized point clouds encoded as $\mathbf{x} \in [0, 1]^{3 \times 512} \subset \R^{3 \times 512}$. 
Again, we employ barycentric noise \revise{with base scale $\sigma_0=0.3$ and coordinate-wise factors $U_{i,j}\in[0.5,1.5]$} and train each model for 500 epochs (Adam, step size: $5 \cdot 10^{-4}$). We evaluate different outer transport plans while fixing the inner transport plan as $\operatorname{IP}_{\W}$.

Although trained at $N=512$, our discretization-invariant transformer generalizes to high-resolution settings, see Figure~\ref{fig:generated_planes_comparison} for $N=2048$.
\revise{After $25$ Euler steps, the independent, LLW, and exact W outer couplings induce similar endpoints.}

\leavevmode\revise{Table~\ref{tab:shapenetpart_Airplane_pt_barycenter_noise_std0.3_N512} compares outer plans at the training batch size $B=8$ under exact $\widehat{\operatorname{IP}}_{\W}$, EMA, a cosine schedule, and equal budgets. NNA clouds are downsampled to $N=256$ \citep{piening2025slicing_wow}. Exact W performs best, while BW and LLW retain most of its five-step gain and LLW scales like flattened Euclidean matching. SW, with eight slices in $d=3$, only matches flattened Euclidean matching here.}

\subsection{\revise{QM9 Molecular Sets (7D, Baseline Model, Resolution: 9)}}
\label{subsec:qm9_experiment}
\begin{wraptable}{r}{0.47\textwidth}
\vspace{-1.2em}
\centering
\setlength{\tabcolsep}{8pt}
\renewcommand{\arraystretch}{0.8}
\resizebox{\linewidth}{!}{\color{black}%
\begin{tabular}{lcc|cc}
\toprule
\textbf{Outer matching} & \multicolumn{2}{c|}{\textbf{Chamfer-NNA} $\downarrow$} & \multicolumn{2}{c}{\textbf{OT-NNA} $\downarrow$} \\
 & Euler-5 & Euler-25 & Euler-5 & Euler-25 \\
\midrule
Euclidean & $0.825{\scriptscriptstyle\pm.00}$ & $0.685{\scriptscriptstyle\pm.01}$ & $0.850{\scriptscriptstyle\pm.00}$ & $0.659{\scriptscriptstyle\pm.01}$ \\
$\widehat{\operatorname{OP}}_{\mathrm{LLW}}$ & $\mathbf{0.752{\scriptscriptstyle\pm.01}}$ & $\mathbf{0.671{\scriptscriptstyle\pm.01}}$ & $\mathbf{0.754{\scriptscriptstyle\pm.01}}$ & $\mathbf{0.652{\scriptscriptstyle\pm.01}}$ \\
\bottomrule
\end{tabular}%
}
\caption{QM9 fixed-$N=9$ outer-matching comparison with exact inner plans \revise{(3 evaluation repeats, see Appendix~\ref{subsec:controlled_details})}.}
\label{tab:qm9_outer_comparison}
\vspace{-1em}
\end{wraptable}
\leavevmode\revise{%
Beyond resampled surfaces, we use QM9 geometries \citep{ramakrishnan2014qm9} with nine heavy atoms. Centred 3D coordinates and one-hot types in $\{\mathrm C,\mathrm N,\mathrm O,\mathrm F\}$ give shuffled sets in $\R^{9\times7}$. We use the baseline network, $B=64$, exact $\widehat{\operatorname{IP}}_{\W}$, and barycentric noise with $\sigma_0=0.1$ and $U_{i,j}\equiv1$. LLW improves five-step quality over Euclidean matching, with a narrower gap at 25 steps. BW was prohibitively slow at $B=64$ (${\approx}6$\,s vs.\ $0.1$\,s per step for LLW, i.e., ${>}1$ day vs.\ $27$\,min). Offline superset couplings such as NSOT (Appendix~\ref{sec:related_work}) do not apply here, since molecules cannot be upsampled by adding atoms.}

\subsection{\texorpdfstring{\revise{MNIST$\to$USPS Transport (2D, Transformer, Multi-Resolution: [64, 128, 256])}}{MNIST to USPS Transport (2D, Transformer, Multi-Resolution: [64, 128, 256])}}
\label{subsec:mnist2usps_main}
\begin{wrapfigure}{r}{0.42\textwidth}
\vspace{-1.6em}
\centering
\includegraphics[width=\linewidth, trim=0 0 92.6cm 0, clip]{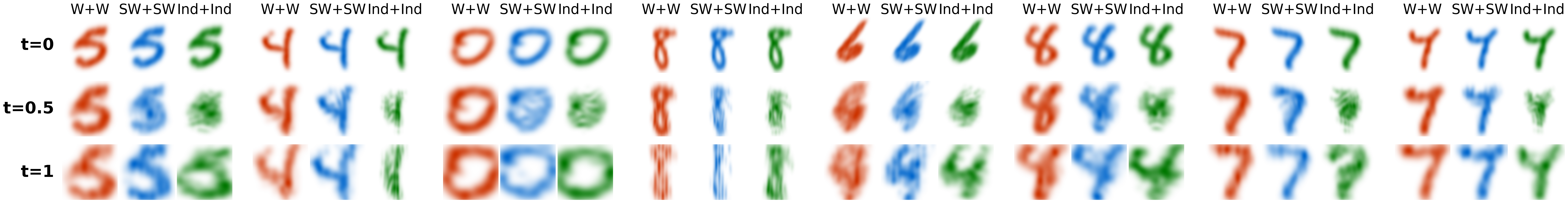}
\caption{\revise{Estimated densities along the learned MNIST$\to$USPS paths.}}
\label{fig:mnist_usps_main}
\vspace{-1.2em}
\end{wrapfigure}
\leavevmode\revise{%
To illustrate differences in the induced flow maps, we transport point clouds of handwritten MNIST digits to the bolder digit style of USPS \citep{uspsdataset_hull1994}, i.e., between two empirical metameasures. Figure~\ref{fig:mnist_usps_main} shows kernel density estimates of the learned paths. At $t=0.5$, $(\widehat{\operatorname{OP}}_{\bW},\widehat{\operatorname{IP}}_{\W})$ and $(\widehat{\operatorname{OP}}_{\mathrm{SW}},\widehat{\operatorname{IP}}_{\mathrm{SW}})$ preserve the digit structure, while independent coupling yields amorphous densities, hinting that the sliced flow map is closer to the WoW one. The complete experiment is given in Appendix~\ref{subsec:mnist2usps}.}

\section{Conclusion}
\label{sec:conclusion}

\revise{We extended multisample flow matching \citep{pooladian2023multisample_couplings} from Euclidean space to the hierarchical WoW space. The WoW continuity equation shows that any admissible nested outer and inner coupling induces a valid metameasure flow whose kinetic energy is bounded by the coupling cost: exact WoW plans attain the minimal energy and straight paths, while sliced and lazy linear plans trade exactness for scalability. Our experiments compare couplings within fixed architectures rather than across differently designed point cloud generators. This isolates the coupling, showing that permutation-invariant geometry, not reordering alone, improves few-step generation, but leaves comparisons across generators open. Further limitations are the restriction to fixed-size empirical measures in low dimensions, the quadratic batch-size scaling of exact WoW couplings, and the reliance of LLW on a source concentrated around the barycenter. Natural extensions include conditional and nested measures \citep{beiglbock2025brenier_wow_theory,chemseddine2025conditional} and parameterizations beyond empirical distributions \citep{haviv2025wasserstein,jiang2025bureswasserstein_flow_matching}.}

\newpage
\revise{\section*{Acknowledgments}
M.P.\ acknowledges funding from the German Research Foundation (DFG, GRK2260 BIOQIC, project 289347353). Further, M.P.\ thanks Gregor Kornhardt for fruitful discussions.
G.S.\ and R.D.\ acknowledge funding from the Deutsche Forschungsgemeinschaft (DFG, German Research Foundation) under Germany's Excellence Strategy -- The Berlin Mathematics Research Center MATH+ (EXC-2046/1, EXC-2046/2, project ID: 390685689). 
Further, G.S. thanks CNRS for support within the CNRS Ambassador
Fellowship.}

\bibliographystyle{plainnat}
\bibliography{main}

@book{PeyreCuturi2019,
  author    = {Gabriel Peyr{\'{e}} and Marco Cuturi},
  title     = {Computational Optimal Transport: With Applications to Data Science},
  publisher = {Now Publishers},
  series    = {Foundations and Trends® in Machine Learning},
  volume    = {11},
  pages     = {355--607},
  year      = {2019},
  doi       = {10.1561/2200000073},
  isbn      = {978-1-68083-550-2}
}

@article{beier2021linear,
  title={On a linear {G}romov--{W}asserstein distance},
  author={Beier, Florian and Beinert, Robert and Steidl, Gabriele},
  journal={IEEE Transactions on Image Processing},
  volume={31},
  pages={7292--7305},
  year={2022},
  publisher={IEEE}
}

@book{boumal2023,
  title     = {An Introduction to Optimization on Smooth Manifolds},
  author    = {Boumal, Nicolas},
  year      = {2023},
  publisher = {Cambridge University Press}
}

@inproceedings{bonet2025busemann,
  title={Busemann functions in the {Wasserstein} space: existence, closed-forms, and applications to slicing},
  author={Clément Bonet and Elsa Cazelles and Lucas Drumetz and Nicolas Courty},
  booktitle={Proceedings of the AISTATS'26},
  volume={300},
  pages={1972--1980},
  year={2026},
  publisher={PMLR}
}

@book{Santambrogio2015,
  author       = {Santambrogio, Filippo},
  title        = {Optimal Transport for Applied Mathematicians: Calculus of Variations, PDEs and Modeling},
  hideseries       = {Progress in Nonlinear Differential Equations and Their Applications},
  hidevolume       = {87},
  publisher    = {Birkh{\"a}user},
  address = {Cham},
  year         = {2015},
  hideisbn         = {978-3-319-20828-2},
}

@inproceedings{nadjahi2020statistical,
  title = {Statistical and topological properties of sliced probability divergences},
  author = {Nadjahi, Kimia and Durmus, Alain and Chizat, L{\'e}na{\"i}c and Kolouri, Soheil and Shahrampour, Shahin and Simsekli, Umut},
  booktitle = {Advances in Neural Information Processing Systems},
  volume = {33},
  pages = {15657--15669},
  year = {2020},
  publisher = {Curran Associates},
}

@phdthesis{bonnotte2013unidimensional,
  author       = {Nicolas Bonnotte},
  title        = {Unidimensional and Evolution Methods for Optimal Transportation},
  school       = {Université Paris Sud–Paris XI},
  year         = {2013},
  address      = {Orsay, France},
  note         = {PhD Thesis}
}

@misc{sato2020fast,
  title        = {Fast and robust comparison of probability measures in heterogeneous spaces},
  author       = {Ryoma Sato and Marco Cuturi and Makoto Yamada and Hisashi Kashima},
  note      = {{arXiv:2002.01615}},
  year         = {2020},
}

@article{flamary2021pot,
  author       = {Flamary, R{\'e}mi and Courty, Nicolas and Gramfort, Alexandre and Alaya, Mokhtar Z. and Boisbunon, Aur{\'e}lie and Chambon, Stanislas and Chapel, Laetitia and Corenflos, Adrien and Fatras, Kilian and Fournier, Nemo and Gautheron, L{\'e}o and Gayraud, Nathalie T. H. and Janati, Hicham and Rakotomamonjy, Alain and Redko, Ievgen and Rolet, Antoine and Schutz, Antony and Seguy, Vivien and Sutherland, Danica J. and Tavenard, Romain and Tong, Alexander and Vayer, Titouan},
  title        = {{POT}: {P}ython optimal transport},
  journal      = {Journal of Machine Learning Research},
  volume       = {22},
  number       = {78},
  pages        = {1--8},
  year         = {2021},
  note         = {Software available at \url{https://pythonot.github.io/}},
}

@inproceedings{nguyen2025sotdd,
  author    = {Nguyen, Khai and Nguyen, Hai and Pham, Tuan and Ho, Nhat},
  title     = {Lightspeed geometric dataset distance via sliced optimal transport},
  booktitle = {Proceedings of the ICML'25},
  year      = {2025},
    publisher={OpenReview.net}
}

@article{
piening2025slicing_gaussian,
title={Slicing the {G}aussian Mixture {W}asserstein Distance},
author={Piening, Moritz and Beinert, Robert},
journal={Transactions on Machine Learning Research},
year={2025},
publisher={OpenReview.net}
}

@article{nguyen2023linearfused,
  title={On a linear fused {Gromov--W}asserstein distance for graph structured data},
  author={Nguyen, Dai Hai and Tsuda, Koji},
  journal={Pattern Recognition},
  volume={138},
  pages={109351},
  year={2023},
  publisher={Elsevier}
}

@inproceedings{
bonet2025flowing,
title={Flowing datasets with {W}asserstein over {W}asserstein gradient flows},
author={Cl{\'e}ment Bonet and Christophe Vauthier and Anna Korba},
booktitle={Proceedings of the ICML'25},
year={2025},
publisher={OpenReview.net}
}

@article{wald2025flow,
  title={Flow Matching: {M}arkov kernels, stochastic processes and transport plans},
  author={Wald, Christian and Steidl, Gabriele},
  journal={Variational and Information Flows in Machine Learning and Optimal Transport},
  pages={185--254},
  year={2025},
  publisher={Springer}
}

@inproceedings{
haviv2025wasserstein,
title={Wasserstein Flow Matching: {G}enerative Modeling Over Families of Distributions},
author={Doron Haviv and Aram-Alexandre Pooladian and Dana Pe'er and Brandon Amos},
booktitle={Proceedings of the ICML'25},
year={2025},
publisher={OpenReview.net}
}

@article{ramakrishnan2014qm9,
  title={Quantum chemistry structures and properties of 134 kilo molecules},
  author={Ramakrishnan, Raghunathan and Dral, Pavlo O. and Rupp, Matthias and von Lilienfeld, O. Anatole},
  journal={Scientific Data},
  volume={1},
  number={1},
  pages={140022},
  year={2014},
  month=aug,
  issn={2052-4463},
  doi={10.1038/sdata.2014.22},
  url={https://doi.org/10.1038/sdata.2014.22},
  publisher={Springer Science and Business Media LLC}
}

@inproceedings{
atanackovic2025meta,
title={Meta Flow Matching: {I}ntegrating Vector Fields on the {W}asserstein Manifold},
author={Lazar Atanackovic and Xi Zhang and Brandon Amos and Mathieu Blanchette and Leo J Lee and Yoshua Bengio and Alexander Tong and Kirill Neklyudov},
booktitle={Proceedings of the ICLR'25},
year={2025},
publisher={OpenReview.net}
}

@inproceedings{piening2025novel,
  title={A Novel Sliced Fused {G}romov-{W}asserstein Distance},
  author={Piening, Moritz and Beinert, Robert},
  booktitle={Proceedings of the AAAI'26},
  volume={40},
  pages={24828--24836},
  year={2026},
  publisher={AAAI Press}
}

@inproceedings{yang2019pointflow_ot_nn,
  title={{PointFlow}: {3D} point cloud generation with continuous normalizing flows},
  author={Yang, Guandao and Huang, Xun and Hao, Zekun and Liu, Ming-Yu and Belongie, Serge and Hariharan, Bharath},
  booktitle={Proceedings of the ICCV'19},
  pages={4541--4550},
    publisher={IEEE},  
  year={2019}
}

@article{lecun1998gradient,
  author       = {LeCun, Yann and Bottou, L\'eon and Bengio, Yoshua and Haffner, Patrick},
  title        = {Gradient-Based Learning Applied to Document Recognition},
  journal      = {Proceedings of the IEEE},
  year         = {1998},
  volume       = {86},
  number       = {11},
  pages        = {2278--2324}
}

@inproceedings{nguyen2021pointcloud_sliced,
  title={Point-set distances for learning representations of {3D} point clouds},
  author={Nguyen, Trung and Pham, Quang-Hieu and Le, Tam and Pham, Tung and Ho, Nhat and Hua, Binh-Son},
  booktitle={Proceedings of the ICCV'21},
  pages={10458--10467},
  year={2021},
  publisher={IEEE}
}

@inproceedings{lin2020fixed,
  title={Fixed-support {W}asserstein barycenters: Computational hardness and fast algorithm},
  author={Lin, Tianyi and Ho, Nhat and Chen, Xi and Cuturi, Marco and Jordan, Michael},
  booktitle={Advances in Neural Information Processing Systems},
  volume={33},
  pages={5368--5380},
  year={2020},
  publisher={Curran Associates}
}

@inproceedings{piening2025slicing_wow,
  title={Slicing {W}asserstein over {W}asserstein via Functional Optimal Transport},
  author={Piening, Moritz and Beinert, Robert},
  booktitle={Proceedings of the ICLR'26},
  year={2026},
  publisher={OpenReview.net}
}

@inproceedings{lipman2023flowmatching_fm,
  title     = {Flow Matching for Generative Modeling},
  author={Lipman, Yaron and Chen, Ricky T. Q. and Ben-Hamu, Heli and Nickel, Maximilian and Le, Matt},
  booktitle = {Proceedings of the ICLR'23},
  year      = {2023},
  publisher={OpenReview.net}
}

@inproceedings{albergo2023building_fm,
  title={Building Normalizing Flows with Stochastic Interpolants},
  author={Albergo, Michael and Vanden-Eijnden, Eric},
  booktitle={Proceedings of the ICLR'23},
  year={2023},
  publisher={OpenReview.net}
}

@inproceedings{liuflow_fm_recti,
  title={Flow Straight and Fast: Learning to Generate and Transfer Data with Rectified Flow},
  author={Liu, Xingchao and Gong, Chengyue and Liu, Qiang},
  booktitle={Proceedings of the ICLR'23},
  year={2023},
  publisher={OpenReview.net}
}

@inproceedings{pooladian2023multisample_couplings,
  title={Multisample Flow Matching: Straightening Flows with Minibatch Couplings},
  author={Pooladian, Aram-Alexandre and Ben-Hamu, Heli and Domingo-Enrich, Carles and Amos, Brandon and Lipman, Yaron and Chen, Ricky T. Q.},
  booktitle = {Proceedings of the ICML'23},
  year      = {2023},
  publisher = {OpenReview.net}
}

@article{tong2023improving_minibatch,
  title={Improving and generalizing flow-based generative models with minibatch optimal transport},
  author={Tong, Alexander and Fatras, Kilian and Malkin, Nikolay and Huguet, Guillaume and Zhang, Yanlei and Rector-Brooks, Jarrid and Wolf, Guy and Bengio, Yoshua},
  journal   = {Transactions on Machine Learning Research},
  year      = {2024},
  publisher = {OpenReview.net}
}

@inproceedings{mousavi2025flow_semidiscre,
  title={Flow matching with semidiscrete couplings},
  author={Mousavi-Hosseini, Alireza and Zhang, Stephen Y and Klein, Michal and Cuturi, Marco},
  booktitle={Proceedings of the ICLR'26},
  year={2026},
  publisher={OpenReview.net}
}

@article{ho2020denoising,
  title={Denoising diffusion probabilistic models},
  author={Ho, Jonathan and Jain, Ajay and Abbeel, Pieter},
  journal={Advances in Neural Information Processing Systems},
  volume={33},
  pages={6840--6851},
  year={2020},
  publisher={Curran Associates}
}

@article{chemseddine2025conditional,
  title={Conditional {W}asserstein distances with applications in {B}ayesian {OT} flow matching},
  author={Chemseddine, Jannis and Hagemann, Paul and Steidl, Gabriele and Wald, Christian},
  journal={Journal of Machine Learning Research},
  volume={26},
  number={141},
  pages={1--47},
  year={2025},
  publisher = {Microtome Publishing}
}

@inproceedings{
liu2025expected,
title={Expected Sliced Transport Plans},
author={Xinran Liu and Rocio Diaz Martin and Yikun Bai and Ashkan Shahbazi and Matthew Thorpe and Akram Aldroubi and Soheil Kolouri},
booktitle={Proceedings of the ICLR'25},
year={2025},
publisher={OpenReview.net}
}

@inproceedings{kim2024unpaired_neuralot_i2i,
  title     = {Unpaired Image-to-Image Translation via Neural {S}chr{\"o}dinger Bridge},
  author    = {Beomsu Kim and Gihyun Kwon and Kwanyoung Kim and Jong Chul Ye},
  booktitle = {Proceedings of the ICLR'24},
  year      = {2024},
  publisher = {OpenReview.net}
}

@inproceedings{korotin2023neural_neuralot_i2i,
  title     = {Neural Optimal Transport},
  author    = {Alexander Korotin and Daniil Selikhanovych and Evgeny Burnaev},
  booktitle = {Proceedings of the ICLR'23},
  year      = {2023},
  publisher = {OpenReview.net}
}

@article{pinzi2025nested_wow,
  title={Nested superposition principle for random measures and the geometry of the {W}asserstein on {W}asserstein space},
  author={Pinzi, Alessandro and Savar{\'e}, Giuseppe},
  journal={arXiv preprint arXiv:2510.07523},
  year={2025}
}

@book{AmbrosioGigliSavare2005,
  author    = {Ambrosio, Luigi and Gigli, Nicola and Savar{\'e}, Giuseppe},
  title     = {Gradient Flows: In Metric Spaces and in the Space of Probability Measures},
  year      = {2005},
  publisher = {Springer Science \& Business Media}
}

@inproceedings{lee2019settransformer,
  title={Set transformer: {A} framework for attention-based permutation-invariant neural networks},
  author={Lee, Juho and Lee, Yoonho and Kim, Jungtaek and Kosiorek, Adam and Choi, Seungjin and Teh, Yee Whye},
  booktitle={Proceedings of the ICML'19},
  pages={3744--3753},
  year={2019},
  organization={PMLR}
}

@article{guo2021pct_point_cloud_transformer,
  title={{PCT}: {P}oint cloud transformer},
  author={Guo, Meng-Hao and Cai, Jun-Xiong and Liu, Zheng-Ning and Mu, Tai-Jiang and Martin, Ralph R and Hu, Shi-Min},
  journal={Computational Visual Media},
  volume={7},
  number={2},
  pages={187--199},
  year={2021},
  publisher={Springer}
}

@inproceedings{qi2017pointnet,
  title={{PointNet}: {D}eep learning on point sets for {3D} classification and segmentation},
  author={Qi, Charles R and Su, Hao and Mo, Kaichun and Guibas, Leonidas J},
  booktitle={Proceedings of the CVPR'17},
  pages={652--660},
  publisher={IEEE},
  year={2017}
}

@inproceedings{
hui2025notsooptimal_point_cloud_flow_matching,
title={Not-So-Optimal Transport Flows for {3D} Point Cloud Generation},
author={Ka-Hei Hui and Chao Liu and Xiaohui Zeng and Chi-Wing Fu and Arash Vahdat},
booktitle={Proceedings of the ICLR'25},
publisher={OpenReview.net},
year={2025}
}

@article{klein2023equivariant_flow_matching,
  title={Equivariant flow matching},
  author={Klein, Leon and Kr{\"a}mer, Andreas and No{\'e}, Frank},
  journal={Advances in Neural Information Processing Systems},
  publisher={Curran Associates},
  volume={36},
  pages={59886--59910},
  year={2023}
}

@article{wang2013linear_lot,
  title={A linear optimal transportation framework for quantifying and visualizing variations in sets of images},
  author={Wang, Wei and Slep{\v{c}}ev, Dejan and Basu, Saurav and Ozolek, John A and Rohde, Gustavo K},
  journal={International Journal of Computer Vision},
  volume={101},
  number={2},
  pages={254--269},
  year={2013},
  publisher={Springer}
}

@article{sarrazin2024linearized_lot_schmitzer,
  title={Linearized optimal transport on manifolds},
  author={Sarrazin, Cl{\'e}ment and Schmitzer, Bernhard},
  journal={SIAM Journal on Mathematical Analysis},
  volume={56},
  number={4},
  pages={4970--5016},
  year={2024},
  publisher={SIAM}
}

@article{yi2016scalable_shapenet,
  title={A scalable active framework for region annotation in {3D} shape collections},
  author={Yi, Li and Kim, Vladimir G and Ceylan, Duygu and Shen, I-Chao and Yan, Mengyan and Su, Hao and Lu, Cewu and Huang, Qixing and Sheffer, Alla and Guibas, Leonidas},
  journal={ACM Transactions on Graphics (ToG)},
  volume={35},
  number={6},
  pages={1--12},
  year={2016},
  publisher={ACM New York, NY, USA}
}

@inproceedings{chen2024flow_riemannian_flow_matching,
  title={Flow Matching on General Geometries},
  author={Chen, Ricky T. Q. and Lipman, Yaron},
  booktitle={Proceedings of the ICLR'24},
  year={2024},
  publisher={OpenReview.net}
}

@article{yang2023diffusion_diffusion_review,
  title={Diffusion models: {A} comprehensive survey of methods and applications},
  author={Yang, Ling and Zhang, Zhilong and Song, Yang and Hong, Shenda and Xu, Runsheng and Zhao, Yue and Zhang, Wentao and Cui, Bin and Yang, Ming-Hsuan},
  journal={ACM Computing Surveys},
  volume={56},
  number={4},
  pages={1--39},
  year={2023},
  publisher={ACM New York, NY, USA}
}

@inproceedings{kerrigan2024functional_fm,
  title={Functional Flow Matching},
  author={Kerrigan, Gavin and Migliorini, Giosue and Smyth, Padhraic},
  booktitle={Proceedings of the AISTATS'24},
  pages={3934--3942},
  year={2024},
  organization={PMLR}
}

@inproceedings{yimingdefog_discrete_fm,
  title={{DeFoG}: Discrete Flow Matching for Graph Generation},
  author={Qin, Yiming and Madeira, Manuel and Thanou, Dorina and Frossard, Pascal},
  booktitle={Proceedings of the ICML'25},
  year={2025},
  publisher={OpenReview.net}
}

@inproceedings{
gat2024discrete,
title={Discrete Flow Matching},
author={Itai Gat and Tal Remez and Neta Shaul and Felix Kreuk and Ricky T. Q. Chen and Gabriel Synnaeve and Yossi Adi and Yaron Lipman},
booktitle={Advances in Neural Information Processing Systems},
volume={37},
year={2024},
publisher={Curran Associates}
}

@article{ruscelli2026flow,
  title={Flow matching on homogeneous spaces},
  author={Ruscelli, Francesco},
  journal={arXiv preprint arXiv:2603.24829},
  year={2026}
}

@article{beiglbock2025brenier_wow_theory,
  title={A {B}renier Theorem on {$(\mathcal{P}_2(\dots\mathcal{P}_2(H)\dots),\mathcal{W}_2)$} and Applications to Adapted Transport},
  author={Beiglb{\"o}ck, Mathias and Pammer, Gudmund and Schrott, Stefan},
  journal={arXiv preprint arXiv:2509.03506},
  year={2025}
}

@article{emami2025optimal_wow_theory,
  title={Optimal transport with optimal transport cost: the {Monge--Kantorovich problem on W}asserstein spaces},
  author={Emami, Pedram and Pass, Brendan},
  journal={Calculus of Variations and Partial Differential Equations},
  volume={64},
  number={2},
  pages={43},
  year={2025},
  publisher={Springer}
}

@article{huesmann2025benamou_wow_theory,
  title={A {Benamou--B}renier formula for transport distances between stationary random measures},
  author={Huesmann, Martin and M{\"u}ller, Bastian},
  journal={Stochastic Processes and their Applications},
  volume={185},
  pages={104633},
  year={2025},
  publisher={Elsevier}
}

@article{li2025generative_set_generation,
  title={Generative Unordered Flow for Set-Structured Data Generation},
  author={Li, Yangming and Liu, Chaoyu and Sch{\"o}nlieb, Carola-Bibiane},
  journal={arXiv preprint arXiv:2501.17770},
  year={2025}
}

@inproceedings{
ludke2025unlocking_set_generation,
title={Unlocking Point Processes through Point Set Diffusion},
author={David L{\"u}dke and Enric Rabasseda Ravent{\'o}s and Marcel Kollovieh and Stephan G{\"u}nnemann},
booktitle={Proceedings of the ICLR'25},
year={2025},
publisher={OpenReview.net}
}

@inproceedings{luo2021diffusion,
  title={Diffusion probabilistic models for {3D} point cloud generation},
  author={Luo, Shitong and Hu, Wei},
  booktitle={Proceedings of the CVPR'21},
  pages={2837--2845},
  year={2021},
  publisher={IEEE}
}

@article{vaswani2017attention,
  title={Attention is all you need},
  author={Vaswani, Ashish and Shazeer, Noam and Parmar, Niki and Uszkoreit, Jakob and Jones, Llion and Gomez, Aidan N and Kaiser, {\L}ukasz and Polosukhin, Illia},
  journal={Advances in Neural Information Processing Systems},
  volume={30},
  year={2017},
  publisher={Curran Associates}
}

@article{uspsdataset_hull1994,
  author={J. J. {Hull}},
  journal={IEEE Transactions on Pattern Analysis and Machine Intelligence}, 
  title={A database for handwritten text recognition research}, 
  year={1994},
  volume={16},
  number={5},
  pages={550-554},
  doi={10.1109/34.291440},
  publisher={IEEE}
}

@inproceedings{holderriethgeneratormatching,
  title={Generator Matching: Generative modeling with arbitrary {M}arkov processes},
  author={Holderrieth, Peter and Havasi, Marton and Yim, Jason and Shaul, Neta and Gat, Itai and Jaakkola, Tommi and Karrer, Brian and Chen, Ricky T. Q. and Lipman, Yaron},
  booktitle={Proceedings of the ICLR'25},
  year={2025},
  publisher={OpenReview.net}
}

@inproceedings{
jiang2025bureswasserstein_flow_matching,
title={Bures-{W}asserstein Flow Matching for Graph Generation},
author={Keyue Jiang and Jiahao Cui and Xiaowen Dong and Laura Toni},
booktitle={Proceedings of the Generative AI and Biology ICML Workshop},
year={2025},
publisher={OpenReview.net}
}

@article{lindheim2023simple_free_support,
  title={Simple approximative algorithms for free-support {W}asserstein barycenters},
  author={Lindheim, Johannes von},
  journal={Computational Optimization and Applications},
  volume={85},
  number={1},
  pages={213--246},
  year={2023},
  publisher={Springer}
}

@article{hagemann2021stabilizing,
  title={Stabilizing invertible neural networks using mixture models},
  author={Hagemann, Paul and Neumayer, Sebastian},
  journal={Inverse Problems},
  volume={37},
  number={8},
  pages={085002},
  year={2021},
  publisher={IOP Publishing}
}

@inproceedings{chemseddine2026adapting_noise_data_quantile,
  title     = {Adapting Noise to Data by Quantile Learning},
  author    = {Chemseddine, Jannis and Kornhardt, Gregor and Duong, Richard and Steidl, Gabriele},
  booktitle = {Proceedings of the ICLR'26 2nd Workshop on Deep Generative Model in Machine Learning: Theory, Principle and Efficacy},
  year      = {2026},
  publisher = {OpenReview.net},
  url       = {https://openreview.net/forum?id=E0pD1gcaX7}
}

@book{Villani2009oldandnew,
  author    = {Villani, C{\'e}dric},
  title     = {Optimal Transport: Old and New},
  series    = {Grundlehren der mathematischen Wissenschaften},
  volume    = {338},
  publisher = {Springer-Verlag},
  address   = {Berlin},
  year      = {2009}
}

@article{parzen1962estimation_kde,
  title={On estimation of a probability density function and mode},
  author={Parzen, Emanuel},
  journal={The Annals of Mathematical Statistics},
  volume={33},
  number={3},
  pages={1065--1076},
  year={1962},
  publisher={JSTOR}
}

@inproceedings{
pandey2025heavytailed_t_distribution_prior,
title={Heavy-Tailed Diffusion Models},
author={Kushagra Pandey and Jaideep Pathak and Yilun Xu and Stephan Mandt and Michael Pritchard and Arash Vahdat and Morteza Mardani},
booktitle={Proceedings of the ICLR'25},
year={2025},
publisher={OpenReview.net}
}

@inproceedings{cuturi2013sinkhorn,
  title={Sinkhorn distances: {L}ightspeed computation of optimal transport},
  author={Cuturi, Marco},
  booktitle={Advances in Neural Information Processing Systems},
  volume={26},
  year={2013},
  publisher={Curran Associates}
}

@book{Scott2015,
  author    = {Scott, David W.},
  title     = {Multivariate Density Estimation: Theory, Practice, and Visualization},
  edition   = {2},
  publisher = {John Wiley \& Sons},
  year      = {2015},
  isbn      = {978-0-471-69755-8}
}

@inproceedings{
tang2026generative_wowtype_humans,
title={Generative Human Geometry Distribution},
author={Xiangjun Tang and Biao Zhang and Peter Wonka},
booktitle={Proceedings of the ICLR'26},
year={2026},
publisher={OpenReview.net}
}
\appendix
\newpage
\section{Background and Proofs}\label{sec:theory_appendix_supp}

In order to describe the continuity equation \eqref{eq:wow_continuity} in the WoW setting, we need to define test functions suitable for testing on the infinite dimensional space $\mathcal P(\R^d)$.
\begin{definition}\label{def1}
Consider the cylindric functions $F: \mathcal P(\R^d) \to \R$ defined by
$$F(\mu) \coloneqq \Psi\Big(\int_{\R^d} \phi_1 \d \mu, \ldots, \int_{\R^d} \phi_k \d \mu \Big), \quad \Psi \in C_b^1(\R^k), \; 
\phi_i \in C_c^1(\R^d), \; i=1,\ldots,k.
$$ 
For $I=[0,1],$ let \revise{$(M_t)_{t \in I} \in C_I(\mathcal P(\mathcal P(\R^d)))$} and $V: I \times \R^d \times \mP(\R^d) \to \R^d$ be a Borel non-local vector field satisfying the integrability condition 
\begin{equation}\label{velo-int}
    \int_I \int_{\mP(\R^d)} \int_{\R^d} \|V(t,x,\mu)\|  \d \mu(x) \d M_t(\mu) \d t < \infty.
\end{equation}
Then, it is said that $(M_t, V_t)$ satisfies 
the continuity equation \eqref{eq:wow_continuity} if for all cylindric functions $F$ it holds
\begin{equation}
    \frac{d}{dt} \int_{\mathcal P} F(\mu) \d M_t(\mu) = \int_{\mathcal P} \int_{\R^d} (\nabla_W F)(x,\mu) \cdot V_t(x,\mu) \d \mu(x) \d M_t(\mu) 
\end{equation}
in the sense of distributions on $I$, where
$(\nabla_W F) (x,\mu) 
\coloneqq
\revise{\sum_{i=1}^k} \partial_i \Psi (\int \phi_1 \d \mu, \ldots, \int \phi_k \d \mu) \nabla \phi_i (x).
$
\end{definition}


Here, we prove Lemma~\ref{prop:wow_distance_equiv_main} on how to construct random couplings from inner and outer plans.
\IPOPlem*

\begin{proof}
 1.   Concerning the first claim, we get for any bounded measurable $f: \ProbTwo(\R^d) \to \R$ that
    \begin{align}
        \int f(\mu) \, \d(\bproj^0_{\sharp}\bPi)(\mu)
        &= \int f(\bproj^0(\pi)) \, \d\bPi(\pi) \\
        &= \int f(\bproj^0(\operatorname{IP}_{\mu,\nu})) \, \d\operatorname{OP}_{\bmu,\bnu}(\mu,\nu) \\
        &= \int f(\pr^0_{\sharp}\operatorname{IP}_{\mu,\nu}) \, \d\operatorname{OP}_{\bmu,\bnu}(\mu,\nu) \\
        &= \int f(\mu) \, \d\operatorname{OP}_{\bmu,\bnu}(\mu,\nu)
        = \int f(\mu) \, \d \bmu(\mu),
    \end{align}
    where the second equality uses $\bPi = \operatorname{IP}_\sharp \operatorname{OP}_{\bmu,\bnu}$,
    the third one relies on $\bproj^0(\pi) = \pr^0_{\sharp}\pi$,
    and the fourth one on $\pr^0_{\sharp}\operatorname{IP}_{\mu,\nu} = \mu$.
    The last equality holds since $\operatorname{OP}_{\bmu,\bnu}$ has first marginal $\bmu$.
    An analogous computation shows $\bproj^1_{\sharp}\bPi = \bnu$, hence $\bPi \in {\rm RC}(\bmu,\bnu)$.
\\[1ex]
2. We prove an upper and lower bound estimate, resulting in the claimed identity. Note that the upper bound estimate is due to \cite[Proposition 3.1]{pinzi2025nested_wow}. For completeness, we include it here.

\textbf{$\quad$ Upper bound `$\leq$'\revise{.}}
Let $\operatorname{OP}^*\in \text{C}(\bmu,\bnu)$ realize the WoW distance via
\begin{equation}
\bW_2^2(\bmu,\bnu) = \int_{\ProbTwo(\R^d) \times \ProbTwo(\R^d)} \W_2^2(\mu,\nu) \, \d \operatorname{OP}^*(\mu,\nu).
\end{equation}
By a measurable selection theorem \cite[Cor.\ 5.22]{Villani2009oldandnew}, there exists a $\operatorname{OP}^*$-a.e.\ measurable map
$\operatorname{IP}^* \colon (\mu,\nu) \mapsto \operatorname{IP}^*_{\mu,\nu} \in c(\mu,\nu)$ such that
\[
\int_{\R^d \times \R^d} \|x-x'\|^2 \, \d\operatorname{IP}^*_{\mu,\nu}(x,x')
= \W_2^2(\mu,\nu).
\]
Define
\[
\bPi^* \coloneqq \operatorname{IP}^*_\sharp \operatorname{OP}^*.
\]
Then it holds $\bPi^* \in {\rm RC}(\bmu,\bnu)$ by the first part.
Taking the infimum over all admissible $\bPi$ gives the bound
\begin{align}
\inf_{\bPi \in {\rm RC}(\bmu,\bnu)} \int \int \|x-x'\|^2 \, \d\pi \, \d\bPi
\leq
 \int \int \|x-x'\|^2 \, \d\pi \, \d\bPi^*
=
\bW_2^2(\bmu,\bnu).
\end{align}

\textbf{$\quad$ Lower bound `$\geq$'.}
Let $\bPi \in {\rm RC}(\bmu,\bnu)$ and define
\begin{equation}
\label{eq:op_construction_supp_proof}
\widetilde{\operatorname{OP}} \coloneqq (\bproj^0,\bproj^1)_\sharp \bPi \in {\rm C}(\bmu,\bnu).
\end{equation}
Then for $\bPi$-a.e. $\pi \in \ProbTwo(\R^d \times \R^d)$ we obtain by definition of $\W_2$ that
\[
\int_{\R^d \times \R^d} \|x-x'\|^2 \, \d\pi(x,x')
\;\geq\;
\W_2^2(\pr^0_{\sharp}\pi,\pr^1_{\sharp}\pi).
\]
By definition of the push-forward and $\bW_2$, we get 
\begin{align}
\int\limits_{\ProbTwo(\R^d \times \R^d)} \int\limits_{\R^d \times \R^d} \|x-x'\|^2 \, \d\pi \, \d\bPi (\pi)
&\geq 
\int\limits_{\ProbTwo(\R^d \times \R^d)}  \W_2^2(\pr^0_{\sharp}\pi,\pr^1_{\sharp}\pi) \, \d\bPi (\pi)
\\
&=
\int\limits_{\ProbTwo(\R^d) \times \ProbTwo(\R^d)}  \W_2^2(\mu,\nu) \, \d((\bproj^0,\bproj^1)_\sharp \bPi)(\mu, \nu)
\\
&=
\int\limits_{\ProbTwo(\R^d) \times \ProbTwo(\R^d)}  \W_2^2(\mu,\nu) \, \d \widetilde{\operatorname{OP}}(\mu, \nu)
\\
&\geq
\inf_{\operatorname{OP} \in {\rm C}(\bmu, \bnu)}\int\limits_{\ProbTwo(\R^d) \times \ProbTwo(\R^d)}  \W_2^2(\mu,\nu) \, \d \operatorname{OP}(\mu, \nu)\\
& = 
\bW_2^2(\bmu,\bnu).
\end{align}
Taking the infimum over $\bPi$ gives the second inequality, and the proof is finished.
\end{proof}

We now come to the proof of Theorem \ref{prop:wow_flow_matching} which characterizes the non-local velocity fields of certain metameasure curves as minimizers of the loss function \eqref{minWoW}.
\WOWthm*

\begin{proof}
First, we prove essential properties of the curve $(M_t)_{t\in[0,1]}$.

\textbf{$\quad$Correct endpoints of the measure curve:}
By the assumption $\bPi \in {\rm RC}(\bmu,\bnu)$, the curve $(M_t)_{t\in[0,1]}$ indeed has the given endpoints $M_0 = \bmu$ and $M_1 = \bnu$.

\textbf{$\quad$Integrability property \eqref{IP-OP-int}:}
First, by the assumptions $\bPi \in {\rm RC}(\bmu,\bnu)$ and $\bmu, \bnu \allowbreak \in \allowbreak \ProbTwo(\ProbTwo(\R^d))$, it holds for $\bPi$-a.e. $\pi \in \Prob(\R^d \times \R^d)$ that 
$\pr^0_\sharp \pi, \pr^1_\sharp \pi \in \ProbTwo(\R^d)$. Thus, 
\begin{equation}
    \int_{\R^d \times \R^d} \|x\|^2 + \|x'\|^2 \d \pi(x,x')
    = \int_{\R^d} \|z\|^2 \d \pr^0_\sharp\pi(z)
    + \int_{\R^d} \|z\|^2 \d \pr^1_\sharp\pi(z) < \infty,
\end{equation}
which means that $\bPi$ is supported on $\ProbTwo(\R^d \times \R^d)$.

Second, note that for all $\mu, \nu \in \ProbTwo(\R^d)$ it holds the estimates
\begin{equation}\label{wasserstein-estimate1}
    W_2^2\bigl(\mu,\nu \bigr) \le 2 \int_{\R^d} \|x\|^2 \d \mu(x) + 2 \int_{\R^d} \|y\|^2 \d \nu(y)
\end{equation}
and
\begin{equation}\label{wasserstein-estimate2}
    \int_{\R^d} \|x\|^2 \d \mu(x) \le 2 W_2^2\bigl(\mu,\nu \bigr) + 2 \int_{\R^d} \|y\|^2 \d \nu(y)
\end{equation}
by the triangle and Young's inequality, and using any coupling $\pi \in \text{c}(\mu, \nu)$ in \eqref{wasserstein-estimate1}, and an optimal coupling $\pi^* \in \text{c}^{\rm opt}(\mu, \nu)$ in \eqref{wasserstein-estimate2}. In particular, it holds for all $M \in \Prob(\ProbTwo(\R^d))$ the equivalence 
\begin{equation}\label{PropTwo-equiv}
   M \in \ProbTwo(\ProbTwo(\R^d)) \iff \int_{\ProbTwo(\R^d)} \int_{\R^d} \|x\|^2 \d \mu(x) \d M (\mu) < \infty.
\end{equation}
Therefore, by the assumption $\bmu, \bnu \in \ProbTwo(\ProbTwo(\R^d))$, we have
\begin{equation}
    \int_{\ProbTwo(\R^d)} \int_{\R^d} \|x\|^2 \d \mu(x) \d \bmu (\mu) +
    \int_{\ProbTwo(\R^d)} \int_{\R^d} \|x\|^2 \d \mu(x) \d \bnu (\mu) < \infty.
\end{equation}
Now, $\bPi \in {\rm RC}(\bmu,\bnu)$ immediately implies that the above sum is equal to
\begin{align}
\int_{\ProbTwo(\R^d\times\R^d)}
\int_{\R^d\times\R^d} \|x\|^2 + \|x'\|^2\,\d\pi(x,x')\,\d\bPi(\pi), 
\end{align}
which proves the integrability bound \eqref{IP-OP-int}, i.e. $\bPi \in \ProbTwo(\ProbTwo(\R^d \times \R^d))$. 

\textbf{$\quad$Absolute continuity of the curve:}
We first show that $M_t \in \ProbTwo(\ProbTwo(\R^d))$ for all $t \in [0,1]$. 
Indeed, since $\bPi$-a.e.\ $\pi$ is in $\ProbTwo(\R^d \times \R^d)$, we have 
$\pr^t_\sharp \pi \in \ProbTwo(\R^d)$, meaning that $M_t$ is supported on $\ProbTwo(\R^d)$.
Then, it holds
\begin{align}
&\int_{\ProbTwo(\R^d)} \int_{\R^d} \|x\|^2 \d \mu(x) \d M_t (\mu)
= \int_{\ProbTwo(\R^d \times \R^d)} \int_{\R^d} \|x\|^2 \d (\pr^t_\sharp \pi) (x) \d\bPi(\pi)\\
= &\int_{\ProbTwo(\R^d \times \R^d)} \int_{\R^d \times \R^d} \|(1-t)x + tx'\|^2 \d \pi(x,x') \d\bPi(\pi)\\
\le &\int_{\ProbTwo(\R^d \times \R^d)} \int_{\R^d \times \R^d} 2 (1-t)^2 \|x\|^2 + 2t^2 \|x'\|^2 \d \pi(x,x') \d\bPi(\pi),
\end{align}
which is finite by the already shown estimate \eqref{IP-OP-int} and $t \in [0,1]$. Hence, $M_t \in \ProbTwo(\ProbTwo(\R^d))$ by \eqref{PropTwo-equiv}.

Next, for all $s,t\in[0,1]$ and every $\pi\in\ProbTwo(\R^d\times\R^d)$, the coupling induced by the map $(x,x')\mapsto (\pr^s(x,x'),\pr^t(x,x'))$ gives the bound
\begin{equation}\label{helper1}
W_2^2\bigl(\pr^s_\#\pi,\pr^t_\#\pi\bigr)
\le
|t-s|^2\int_{\R^d\times\R^d}|x'-x|^2\,\d\pi(x,x').
\end{equation}

With the mapping $T: \ProbTwo(\R^d \times \R^d) \to \ProbTwo(\R^d) \times \ProbTwo(\R^d), \pi \mapsto (\proj^s_{\sharp} \pi, \proj^t_{\sharp} \pi),$ we obtain a coupling $T_\sharp \bPi \in \text{C}(M_s, M_t)$: the endpoint relations are clear by definition, and we have 
\begin{align}
&\int_{\ProbTwo(\R^d) \times \ProbTwo(\R^d)}  W_2^2(\mu, \tilde \mu) + W_2^2(\nu, \tilde \nu)\d T_\sharp \bPi (\mu, \nu)\\
=&\int_{\ProbTwo(\R^d \times \R^d)}  W_2^2(\proj^s_{\sharp} \pi, \tilde \mu) + W_2^2(\proj^t_{\sharp} \pi, \tilde \nu)\d \bPi (\pi)\\
=&\int_{\ProbTwo(\R^d)}  W_2^2(\mu, \tilde \mu) \d M_s (\mu)
+ \int_{\ProbTwo(\R^d)}  W_2^2(\mu, \tilde \nu)\d M_t (\mu) < \infty 
\end{align}
for any $\tilde \mu, \tilde \nu \in \ProbTwo(\R^d)$, since $M_s, M_t \in \ProbTwo(\ProbTwo(\R^d))$. Hence, $T_\sharp \bPi \in \ProbTwo(\ProbTwo(\R^d) \times \ProbTwo(\R^d))$.

Therefore we verified $T_\sharp \bPi \in \text{C}(M_s, M_t)$, and the above estimate \eqref{helper1} yields
\begin{equation}
\bW_2^2(M_s,M_t)
\le
|t-s|^2
\int_{\ProbTwo(\R^d\times\R^d)}
\int_{\R^d\times\R^d} \|x'-x\|^2\,\d\pi(x,x')\,\d\bPi(\pi).
\end{equation}
By the proven estimate \eqref{IP-OP-int}, the last double integral is finite.
Hence $t\mapsto M_t$ is Lipschitz, in particular absolutely continuous with respect to $\bW_2$.

\textbf{$\quad$Square-integrable conditional vector field:}
Define the two-level measure $Q$ on $\mathcal{X} \coloneqq \ProbTwo(\R^d \times \R^d) \times \R^d \times \R^d$ via
\begin{equation}
\d Q(\pi,x,x')
\coloneqq
\d \pi(x,x') \d \bPi(\pi).
\end{equation}

Next, consider the measurable map $P^t: \mathcal{X} \to \ProbTwo(\R^d) \times \R^d$ given by
\begin{equation}
P^t : (\pi,x,x') \mapsto \bigl(\pr^t_\#\pi,\pr^t(x,x')\bigr)
= (\tilde\mu,\tilde x).
\end{equation}
Now we have $\d (P^t_\# Q)(\tilde \mu, \tilde x) = \d \tilde\mu(\tilde x) \revise{\d M_t(\tilde \mu)}$.
Since all concerned spaces are Polish, disintegration of $Q$ with regard to $P_t$ defines a regular conditional probability
\begin{equation}
\Gamma_{t,\tilde\mu,\tilde x}\in \Prob\bigl(\mathcal{X} \bigr),
\end{equation}
such that
\begin{equation}
\d Q(\pi,x,x')
=
\d \Gamma_{t,\tilde\mu,\tilde x}(\pi,x,x')\,\d \tilde\mu(\tilde x)\,\d M_t(\tilde\mu).
\end{equation}
Then we define the \emph{conditional vector field}
\begin{equation}
V(t,\tilde x,\tilde\mu)
\coloneqq
\int_{\mathcal{X}}
(x'-x)\,\d \Gamma_{t,\tilde\mu,\tilde x}(\pi,x,x').
\end{equation}
Let us verify that 
$V$  satisfies a square-integrability condition.
By Jensen's inequality applied to the probability measure
$\Gamma_{t,\tilde\mu,\tilde x}$, we obtain
\begin{align}
|V(t,\tilde x,\tilde\mu)|^2
&=
\left|
\int_{\mathcal{X}}
(x'-x)\,\d\Gamma_{t,\tilde\mu,\tilde x}(\pi,x,x')
\right|^2\\
&\le
\int_{\mathcal{X}}
\|x'-x\|^2\,\d\Gamma_{t,\tilde\mu,\tilde x}(\pi,x,x').
\end{align}
Integrating with respect to $\d\tilde\mu(\tilde x)\,\d M_t(\tilde\mu)\,\d t$ and using the
disintegration identity for $Q$, we find 

\begin{align} \label{eq:kinetic_energy_less_than_coupled_energy}
&\int_0^1\int_{\ProbTwo(\R^d)}\int_{\R^d}
|V(t,\tilde x,\tilde\mu)|^2\,\d\tilde\mu(\tilde x)\,\d\ M_t(\tilde\mu)\,\d t\\
&\le
\int_0^1
\int_{\ProbTwo(\R^d)}
\int_{\R^d}
\int_{\mathcal{X}}
\|x'-x\|^2\,\d\Gamma_{t,\tilde\mu,\tilde x}(\pi,x,x')
\,\d\tilde\mu(\tilde x)\,\d M_t(\tilde\mu)\,\d t\\
&=
\int_0^1
\int_{\ProbTwo(\R^d\times\R^d)}
\int_{\R^d\times\R^d}
\|x'-x\|^2\,\d\pi(x,x')\,\d\bPi(\pi)\,\d t
<\infty,
\end{align}
where the last term is finite by the bound \eqref{IP-OP-int}. In particular, $V$ satisfies the integrability condition \eqref{velo-int} of Definition \ref{def1}.


\textbf{$\quad$Metameasure flow is generated via conditional vector field:}
We claim that $\partial_t M_t+\operatorname{div}_{\Prob}(V_t M_t)=0$ in the sense of
Definition~\ref{def1}. To prove this, let
\begin{equation}
F(\mu)=\Psi\Bigl(\int_{\R^d}\phi_1\,\d\mu,\dots,\int_{\R^d}\phi_k\,\d\mu\Bigr)
\end{equation}
with $\phi_i\in C_c^1(\R^d), i=1,\dots,k,$ and $\Psi\in C_b^1(\R^k)$. Since $M_t=\bproj^t_\#\bPi$, we have
\begin{equation}
\int_{\ProbTwo(\R^d)}F(\mu)\,\d M_t(\mu)
=
\int_{\ProbTwo(\R^d\times\R^d)}F\bigl(\pr^t_\#\pi\bigr)\,\d\bPi(\pi).
\end{equation}
For fixed $\pi \in \ProbTwo(\R^d \times \R^d)$, the map $t\mapsto F(\pr^t_\#\pi)$ is differentiable. Thus, we may apply the chain rule to reformulate
\begin{equation}
\begin{aligned}
&\frac{\d}{\d t} ~ F\bigl(\pr^t_\#\pi\bigr)\\
&=
\frac{\d}{\d t} ~
\Psi\Bigl(
\int_{\R^d\times\R^d}\phi_1(\pr^t(x,x'))\,\d\pi(x,x'),
\dots,
\int_{\R^d\times\R^d}\phi_k(\pr^t(x,x'))\,\d\pi(x,x')
\Bigr)
\\
&=
\sum_{i=1}^k
\partial_i\Psi\Bigl(
\int_{\R^d\times\R^d}\phi_1(\pr^t(x,x'))\,\d\pi(x,x'),
\dots,
\int_{\R^d\times\R^d}\phi_k(\pr^t(x,x'))\,\d\pi(x,x')
\Bigr)
\\
&\qquad\qquad \cdot
\frac{\d}{\d t}
\int_{\R^d\times\R^d}\phi_i(\pr^t(x,x'))\,\d\pi(x,x')
\\
&=
\sum_{i=1}^k
\partial_i\Psi\Bigl(
\int_{\R^d\times\R^d}\phi_1(\pr^t(x,x'))\,\d\pi(x,x'),
\dots,
\int_{\R^d\times\R^d}\phi_k(\pr^t(x,x'))\,\d\pi(x,x')
\Bigr)
\\
&\qquad\qquad \cdot
\int_{\R^d\times\R^d}
\nabla \phi_i(\pr^t(x,x'))\cdot(x'-x)\,\d\pi(x,x')
\\
&=
\int_{\R^d\times\R^d}
\sum_{i=1}^k
\partial_i\Psi\Bigl(
\int_{\R^d\times\R^d}\phi_1(\pr^t(x,x'))\,\d\pi(x,x'),
\dots,
\int_{\R^d\times\R^d}\phi_k(\pr^t(x,x'))\,\d\pi(x,x')
\Bigr)
\\
&\qquad\qquad \cdot
\nabla \phi_i(\pr^t(x,x'))\cdot(x'-x)\,\d\pi(x,x')
\\
&=
\int_{\R^d\times\R^d}
\nabla_WF\bigl(\pr^t(x,x'),\pr^t_\#\pi\bigr)\cdot(x'-x)\,\d\pi(x,x').
\end{aligned}
\end{equation}
Here, we applied the definition for the first equation and the classical chain rule for the second one. The third equality is justified by differentiation under the integral sign, since $\phi_i\in C_c^1(\R^d)$, the chain rule and Cauchy-Schwarz inequality imply
\begin{equation}
\left|\partial_t \phi_i(\pr_t(x,x'))\right|
\le
\|\nabla\phi_i\|_\infty\, \|x'-x\|,
\end{equation}
and the right-hand side is $\pi$-integrable by assumption. In the end, we employed the definition of $\nabla_W F$ after extracting the integral.

The above calculation gives the estimate
\begin{equation}
\left|
\frac{\d}{\d t}F\bigl(\pr^t_\#\pi\bigr)
\right|
\le
C_F \int_{\R^d\times\R^d} \|x'-x\|\,\d\pi(x,x'),
\end{equation}
for some constant $0 < C_F < \infty$ which is finite by the boundedness of the derivatives of $\Psi$ and $\phi_i$. In particular, by the estimate \eqref{IP-OP-int} the mapping
$
\pi\mapsto \int_{\R^d\times\R^d} \|x'-x\| \,\d\pi(x,x')
$
is in $L^1(\bPi)$.
Thus, we may again differentiate under the integral sign with respect to $\bPi$. Finally, 
by the construction of $V$ and $M_t$ and using that $(\pr^t_\#\pi, \pr^t(x,x')) = P^t(\pi, x, x') = (\tilde\mu,\tilde x)$ for $\Gamma_{t,\tilde\mu,\tilde x}$-a.e.\ $(\pi, x, x')$,
we obtain
\begin{align*}
&\frac{\d}{\d t}\int_{\ProbTwo(\R^d)}F(\mu)\,\d M_t(\mu)\\
&= \int_{\ProbTwo(\R^d)} \int_{\R^d} \int_{\mathcal{X}}
\hspace{-0mm} \nabla_WF\bigl(\pr^t(x,x'),\pr^t_\#\pi\bigr)\cdot(x'-x)\, \d \Gamma_{t,\tilde\mu,\tilde x}(\pi,x,x')\,\d \tilde\mu(\tilde x)\,\d M_t(\tilde\mu)\\
&= \int_{\ProbTwo(\R^d)} \int_{\R^d} \nabla_WF\bigl(\tilde x, \tilde \mu\bigr) \cdot\int_{\mathcal{X}}
\hspace{0mm} (x'-x)\, \d \Gamma_{t,\tilde\mu,\tilde x}(\pi,x,x')\,\d \tilde\mu(\tilde x)\,\d M_t(\tilde\mu)\\
&= \int_{\ProbTwo(\R^d)}\int_{\R^d}
\nabla_WF(\tilde x,\tilde \mu)\cdot V(t,\tilde x,\tilde \mu)\,\d\tilde \mu(\tilde x)\,\d M_t(\tilde \mu).
\end{align*}
Hence, $(M_t, V_t)$ solves the WoW continuity equation
$
\partial_t M_t+\operatorname{div}_{\Prob}(V_t M_t)=0
$
in the sense of Definition~\ref{def1}.

\textbf{$\quad$Conditional vector field minimizes $\mathcal{J}$:}
In line with the derivation of the classical flow matching loss, we now employ the decomposition
\begin{align}
&\bigl\|\tilde V\bigl(t,\pr^t(x,x'),\pr^t_\#\pi\bigr)-(x'-x)\bigr\|^2\\
&=
\bigl\|\tilde V\bigl(t,\pr^t(x,x'),\pr^t_\#\pi\bigr)\bigr\|^2
-2\Bigl\langle \tilde V\bigl(t,\pr^t(x,x'),\pr^t_\#\pi\bigr),x'-x\Bigr\rangle\\
&\qquad+\|x'-x\|^2.
\end{align}
For the first term, by disintegration with respect to
\begin{equation}
P^t:(\pi,x,x')\mapsto \bigl(\pr^t_\#\pi,\pr^t(x,x')\bigr)=(\tilde\mu,\tilde x),
\end{equation}
we obtain
\begin{align*}
&\int_0^1
\int_{\ProbTwo(\R^d\times\R^d)}
\int_{\R^d\times\R^d}
\bigl\|\tilde V\bigl(t,\pr^t(x,x'),\pr^t_\#\pi\bigr)\bigr\|^2
\,\d\pi(x,x')\,\d\bPi(\pi)\,\d t\\
&\qquad =
\int_0^1
\int_{\ProbTwo(\R^d)}
\int_{\R^d}
\|\tilde V(t,\tilde x,\tilde\mu)\|^2
\,\d\tilde\mu(\tilde x)\,\d M_t(\tilde\mu)\,\d t.
\end{align*}
For the second mixed term, using again the disintegration kernel $\Gamma_{t,\tilde\mu,\tilde x}$ and the definition of $V$, we conclude
\begin{align*}
&\int_0^1
\int_{\ProbTwo(\R^d\times\R^d)}
\int_{\R^d\times\R^d}
\Bigl\langle \tilde V\bigl(t,\pr^t(x,x'),\pr^t_\#\pi\bigr),x'-x\Bigr\rangle
\,\d\pi(x,x')\,\d\bPi(\pi)\,\d t\\
& =
\int_0^1
\int_{\ProbTwo(\R^d)}
\int_{\R^d}
\left\langle
\tilde V(t,\tilde x,\tilde\mu),
\int_{\mathcal{X}}
(x'-x)\,\d\Gamma_{t,\tilde\mu,\tilde x}(\pi,x,x')
\right\rangle \d\tilde\mu(\tilde x)\,\d M_t(\tilde\mu)\,\d t\\
& =
\int_0^1
\int_{\ProbTwo(\R^d)}
\int_{\R^d}
\langle \tilde V(t,\tilde x,\tilde\mu),V(t,\tilde x,\tilde\mu)\rangle
\,\d\tilde\mu(\tilde x)\,\d M_t(\tilde\mu)\,\d t.
\end{align*}
Therefore, setting
\begin{equation}
C:=
\int_0^1
\int_{\ProbTwo(\R^d\times\R^d)}
\int_{\R^d\times\R^d}
\|x'-x\|^2\,\d\pi(x,x')\,\d\bPi(\pi)\,\d t,
\end{equation}
we arrive at
\begin{align*}
\mathcal J(\tilde V)
&=
\int_0^1
\int_{\ProbTwo(\R^d)}
\int_{\R^d}
\|\tilde V(t,x,\mu)\|^2
\,\d\mu(x)\,\d M_t(\mu)\,\d t\\
&\qquad
-2\int_0^1
\int_{\ProbTwo(\R^d)}
\int_{\R^d}
\langle \tilde V(t,x,\mu),V(t,x,\mu)\rangle
\,\d\mu(x)\,\d M_t(\mu)\,\d t
+ C\\
&=
\int_0^1
\int_{\ProbTwo(\R^d)}
\int_{\R^d}
\|\tilde V(t,x,\mu)-V(t,x,\mu)\|^2
\,\d\mu(x)\,\d M_t(\mu)\,\d t
\\
&\qquad+
\Bigl(
C-
\int_0^1
\int_{\ProbTwo(\R^d)}
\int_{\R^d}
\|V(t,x,\mu)\|^2
\,\d\mu(x)\,\d M_t(\mu)\,\d t
\Bigr).
\end{align*}
Hence it follows
\begin{equation}
\mathcal J(\tilde V)
=
\int_0^1
\int_{\ProbTwo(\R^d)}
\int_{\R^d}
\|\tilde V(t,x,\mu)-V(t,x,\mu)\|^2
\,\d\mu(x)\,\d M_t(\mu)\,\d t
+ C_V,
\end{equation}
where $C_V$ is a constant independent of $\tilde V$. Since the first term is nonnegative, we obtain
\[
\mathcal J(\tilde V)\ge \mathcal J(V),
\]
with equality if and only if
\[
\tilde V=V
\qquad\text{in }L^2([0,1]\times\ProbTwo(\R^d)\times\R^d;\,\d \mu \d M_t(\mu)\d t;\, \R^d).
\]
\revise{Consequently}, $V$ is the unique minimizer of $\mathcal J$ up to $L^2$-equivalence.

\textbf{~ $\quad$ Optimality.} Now let $\bPi = \operatorname{IP}^*_\sharp \operatorname{OP}^*$.
As established in \eqref{eq:kinetic_energy_less_than_coupled_energy} and in the proof of Lemma \ref{prop:wow_distance_equiv_main}, it holds that
\begin{align}
\label{eq:variance_decomp}
&\int_0^1\int_{\ProbTwo(\R^d)}\int_{\R^d}
\|V(t,x,\mu)\|^2\,\d\mu(x)\,\d M_t(\mu)\,\d t\\
&\leq
\int_0^1\int_{\ProbTwo(\R^d\times\R^d)}\int_{\R^d\times\R^d}
\|x'-x\|^2
\,\d\pi(x,x')\,\d\bPi^*(\pi)\,\d t
= \bW_2^2(\bmu,\bnu).
\end{align}
On the other hand, the velocity field $V$ satisfies the WoW continuity equation
$\partial_t M_t + \operatorname{div}_{\Prob}(V_t M_t) = 0$
\revise{as shown above}.
Hence it is an admissible velocity field in the dynamic WoW formulation, see \cite[Thm. 1.3]{pinzi2025nested_wow}, giving
\begin{equation}
\int_0^1\int_{\ProbTwo(\R^d)}\int_{\R^d}
\|V(t,x,\mu)\|^2\,\d\mu(x)\,\d M_t(\mu)\,\d t
\geq \bW_2^2(M_0, M_1).
\end{equation}
Combining both bounds results in \revise{equality in} \eqref{eq:wow_flow_matching_formulation_main}\revise{. For general $\bPi \in {\rm RC}(\bmu,\bnu)$, the upper bound in \eqref{eq:wow_flow_matching_formulation_main} is \eqref{eq:kinetic_energy_less_than_coupled_energy}, and the lower bound follows by the same dynamic argument}, \revise{which} concludes the proof.
\end{proof}

In the following we prove Proposition \ref{prop:lazy_linear_justification} justifying the lazy linearization strategy proposed in Section \ref{sec:couplings}.
\LAZYprop*
\begin{proof}
\textbf{~ Almost every permutation minimizer is unique:} For $\mathbf{y},\mathbf{z} \in \R^{d \times N}$, the measures
$\operatorname{Vec2M}(\mathbf{y})$ and $\operatorname{Vec2M}(\mathbf{z})$
are empirical measures with $N$ atoms of equal mass. Hence,
\begin{equation}
\label{eq:discrete_wasserstein_as_permutation_problem}
\W_2^2(\operatorname{Vec2M}(\mathbf{y}),\operatorname{Vec2M}(\mathbf{z}))
=
\min_{P \in \Perm(N)}
\|\mathbf{y} - P\mathbf{z}\|^2,
\end{equation}
where $\Perm(N) \subset \R^{N \times N}$ denotes the set of permutation matrices \cite[Prop. 2.1]{PeyreCuturi2019}.

We first show that for almost every $\mathbf{x}' \in \R^{d \times N}$, the minimizer of
\[
P \mapsto \|\mathbf{x}_{\hat{\rho}} - P\mathbf{x}'\|^2
\]
over $P \in \Perm(N)$ is unique.

To this aim, let $P,Q \in \Perm(N)$ with $P \neq Q$. Using that permutations are orthogonal and hence preserve the norm, i.e., $\|P\mathbf{x}'\|^2 = \|Q\mathbf{x}'\|^2 = \|\mathbf{x}'\|^2$, we obtain
\begin{align}
\|\mathbf{x}_{\hat{\rho}} - P\mathbf{x}'\|^2
-
\|\mathbf{x}_{\hat{\rho}} - Q\mathbf{x}'\|^2
&=
-2\langle \mathbf{x}_{\hat{\rho}}, P\mathbf{x}' - Q\mathbf{x}' \rangle \\
&=
-2\langle (P-Q)^T \mathbf{x}_{\hat{\rho}}, \mathbf{x}' \rangle.
\end{align}
Thus, equality of the two costs is equivalent to
\[
\langle (P-Q)^T \mathbf{x}_{\hat{\rho}}, \mathbf{x}' \rangle = 0.
\]
Since $\mathbf{x}_{\hat{\rho},i} \neq \mathbf{x}_{\hat{\rho},i'}$ for all $i \neq i'$, distinct permutations $P \neq Q$ satisfy $P\mathbf{x}_{\hat{\rho}} \neq Q\mathbf{x}_{\hat{\rho}}$, and thus
\[
(P-Q)^T \mathbf{x}_{\hat{\rho}} \neq 0.
\]
Therefore, for every pair $P \neq Q$, the set
\[
H_{P,Q}
:=
\left\{
\mathbf{x}' \in \R^{d \times N}
\;: \;
\|\mathbf{x}_{\hat{\rho}} - P\mathbf{x}'\|^2
=
\|\mathbf{x}_{\hat{\rho}} - Q\mathbf{x}'\|^2
\right\}
\]
is a hyperplane in $\R^{d \times N}$ and hence has Lebesgue measure zero.
As $\Perm(N)$ is finite, the union of all such tie sets also has Lebesgue measure zero. Consequently, for almost every $\mathbf{x}' \in \R^{d \times N}$, there exists a unique permutation matrix
\[
P_* = P_{\operatorname{Vec2M}(\mathbf{x}'),\mathbf{x}_{\hat{\rho}}}
\]
such that
\[
\|\mathbf{x}_{\hat{\rho}} - P_*\mathbf{x}'\|^2
<
\|\mathbf{x}_{\hat{\rho}} - P\mathbf{x}'\|^2
\qquad
\text{for all } P \in \Perm(N)\setminus\{P_*\}.
\]

\textbf{$\quad$ The permutation minimizer is locally stable:} Now fix such an $\mathbf{x}'$.
Since $\Perm(N)$ is finite, the quantity
\[
\gamma
:=
\min_{P \in \Perm(N)\setminus\{P_*\}}
\left(
\|\mathbf{x}_{\hat{\rho}} - P\mathbf{x}'\|^2
-
\|\mathbf{x}_{\hat{\rho}} - P_*\mathbf{x}'\|^2
\right)
\]
is strictly positive.

For $P \in \Perm(N)$, define
\[
f_P(\mathbf{x}) := \|\mathbf{x} - P\mathbf{x}'\|^2.
\]
Then, for any $\mathbf{x} \in \R^{d \times N}$ and $P \in \Perm(N)$,
\begin{align}
\bigl(f_P(\mathbf{x}) - f_{P_*}(\mathbf{x})\bigr)
-
\bigl(f_P(\mathbf{x}_{\hat{\rho}}) - f_{P_*}(\mathbf{x}_{\hat{\rho}})\bigr)
=
2\langle \mathbf{x}-\mathbf{x}_{\hat{\rho}}, (P_* - P)\mathbf{x}' \rangle.
\end{align}
Hence, by Cauchy--Schwarz,
\[
\left|
\bigl(f_P(\mathbf{x}) - f_{P_*}(\mathbf{x})\bigr)
-
\bigl(f_P(\mathbf{x}_{\hat{\rho}}) - f_{P_*}(\mathbf{x}_{\hat{\rho}})\bigr)
\right|
\leq
2\|\mathbf{x}-\mathbf{x}_{\hat{\rho}}\|\,\|(P_* - P)\mathbf{x}'\|.
\]
Again using finiteness of $\Perm(N)$, we may set
\[
M := \max_{P \in \Perm(N)\setminus\{P_*\}} \|(P_* - P)\mathbf{x}'\| < \infty.
\]
Choose $R_{\mathbf{x}'} > 0$ such that $2 R_{\mathbf{x}'} M < \gamma$. 
Then for every $\mathbf{x}$ with $\|\mathbf{x}-\mathbf{x}_{\hat{\rho}}\| < R_{\mathbf{x}'}$, we have for all $P \neq P_*$ that
\begin{align}
f_P(\mathbf{x}) - f_{P_*}(\mathbf{x})
&\geq
f_P(\mathbf{x}_{\hat{\rho}}) - f_{P_*}(\mathbf{x}_{\hat{\rho}}) - 2 R_{\mathbf{x}'} M \\
&\geq
\gamma - 2 R_{\mathbf{x}'} M \\
&> 0.
\end{align}
Thus, for all $\mathbf{x}$ with $\|\mathbf{x}-\mathbf{x}_{\hat{\rho}}\| < R_{\mathbf{x}'}$,
\[
f_{P_*}(\mathbf{x}) < f_P(\mathbf{x})
\qquad
\text{for all } P \in \Perm(N)\setminus\{P_*\},
\]
so $P_*$ remains the unique minimizer of
\[
P \mapsto \|\mathbf{x} - P\mathbf{x}'\|^2.
\]

Applying \eqref{eq:discrete_wasserstein_as_permutation_problem}, we conclude that
\[
\W_2^2(\operatorname{Vec2M}(\mathbf{x}), \operatorname{Vec2M}(\mathbf{x}'))
=
\|\mathbf{x} - P_*\mathbf{x}'\|^2
=
\|\mathbf{x} - P_{\operatorname{Vec2M}(\mathbf{x}'),\mathbf{x}_{\hat{\rho}}} \mathbf{x}'\|^2,
\]
which proves the claim.    
\end{proof}

\newpage

\section{Further \revise{Experiments}}
\label{sec:further_exp}
\subsection{\texorpdfstring{MNIST$\to$\revise{USPS} Transport (2D, Transformer, Multi-Resolution: [64, 128, 256])}{MNIST to USPS Transport (2D, Transformer, Multi-Resolution: [64, 128, 256])}}
\label{subsec:mnist2usps}

\begin{table}[h]
\centering
    \setlength{\tabcolsep}{2.6pt}
    \renewcommand{\arraystretch}{0.9}
    \begin{tabular}{llccc|ccc|c}
    \toprule
     &  & \multicolumn{3}{c}{Chamfer-NNA $\downarrow$} & \multicolumn{3}{c}{OT-NNA $\downarrow$} & \textbf{Train} \\
    \textbf{OP} & \textbf{IP} & Euler-5 & Euler-25 & Euler-125 & Euler-5 & Euler-25 & Euler-125 &  \\
    \midrule
    $\widehat{\operatorname{OP}}_{\text{\textbf{ind}}}$ & $\widehat{\operatorname{IP}}_{\mathrm{ind}}$ & $1.00{\scriptstyle\pm0.00}$ & $0.79{\scriptstyle\pm0.01}$ & $0.70{\scriptstyle\pm0.00}$ & $0.99{\scriptstyle\pm0.00}$ & $0.67{\scriptstyle\pm0.01}$ & $\underline{0.61{\scriptstyle\pm0.01}}$ & 50--60h\\
    $\widehat{\operatorname{OP}}_{\bW}$ & $\widehat{\operatorname{IP}}_{\mathrm{W}}$ & $\mathbf{0.64{\scriptstyle\pm0.01}}$ & $\mathbf{0.60{\scriptstyle\pm0.01}}$ & $\mathbf{0.60{\scriptstyle\pm0.01}}$ & $\mathbf{0.62{\scriptstyle\pm0.02}}$ & $\mathbf{0.60{\scriptstyle\pm0.02}}$ & $\mathbf{0.61{\scriptstyle\pm0.02}}$ & \revise{100--110h} \\
    $\widehat{\operatorname{OP}}_{\text{\textbf{SW}}}$ & $\widehat{\operatorname{IP}}_{\mathrm{SW}}$ & $\underline{0.94{\scriptstyle\pm0.01}}$ & $\underline{0.65{\scriptstyle\pm0.01}}$ & $\underline{0.62{\scriptstyle\pm0.01}}$ & $\underline{0.84{\scriptstyle\pm0.01}}$ & $\underline{0.64{\scriptstyle\pm0.01}}$ & $0.63{\scriptstyle\pm0.01}$ & 50--60h \\
    \bottomrule
    \end{tabular}
\caption{Generation quality for MNIST$\to$USPS point clouds.}
\label{tab:usps_pt_mnist_paired_stdnone_N128}
\end{table}

Beyond generation from noise, we evaluate our framework on a cross-domain transport task.
We use point clouds sampled from MNIST as the 
empirical source metameasure, and point clouds sampled from USPS \citep{uspsdataset_hull1994} as the
target metameasure, treating both datasets analogously to the MNIST setting above. In particular, normalized pixel intensities define a 2D histogram on $[0,1]^2$
from which $N=64$,$128$, or $256$ points are sampled. Thus, we randomly obtain $\mathbf{x}\in\R^{2\times 64}$, $\mathbf{x}\in\R^{2\times 128}$
or  $\mathbf{x}\in\R^{2\times 256}$.
We train a standard transformer for $350$ epochs using Adam (step size: $5\cdot 10^{-4}$) using different coupling strategies.
Notably, a domain-adaptive model should learn that the USPS digits are bolder than the MNIST digits.

Evaluating the quality of translated samples at $N=256$, Table~\ref{tab:usps_pt_mnist_paired_stdnone_N128}
 shows the clear advantage of OT-based couplings.
Especially for few-step generation (Euler-5),
the approximation of the WoW plan clearly gives the best results. 
In this few-step case, the sliced coupling drastically improves upon the independent coupling as well, while adding almost no computational training overhead in contrast to the WoW approximation.
For sufficiently discretized flow ODEs,
the NNA metrics stabilize, see the comparison between Euler-25 and Euler-125, and the independent coupling becomes competitive. 

Figure~\ref{fig:mnist_usps_interp} visualizes learned interpolations via kernel density estimation \citep{parzen1962estimation_kde}, see
Appendix~\ref{subsec:kde_viz_supp}. While $(\widehat{\operatorname{OP}}_{\bW}, \widehat{\operatorname{IP}}_{\W})$ and $(\widehat{\operatorname{OP}}_{\text{SW}}, \widehat{\operatorname{IP}}_{\text{SW}})$ preserve digit structure at the midpoint ($t=0.5$), the independent coupling $(\widehat{\operatorname{OP}}_{\text{\textbf{ind}}}, \widehat{\operatorname{IP}}_{\mathrm{ind}})$ yields amorphous intermediate densities.

\begin{figure}[h]
    \centering
    \begin{subfigure}[b]{0.495\textwidth}
        \centering
\includegraphics[width=\textwidth, trim=0 0 77.5cm 0, clip]{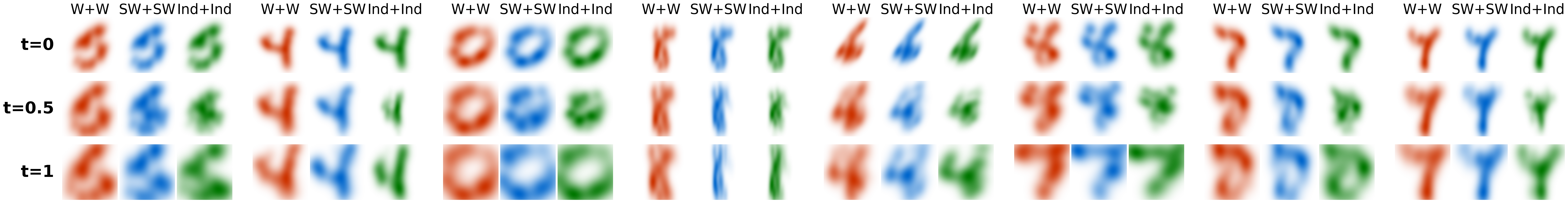}
        \caption{$N = 128$}
        \label{subfig:mnist_usps_lr}
    \end{subfigure}
    \hfill
    \begin{subfigure}[b]{0.495\textwidth}
        \centering
\includegraphics[width=\textwidth, trim=0 0 77.5cm 0, clip]{mnist2usps/traj_kde_grid_n4096.png}
        \caption{$N=4096$}
        \label{subfig:mnist_usps_hr}
    \end{subfigure}
    \caption{Learned interpolation for MNIST$\to$USPS with $N=128$ and out-of-training $N=4096$ for (i) $(\widehat{\operatorname{OP}}_{\bW}, \widehat{\operatorname{IP}}_{\W})$, (ii) $(\widehat{\operatorname{OP}}_{\text{SW}}, \widehat{\operatorname{IP}}_{\text{SW}})$ and (iii) $(\widehat{\operatorname{OP}}_{\text{\textbf{ind}}}, \widehat{\operatorname{IP}}_{\mathrm{ind}})$ visualized via estimated densities. MNIST digits ($t=0$) are mapped to emboldened USPS-like digits ($t=1$) for all couplings\revise{.}
    However,  
    $(\widehat{\operatorname{OP}}_{\bW}, \widehat{\operatorname{IP}}_{\W})$ and $(\widehat{\operatorname{OP}}_{\text{SW}}, \widehat{\operatorname{IP}}_{\text{SW}})$ midpoints ($t=0.5$) maintain digit structure for $N=4096$, whereas $(\widehat{\operatorname{OP}}_{\text{\textbf{ind}}}, \widehat{\operatorname{IP}}_{\mathrm{ind}})$ midpoints appear amorphous. Similar trajectories emerge for low-resolution (\ref{subfig:mnist_usps_lr}) and high-resolution (\ref{subfig:mnist_usps_hr}).
    }
    \label{fig:mnist_usps_interp}
\end{figure}
\newpage

\section{Ablation Studies}
\label{sec:ablation_exp}
\subsection{Ablation of Latent Source Metameasure}
\label{subsec:latent_pure_noise_ablation}
Based on \revise{the} MNIST generation experiment in Section~\ref{subsec:noise2mnist_64_experiment}, we investigate the importance of the latent noise metameasure. Since the barycentric latent employed in the experiment is closer to \revise{the} data distribution than raw noise, we repeat the experiment under the same conditions with our pure noise latent ($\underline{\sigma}=0.05$, $\overline{\sigma}=0.15$). The results are presented in Table~\ref{tab:mnist_mlp_local_pure_noise_std0.05-0.15_N64}. In comparison to the results based on the barycentric latent in Table~\ref{tab:mnist_barycenter_noise_results},
\revise{exact W and SW perform comparably or slightly worse, while the quality clearly declines for our lazy linear estimators.}
 This is expected by our theory, since the pure noise latent is not covered by Corollary~\ref{corr:lazy_linear_justification} because it is centred around the origin. \revise{In contrast, the independent coupling benefits from pure noise in few-step generation (five-step OT-NNA $0.82$ vs.\ $0.95$). Thus, the barycentric source mainly pays off in combination with OT-based couplings.}

\begin{table}[h]
\centering
    \setlength{\tabcolsep}{3.2pt}
    \renewcommand{\arraystretch}{0.9}
    \begin{tabular}{llccc|ccc|c}
    \toprule
     &  & \multicolumn{3}{c}{Chamfer-NNA $\downarrow$} & \multicolumn{3}{c}{OT-NNA $\downarrow$} & \textbf{Train} \\
    \textbf{OP} & \textbf{IP} & Euler-5 & Euler-25 & Euler-125 & Euler-5 & Euler-25 & Euler-125 &  \\
    \midrule
    $\widehat{\operatorname{OP}}_{\text{\textbf{ind}}}$ & $\widehat{\operatorname{IP}}_{\mathrm{ind}}$ & $0.84{\scriptstyle\pm0.01}$ & $0.67{\scriptstyle\pm0.02}$ & $0.65{\scriptstyle\pm0.01}$ & $0.82{\scriptstyle\pm0.01}$ & $0.62{\scriptstyle\pm0.01}$ & $0.61{\scriptstyle\pm0.02}$ & 2--3h \\
    $\widehat{\operatorname{OP}}_{\bW}$ & $\widehat{\operatorname{IP}}_{\mathrm{W}}$ & $\mathbf{0.72{\scriptstyle\pm0.02}}$ & $0.65{\scriptstyle\pm0.01}$ & $0.66{\scriptstyle\pm0.03}$ & $\mathbf{0.67{\scriptstyle\pm0.02}}$ & $\underline{0.60{\scriptstyle\pm0.01}}$ & $\underline{0.59{\scriptstyle\pm0.03}}$ & 5--6h \\
    $\widehat{\operatorname{OP}}_{\text{\textbf{SW}}}$ & $\widehat{\operatorname{IP}}_{\mathrm{SW}}$ & $\underline{0.78{\scriptstyle\pm0.02}}$ & $\underline{0.65{\scriptstyle\pm0.02}}$ & $\mathbf{0.64{\scriptstyle\pm0.01}}$ & $0.74{\scriptstyle\pm0.01}$ & $0.62{\scriptstyle\pm0.02}$ & $0.62{\scriptstyle\pm0.01}$ & 3--4h \\
    $\widehat{\operatorname{OP}}_{\text{\textbf{LLW}}}$ & $\widehat{\operatorname{IP}}_{\mathrm{LLW}}$ & $0.78{\scriptstyle\pm0.01}$ & $\mathbf{0.65{\scriptstyle\pm0.01}}$ & $\underline{0.64{\scriptstyle\pm0.01}}$ & $\underline{0.71{\scriptstyle\pm0.01}}$ & $\mathbf{0.59{\scriptstyle\pm0.02}}$ & $\mathbf{0.58{\scriptstyle\pm0.01}}$ & 3--4h \\
    \bottomrule
    \end{tabular}
\caption{MNIST generation quality and training time for pure noise.}
\label{tab:mnist_mlp_local_pure_noise_std0.05-0.15_N64}
\end{table}

\subsection{\revise{Locality of the Lazy Linear Approximation}}
\label{subsec:llw_locality}
\revise{Proposition~\ref{prop:lazy_linear_justification} is local, so we quantify how LLW degrades away from the reference. On ShapeNet-Airplane, we draw a source cloud from the barycenter plus fixed-scale Gaussian noise and compare the LLW assignment with exact $\widehat{\operatorname{IP}}_{\W}$ over 300 target clouds. This diagnostic deliberately sweeps a clean fixed $\sigma$. The training source instead uses the coordinate-wise uniform rescaling described in Section~\ref{subsec:main_text_implementation}, with base scale $0.3$.}

\begin{table}[h]
\centering
\color{black}
\setlength{\tabcolsep}{8pt}
\renewcommand{\arraystretch}{0.92}
\begin{tabular}{cccc}
\toprule
$\sigma$ & exact assignments (LLW / random) & $\mathrm{cost}_{\mathrm{LLW}}/\W_2^2$ & $\mathrm{cost}_{\mathrm{rand}}/\W_2^2$ \\
\midrule
$0.001$ & $0.93 / 0.002$ & $1.00$ & $8.7$ \\
$0.01$  & $0.57 / 0.002$ & $1.01$ & $8.7$ \\
$0.15$  & $0.02 / 0.002$ & $2.1$  & $6.7$ \\
$0.30$  & $0.01 / 0.002$ & $2.2$  & $3.7$ \\
\bottomrule
\end{tabular}
\caption{\revise{ShapeNet LLW locality diagnostic. Cost $1$ is optimal. The random assignment rate is approximately $1/512$.}}
\label{tab:llw_locality}
\end{table}

\revise{For $\sigma\leq0.01$, LLW is essentially exact, as predicted. Exact-permutation agreement then falls rapidly, but this strict statistic treats a nearby alternative partner as entirely wrong. At $\sigma=0.3$, LLW still has substantially lower realized cost than random matching. In the ShapeNet experiment LLW is used only for the outer cost and selected cloud pairs still use exact $\widehat{\operatorname{IP}}_{\W}$. Preserving pairwise cost structure and ranking therefore matters more than recovering the exact inner permutation. With the actual training source, LLW pairwise costs have Pearson correlation $0.70$ with exact $\W_2^2$, compared with $0.28$ for flattened Euclidean costs. On MNIST the corresponding correlations are $0.76$ and $0.31$. This supports interpreting LLW as a cheap geometry-aware cost estimator beyond the strictly exact local regime, not as exact OT at the operating noise scale.}

\subsection{Ablation of Batch Size}
In general, minibatch OT methods are heavily influenced by the batch size \citep{mousavi2025flow_semidiscre}. While a large batch size $B$ becomes computationally infeasible, we would still expect the straightness of trajectories to improve for larger batches. 
Therefore, we repeat the experiment from Section~\ref{subsec:circular_toy_experiment} \revise{for one epoch each} with $B=4, 8, 16$ for our $(\widehat{\operatorname{OP}}_{\bW},\,\widehat{\operatorname{IP}}_{\W})$ estimator\revise{, i.e., with the same number of training samples ($15{,}000$, $7{,}500$, and $3{,}750$ updates)}. Looking at Figure~\ref{fig:batch_ablation}, we clearly see that increasing the batch size improves straightness\revise{, even though larger batches receive fewer updates}.

\begin{figure}[h]
    \centering
    \begin{subfigure}[t]{0.315\textwidth}
        \centering
        \fbox{\includegraphics[width=\textwidth, trim=8mm 45mm 8mm 45mm, clip]{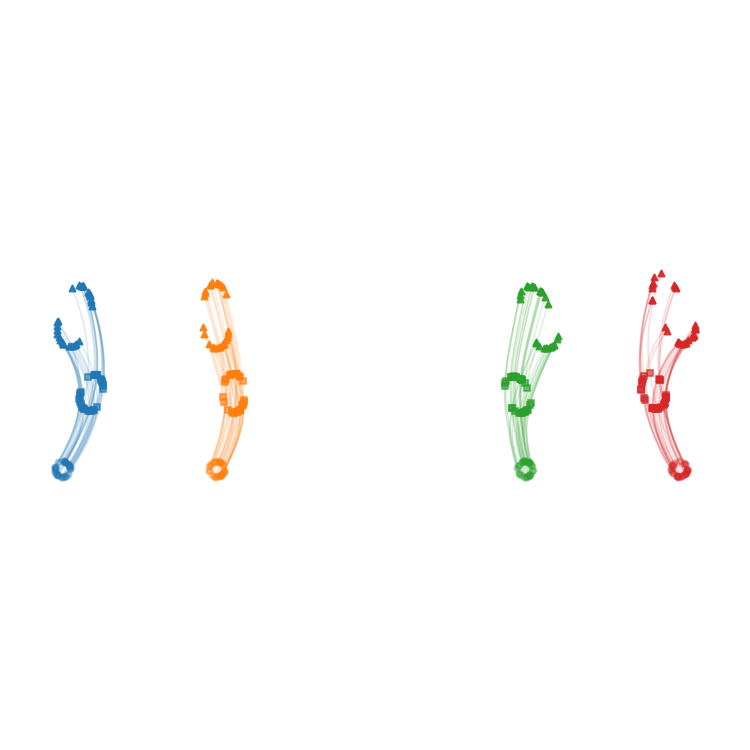}}
        \caption{$B=4$}
        \label{subfig:gmm_ot_ot_4}
    \end{subfigure}
    \hfill
    \begin{subfigure}[t]{0.315\textwidth}
        \centering
        \fbox{\includegraphics[width=\textwidth, trim=8mm 45mm 8mm 45mm, clip]{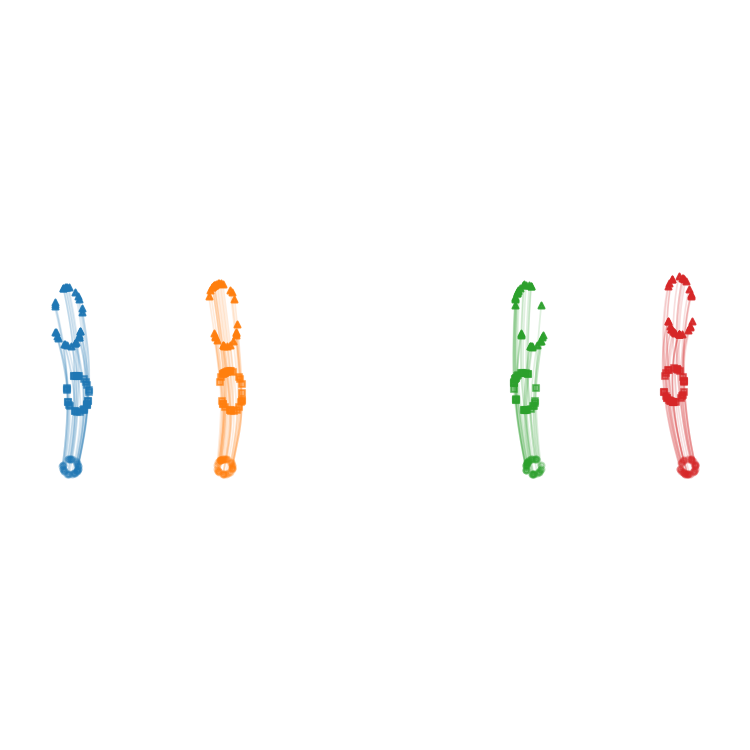}}
        \caption{$B=8$}
        \label{subfig:gmm_ot_ot_8}
    \end{subfigure}
    \hfill
    \begin{subfigure}[t]{0.315\textwidth}
        \centering
        \fbox{\includegraphics[width=\textwidth, trim=8mm 45mm 8mm 45mm, clip]{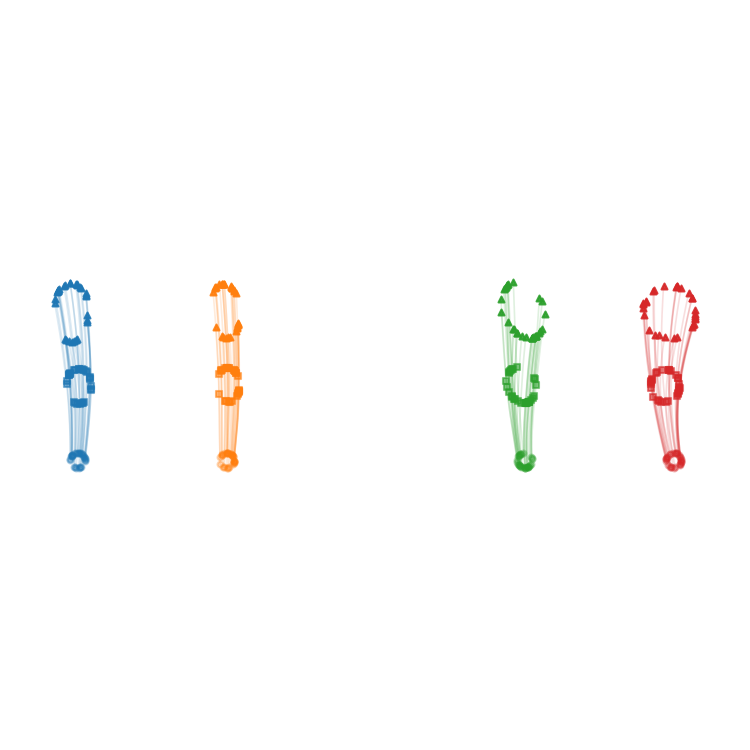}}
        \caption{$B=16$}
        \label{subfig:gmm_ot_ot_16}
    \end{subfigure}
    \caption{Learned paths for $(\widehat{\operatorname{OP}}_{\bW},\,\widehat{\operatorname{IP}}_{\W})$ and $B=4, 8, 16$. Larger batches lead to \revise{straighter} paths.}
    \label{fig:batch_ablation}
\end{figure}

\subsection{Ablation of Sliced Wasserstein Solver}
\label{subsec:sliced_ablation}
The quality of any sliced Wasserstein approximation is directly linked to the number of random projections $\sharp \operatorname{Slices}$ in
\eqref{eq:sw_estimator_mc}.
In order to study its impact, we repeat the experiment from Section~\ref{subsec:circular_toy_experiment} \revise{for $7{,}500$ steps}, but now with our sliced transport plan estimators $(\widehat{\operatorname{OP}}_{\bSW},\,\widehat{\operatorname{IP}}_{\SW})$,
and we vary \revise{$\sharp \operatorname{Slices}=1, 8, 32$}.
Figure \ref{fig:slices_ablation} clearly shows that more projections lead to \revise{straighter} paths.

\begin{figure}[h]
    \centering
    \begin{subfigure}[t]{0.315\textwidth}
        \centering
        \fbox{\includegraphics[width=\textwidth, trim=8mm 45mm 8mm 45mm, clip]{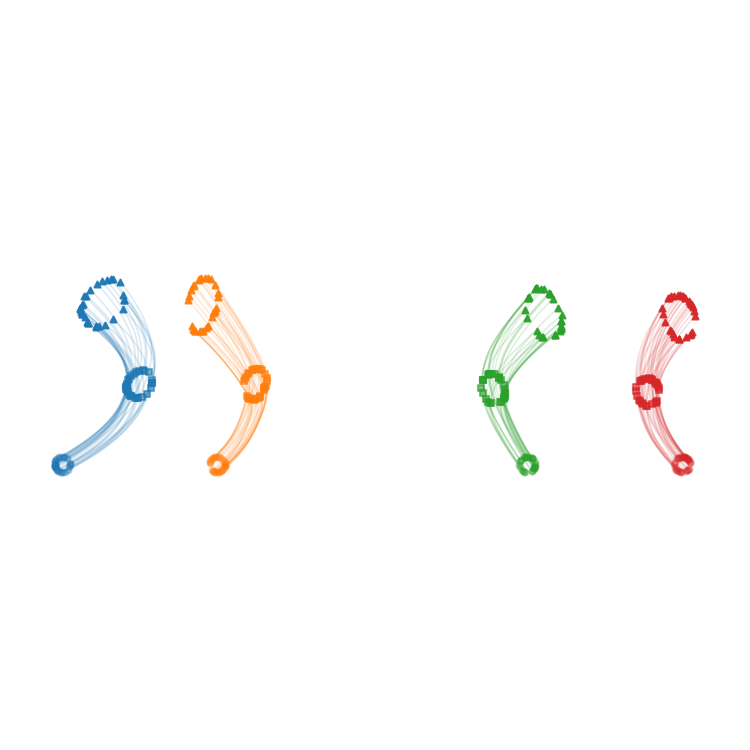}}
        \caption{\revise{$\sharp \operatorname{Slices}=1$}}
        \label{subfig:gmm_sw_sw_2}
    \end{subfigure}
    \hfill
    \begin{subfigure}[t]{0.315\textwidth}
        \centering
        \fbox{\includegraphics[width=\textwidth, trim=8mm 45mm 8mm 45mm, clip]{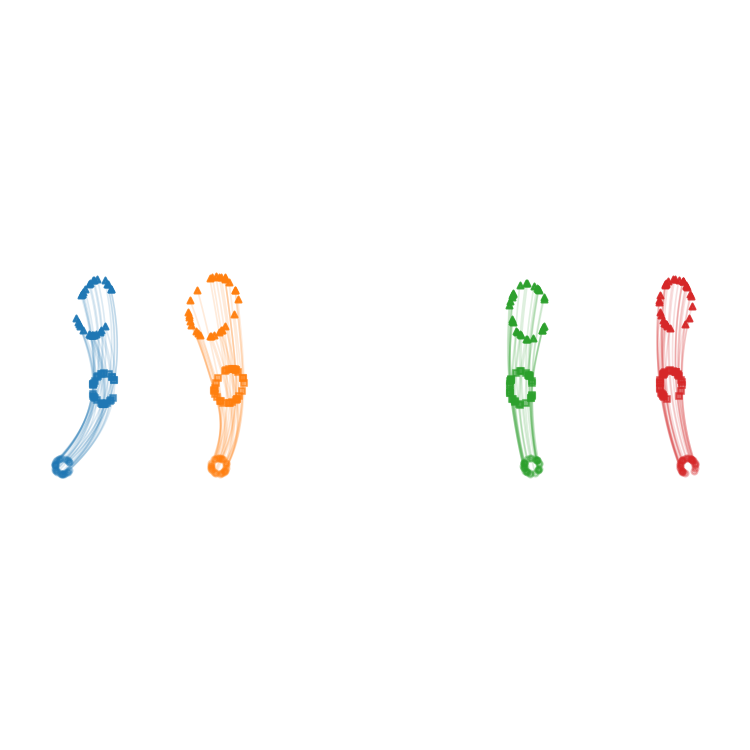}}
        \caption{$\sharp \operatorname{Slices}=8$}
        \label{subfig:gmm_sw_sw_8}
    \end{subfigure}
    \hfill
    \begin{subfigure}[t]{0.315\textwidth}
        \centering
        \fbox{\includegraphics[width=\textwidth, trim=8mm 45mm 8mm 45mm, clip]{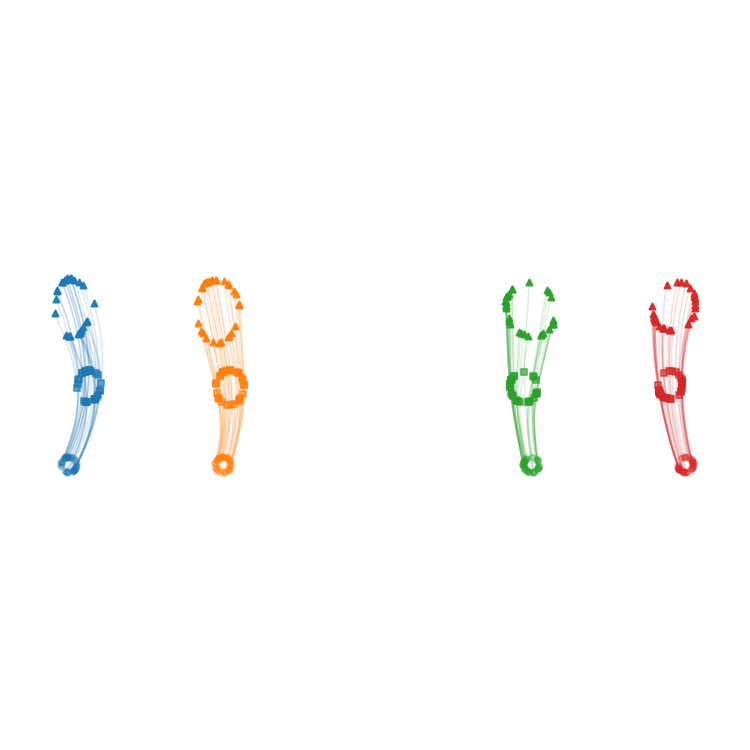}}
        \caption{$\sharp \operatorname{Slices}=32$}
        \label{subfig:gmm_sw_sw_32}
    \end{subfigure}
    \caption{Learned paths for $(\widehat{\operatorname{OP}}_{\bSW},\,\widehat{\operatorname{IP}}_{\SW})$ and \revise{$\sharp \operatorname{Slices}=1, 8, 32$}. More projections lead to \revise{straighter} paths.}
    \label{fig:slices_ablation}
\end{figure}

\subsection{Ablation of Wasserstein Solver}
\label{subsec:ablation_ot_solver}
Notably, we mainly relied on exact linear program solvers for our Wasserstein OT estimators due to the availability of highly \revise{efficient} implementations \citep{flamary2021pot}.
Especially for large-scale problems, such exact solvers are often replaced by Sinkhorn solvers
that solve an entropically regularized OT problem \citep{cuturi2013sinkhorn,PeyreCuturi2019}.
These solvers are governed by an entropic regularization parameter $\operatorname{reg}>0$, 
where $\operatorname{reg} \to 0$ recovers an exact OT plan and $\operatorname{reg} \to \infty$ simply recovers the independent coupling between the input measures.
To investigate the impact of replacing the exact OT solver with a \revise{Sinkhorn} solver, 
we repeat the experiment from Section~\ref{subsec:circular_toy_experiment} \revise{for $7{,}500$ steps} with $(\widehat{\operatorname{OP}}_{\bW},\,\widehat{\operatorname{IP}}_{\W})$, but replace the inner and outer \texttt{ot.emd} solvers from the Python Optimal Transport library \citep{flamary2021pot}, 
with the default \texttt{ot.sinkhorn} solvers \revise{on cost matrices normalized by their maximum} for $ \operatorname{reg}=0.01, 0.1, 1$.
For our computational setup, training takes around $\sim$3:50 minutes using the \texttt{ot.emd} solver and $ \operatorname{reg}=0.1, 1$ and increases to $\sim$5 minutes for $ \operatorname{reg}=0.01$. 
While we observe no runtime advantage for the Sinkhorn solver in comparison to the exact solver in this case, computational savings are expected for settings with more points, GPU acceleration or limited Sinkhorn iterations. The learned paths visualized in Figure~\ref{fig:sinkhorn_ablation} clearly show that lowering the entropy parameter $\operatorname{reg}$ reduces curvature.

\begin{figure}[h]
    \centering
    \begin{subfigure}[t]{0.315\textwidth}
        \centering
        \fbox{\includegraphics[width=\textwidth, trim=8mm 45mm 8mm 45mm, clip]{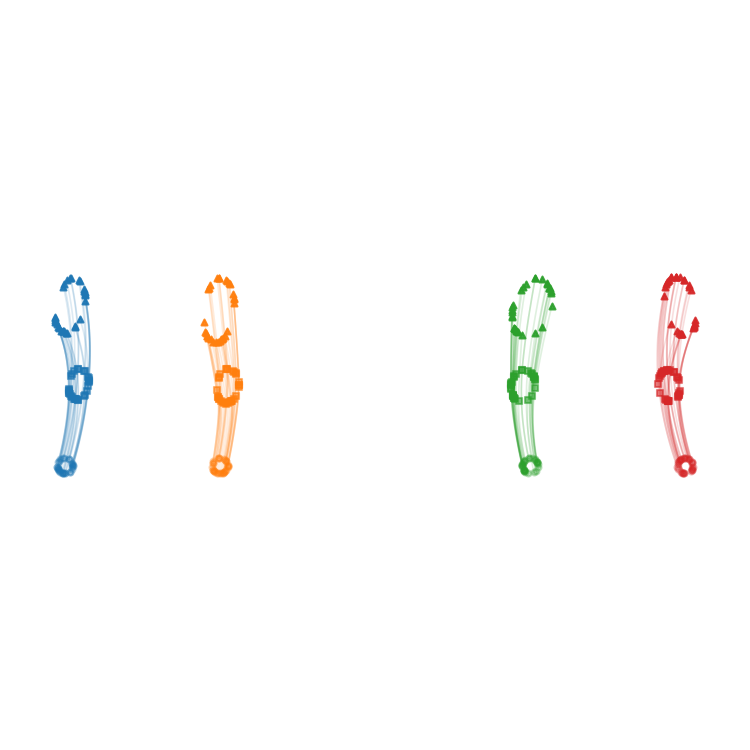}}
        \caption{$\operatorname{reg}=0.01$}
        \label{subfig:reg_.001}
    \end{subfigure}
    \hfill
    \begin{subfigure}[t]{0.315\textwidth}
        \centering
        \fbox{\includegraphics[width=\textwidth, trim=8mm 45mm 8mm 45mm, clip]{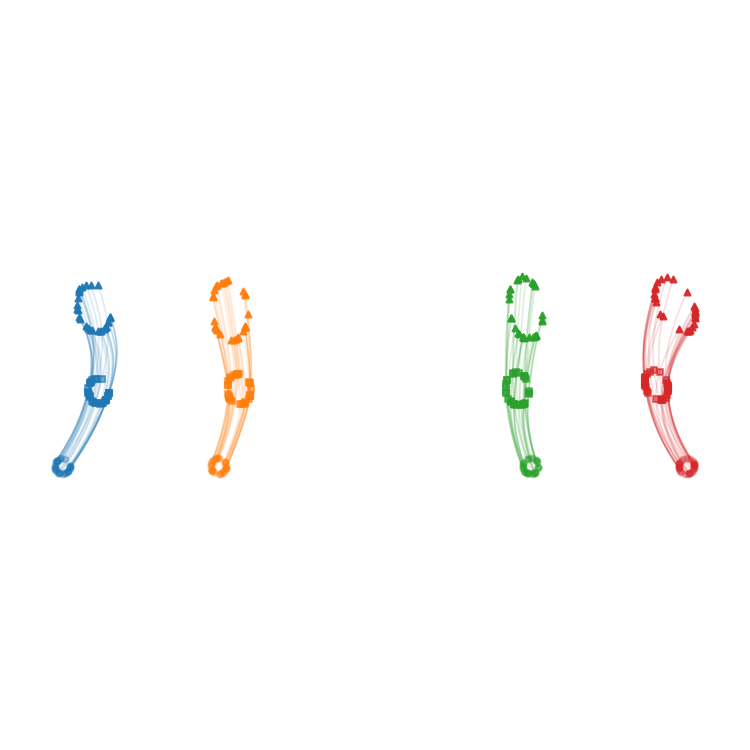}}
        \caption{$\operatorname{reg}=0.1$}
        \label{subfig:reg_.1}
    \end{subfigure}
    \hfill
    \begin{subfigure}[t]{0.315\textwidth}
        \centering
        \fbox{\includegraphics[width=\textwidth, trim=8mm 45mm 8mm 45mm, clip]{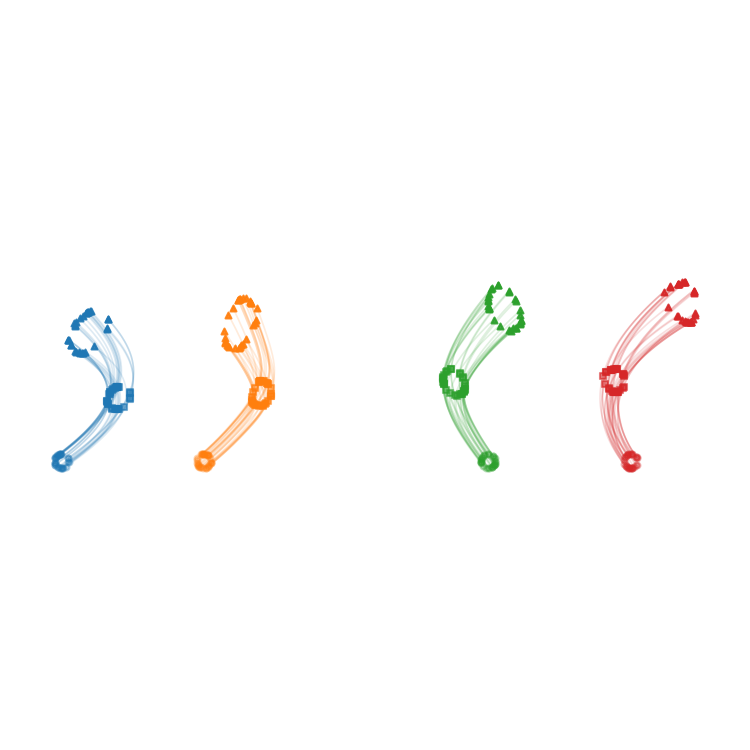}}
        \caption{$\operatorname{reg}=1$}
        \label{subfig:reg_1}
    \end{subfigure}
    \caption{Learned paths for $(\widehat{\operatorname{OP}}_{\bW},\,\widehat{\operatorname{IP}}_{\W})$ with Sinkhorn regularizer $ \operatorname{reg}=0.01, 0.1, 1$. Less regularization leads to \revise{straighter} paths, but increases training time.}
    \label{fig:sinkhorn_ablation}
\end{figure}

\subsection{Ablation of Transport Solver Runtime}                                         
  \label{subsec:ablation_runtime}                   
  To quantify the computational overhead of different transport solver combinations, we measure the wall-clock time required by the transport solver per training step for varying cloud size $N$ and batch size $B$.     
  Concretely, we simulate $B$ independent point clouds of dimension $d=2$ with points drawn uniformly at random from $[0,1]^2$ and solve the necessary transport problems.
  We report the mean CPU runtime (in milliseconds) and standard deviation over 5 runs.  
  Results are shown in Table~\ref{tab:timing_ot_combined}.
  
  As expected, the WoW approximation $(\widehat{\operatorname{OP}}_{\bW},\,\widehat{\operatorname{IP}}_{\W})$    becomes prohibitively expensive even at moderate $N$,  already exceeding $75$ seconds per step at $N=1024$, $B=32$.
  The mixed coupling $(\widehat{\operatorname{OP}}_{\mathrm{ind}},\,\widehat{\operatorname{IP}}_{\W})$  used in Wasserstein flow matching reduces this cost substantially but still scales with $N$,        reaching $\sim$2.3 seconds at $N=1024$, $B=32$.           
  Similarly, $(\widehat{\operatorname{OP}}_{\bSW},\,\widehat{\operatorname{IP}}_{\SW})$ with $\sharp\operatorname{Slices}=8$ projections (as used in our main experiments)      incurs a moderate and predictable overhead that grows more slowly with $N$ and $B$.                     
  In contrast, $(\widehat{\operatorname{OP}}_{\mathrm{LLW}},\,-)$ introduces only negligible overhead, as the Euclidean cost on preordered clouds based on the barycenter requires no inner resorting.
  
  Note that the use of the lazy linear strategy requires solving OT problems once during a preprocessing step. The computational overhead of this preprocessing step is in line with the computational cost of solving $(\widehat{\operatorname{OP}}_{\mathrm{ind}},\,\widehat{\operatorname{IP}}_{\W})$ once for the whole training dataset.
  
\begin{table}[h]
\resizebox{\linewidth}{!}{%
\centering
\setlength{\tabcolsep}{4pt}
\renewcommand{\arraystretch}{0.95}
\begin{tabular}{ll | ccc | ccc | ccc}
\toprule
\textbf{OP} & \textbf{IP} & \multicolumn{3}{c}{$B=8$} & \multicolumn{3}{c}{$B=16$} & \multicolumn{3}{c}{$B=32$} \\
 & & $N=64$ & $N=256$ & $N=1024$ & $N=64$ & $N=256$ & $N=1024$ & $N=64$ & $N=256$ & $N=1024$ \\
\midrule
$\widehat{\operatorname{OP}}_{\mathrm{ind}}$ & $\widehat{\operatorname{IP}}_{\mathrm{ind}}$ & ${<}0.1$ & ${<}0.1$ & ${<}0.1$ & ${<}0.1$ & ${<}0.1$ & ${<}0.1$ & ${<}0.1$ & ${<}0.1$ & ${<}0.1$ \\
$\widehat{\operatorname{OP}}_{\mathrm{SW}}$ & $\widehat{\operatorname{IP}}_{\mathrm{SW}}$ & $3.1_{\pm0.4}$ & $10.3_{\pm0.3}$ & $46.0_{\pm0.8}$ & $11.2_{\pm0.4}$ & $38.7_{\pm0.8}$ & $165.2_{\pm3.0}$ & $39.5_{\pm0.3}$ & $142.2_{\pm0.8}$ & $594.4_{\pm1.5}$ \\
$\widehat{\operatorname{OP}}_{\mathrm{W}}$ & $\widehat{\operatorname{IP}}_{\mathrm{W}}$ & $13.7_{\pm0.5}$ & $141.7_{\pm1.5}$ & $5128.0_{\pm31.1}$ & $51.9_{\pm3.2}$ & $565.9_{\pm3.6}$ & $19244.6_{\pm91.7}$ & $196.6_{\pm1.8}$ & $2264.8_{\pm20.0}$ & $75792.9_{\pm489.4}$ \\
$\widehat{\operatorname{OP}}_{\mathrm{LLW}}$ & $-$ & ${<}0.1$ & ${<}0.1$ & ${<}0.1$ & ${<}0.1$ & ${<}0.1$ & ${<}0.1$ & ${<}0.1$ & ${<}0.1$ & $1.3_{\pm0.3}$ \\
$\widehat{\operatorname{OP}}_{\mathrm{ind}}$ & $\widehat{\operatorname{IP}}_{\mathrm{W}}$ & $1.2_{\pm0.0}$ & $15.9_{\pm0.4}$ & $586.7_{\pm13.2}$ & $2.5_{\pm0.1}$ & $34.3_{\pm1.5}$ & $1158.0_{\pm13.4}$ & $5.1_{\pm0.4}$ & $68.6_{\pm0.6}$ & $2303.2_{\pm23.5}$ \\
\bottomrule
\end{tabular}
}
\caption{Wall-clock time per training step (mean\,$\pm$\ standard deviation over 5 runs, ms) for each (OP, IP) combination, varying batch size $B$ and cloud size $N$, CPU only. LLW outer uses Euclidean distance on pre-ordered clouds with no inner resorting~($-$).}
\label{tab:timing_ot_combined}
\end{table}

\newpage
\section{\revise{Related Work}}
\label{sec:related_work}

\revise{\paragraph{Wasserstein flow matching.}
WFM \citep{haviv2025wasserstein} introduced flow matching for families of distributions and, for empirical point clouds, uses an optimal inner map with independent outer pairing. Its multisample variant performs outer matching with Fr\'echet/Bures distances between Gaussian approximations of the clouds. Our framework contains both regimes: $(\widehat{\operatorname{OP}}_{\mathrm{ind}},\widehat{\operatorname{IP}}_{\W})$ recovers the former, while $(\widehat{\operatorname{OP}}_{\mathrm{BW}},\widehat{\operatorname{IP}}_{\W})$ mirrors the latter. The distinction is that Theorem~\ref{prop:wow_flow_matching} proves validity directly on WoW space for arbitrary admissible inner and outer plans, including empirical measures with singularities, and therefore also covers the sliced and lazy linear constructions.}

\revise{\paragraph{Not-so-optimal transport flows.}
NSOT \citep{hui2025notsooptimal_point_cloud_flow_matching} also admits an interpretation through our random-coupling framework. For each target surface measure, it precomputes a high-resolution coupling with a Gaussian source and then subsamples paired points. The resulting empirical plan has the sampled point clouds as its marginals, and its conditional expected transport cost equals the cost of the high-resolution plan. Its law therefore defines a measure over transport plans of the type considered here. In the population limit, the source metameasure is the Dirac measure at the Gaussian source, so the outer coupling is unique. The noiseless NSOT construction may consequently be viewed as an offline approximation of the corresponding optimal nested coupling. It does not solve a nontrivial outer matching problem, and finite-superset errors and the hybrid perturbation move it away from exact optimality. Unlike our online construction on observed finite measures, its superset approach is natural for resampled surfaces but does not directly apply to sets such as molecules.}

\revise{For context, we adapted NSOT to MNIST by upsampling each 64-point digit to a 1024-point superset, using its standard-Gaussian source and offline coupling, and disabling superset augmentation to isolate the coupling. With the recommended hybrid noise $\beta=0.2$, Chamfer-NNA at 5/25 steps is $0.813/0.694$ and OT-NNA is $0.717/0.626$. With $\beta=0$, the respective values are $0.722/0.654$ and $0.643/0.600$. Exact WoW, BW, LLW, and the noiseless NSOT adaptation are broadly comparable on this saturated MNIST setting, but use different sources and incur different online/offline costs. For this reason, NSOT is reported here as contextual evidence rather than inserted into the controlled same-source table.}

\newpage
\section{Implementation Details}
\label{sec:supp_implementation_details}
 \subsection{Computational Setup}
  All experiments are conducted on a single NVIDIA GeForce RTX~5090 (32\,GB VRAM) or comparable.
  All transport plan computations are performed on the
  CPU using the Python Optimal Transport library~\citep{flamary2021pot}, while model
  forward and backward passes run on the GPU.
  Transport problems are computed fresh at every gradient step. 
  We employ exact optimal transport solvers without entropic regularization, see Section~\ref{subsec:ablation_ot_solver}.
  MNIST, USPS and ShapeNet data are all normalized to $[0, 1]^d$.
  Data loading and coupling recomputation run in the main training process.
\subsection{\revise{Training Details of the Controlled Comparisons}}
\label{subsec:controlled_details}
\revise{\textbf{ShapeNet (Table~\ref{tab:shapenetpart_Airplane_pt_barycenter_noise_std0.3_N512}).} All outer plans share one configuration: transformer, $N=512$, $B=8$, $500$ epochs, Adam (step size $5\cdot10^{-4}$, cosine schedule), EMA with decay $0.999$ (EMA weights are used for sampling), barycentric noise with $\sigma_0=0.3$ and $U_{i,j}\sim\operatorname{Unif}[0.5,1.5]$, exact $\widehat{\operatorname{IP}}_{\W}$, and seed $7$. NNA compares $512$ generated with $512$ test clouds, both subsampled to $256$ points, over $5$ repeats.}

\revise{\textbf{QM9 (Table~\ref{tab:qm9_outer_comparison}).} We take the QM9 molecules provided by \texttt{torch\_geometric} with exactly nine heavy atoms in $\{\mathrm C,\mathrm N,\mathrm O,\mathrm F\}$ after removing hydrogens, and split them by a fixed list of held-out indices. Training uses a random subset of $20{,}000$ training molecules and evaluation draws from $12{,}000$ held-out molecules. Each atom is the centred 3D position concatenated with its one-hot type (unit weight, no coordinate rescaling), and atom order is shuffled. Both outer plans use the QM9 baseline network (Appendix~\ref{sec:supp_implementation_details}), $B=64$, $50$ epochs, Adam (step size $5\cdot10^{-4}$, cosine schedule), EMA with decay $0.999$, exact $\widehat{\operatorname{IP}}_{\W}$, barycentric noise with $\sigma_0=0.1$ and $U_{i,j}\equiv1$, and seed $42$. Before computing NNA, the type coordinates of generated atoms are projected to the nearest one-hot vector. NNA compares $512$ generated with $512$ held-out molecules over $3$ repeats.}

\revise{\textbf{MNIST baselines.} The BW, Euclidean, and NSOT runs (Table~\ref{tab:mnist_barycenter_noise_results}, Appendix~\ref{sec:related_work}) use the MNIST setting of Section~\ref{subsec:noise2mnist_64_experiment}: baseline network, $B=8$, $50$ epochs, Adam (step size $5\cdot10^{-4}$, constant), no EMA, and seed $7$. The $B=64$ runs train for $150$ epochs (${\approx}140$k steps) with otherwise identical settings.}
\subsection{Practical Procedure for Sampling from Discrete Plans}
\label{sec:supp_sampling_procedure}
In practice, the Monte Carlo sampling of matched pairs from the outer and inner plans proceeds as follows. \revise{Each source measure $\hat{\mu}_b$ draws its partner $\hat{\nu}_b$ by multinomial sampling from the $b$-th row of the outer plan matrix $\hat{\Pi}^{\operatorname{OP}}$. Likewise, each source particle $x_{b,j}$ draws its partner $x'_{b,j}$ from the $j$-th row of the inner plan matrix $\hat{\pi}^{\operatorname{IP}}_b$,} yielding the matrix-valued samples $(\mathbf{x}_b, \mathbf{x}'_b) \in \R^{d \times N} \times \R^{d \times N}$ used in the loss. \revise{For $\widehat{\operatorname{IP}}_{\mathrm{ind}}$, we draw a fresh uniformly random permutation at every step, i.e., we sample the product plan without replacement.}
  \subsection{Sliced Wasserstein Estimator}
  \label{sec:supp_transport_plan_computation_sliced}
In practice, the integral over the sphere in 
$
\SW_2^2(\mu,\nu)
=
\int_{\mathbb{S}^{d-1}}
\W_2^2\bigl(\pr^\theta_\#\mu,\pr^\theta_\#\nu\bigr)\,\d S(\theta),
$
is approximated via Monte Carlo sampling. In particular, let $\{\theta_\ell\}_{\ell=1}^{\sharp \operatorname{Slices}} \sim S \coloneqq \mathrm{Unif}(\mathbb{S}^{d-1})$, then
\begin{align}
\label{eq:sw_estimator_mc}
\SW^2_2(\mu,\nu)
\approx
\frac{1}{\sharp \operatorname{Slices}}\sum_{\ell=1}^{\sharp \operatorname{Slices}}
\W^2_2\bigl(\pr^{\theta_\ell}_\#\mu,\pr^{\theta_\ell}_\#\nu\bigr),
\qquad \pr^{\theta_\ell}(x)=\langle \theta_\ell,x\rangle.
\end{align}
  In all experiments, the Monte Carlo estimator in~\eqref{eq:sw_estimator_mc} uses
  $\sharp\operatorname{Slices} = 8$ uniformly sampled directions
  $\theta_\ell \sim \mathrm{Unif}(\mathbb{S}^{d-1})$, drawn fresh at each coupling
  recomputation step. We refer to Section~\ref{subsec:sliced_ablation} for an ablation study.

\subsection{Sliced Transport Plan Computation}
\label{sec:supp_transport_plan_computation_sliced_ip}
For the sliced inner transport plan ${\operatorname{IP}}_{\text{SW}}(\hat{\mu}, \hat{\nu}) = \int_{\mathbb{S}^{d-1}} \operatorname{IP}_\theta(\hat{\mu}, \hat{\nu})\,\d S(\theta)$, we employ the same Monte Carlo estimator as in \eqref{eq:sw_estimator_mc}. For each direction $\theta_\ell$, we project both measures onto $\R$ and compute the 1D optimal transport plan $\operatorname{IP}_{\theta_\ell}(\hat{\mu}, \hat{\nu}) \in {\rm c}(\hat{\mu}, \hat{\nu})$ via sorting of the projected support points, which lifts to a permutation on the original $d$-dimensional particles. Averaging across directions yields
\begin{equation}
    \widehat{\operatorname{IP}}_{\text{SW}}(\hat{\mu}, \hat{\nu}) \approx \frac{1}{\sharp\operatorname{Slices}} \sum_{\ell=1}^{\sharp\operatorname{Slices}} \operatorname{IP}_{\theta_\ell}(\hat{\mu}, \hat{\nu}),
\end{equation}
with $\sharp\operatorname{Slices} = 8$ as above. In practice, most 1D plans $ \operatorname{IP}_{\theta_\ell}(\hat{\mu}, \hat{\nu})$
are permutation matrices and our sliced plan is computed by averaging permutation matrices.
We refer to~\citep{liu2025expected} for further details.

  \subsection{Lazy Linear Inner Plan}                                       \label{sec:supp_transport_plan_computation_llw_ip}                         Unlike $\widehat{\operatorname{IP}}_{\mathrm{W}}$ and $\widehat{\operatorname{IP}}_{\mathrm{SW}}$,   which may yield fractional couplings during training,  the lazy linear plan $\widehat{\operatorname{IP}}_{\mathrm{LLW}}(\hat{\mu}, \hat{\nu})$ always
  produces a permutation matrix by construction: both clouds are aligned to a shared barycenter $\hat{\rho}$.
    By construction of our barycentric metameasure, the source is trivially aligned with the barycenter. Each target point cloud $\hat{\nu}$ is aligned via a precomputed permutation $\sigma_{\hat{\nu}}$.
  Therefore, we skip the inner sampling step and simply employ the precomputed point-to-point correspondence
  $(x_{b,i},\, x'_{b,\sigma_{\hat{\nu}}(i)})$\revise{, which is exactly the plan $\widehat{\operatorname{IP}}_{\mathrm{LLW}}$ from Section~\ref{sec:couplings}}.  

  \subsection{Outer Plan Computation}
\label{sec:supp_transport_plan_computation_outer}
Given a choice of pairwise divergence $D: \ProbTwo(\R^d) \times \ProbTwo(\R^d) \to \R_{\geq 0}$ (e.g., $\W_2^2$, $\SW_2^2$, or $\operatorname{LLW}_2^2$) and a batch of source and target measures $(\hat{\mu}_i)_{i=1}^B, (\hat{\nu}_{i'})_{i'=1}^B$, we compute the outer transport plan as follows. First, we assemble the $B \times B$ cost matrix
\begin{equation}
    C \in \R^{B \times B}, \qquad C_{i, i'} \coloneqq D(\hat{\mu}_i, \hat{\nu}_{i'}).
\end{equation}
We then solve the resulting discrete optimal transport problem
\begin{equation}
    \hat{\Pi}^{\operatorname{OP}} \in \argmin_{\Pi \in \R_{\geq 0}^{B \times B}} \sum_{i, i'=1}^{B} C_{i, i'}\, \Pi_{i, i'}
    \quad \text{s.t.} \quad
    \sum_{i'} \Pi_{i, i'} = \sum_i \Pi_{i, i'} = \tfrac{1}{B},
\end{equation}
using an exact linear program solver from the Python Optimal Transport library~\citep{flamary2021pot}. The resulting transport plan
$\widehat{\operatorname{OP}}_{\hat{\bmu}, \hat{\bnu}} = \sum_{i, i'=1}^{B} \hat{\Pi}^{\operatorname{OP}}_{i, i'}\, \delta_{(\hat{\mu}_i, \hat{\nu}_{i'})}$ 
is then used to sample matched pairs $(\hat{\mu}_b, \hat{\nu}_b) \sim \widehat{\operatorname{OP}}_{\hat{\bmu}, \hat{\bnu}}$ as described in Section~\ref{subsec:gen_w_fm}. Depending on the divergence $D$, this procedure yields $\widehat{\operatorname{OP}}_{\bW}$ (using $\W_2^2$), $\widehat{\operatorname{OP}}_{\text{SW}}$ (using $\SW_2^2$), or $\widehat{\operatorname{OP}}_{\text{\textbf{LLW}}}$ (using $\operatorname{LLW}_2^2$). \revise{The BW baseline uses Gaussian Bures--Wasserstein costs computed in double precision from symmetrized, jittered empirical covariance matrices. The Euclidean baseline uses mean squared distance between vectorized clouds without point reordering when constructing $C$.}

  \subsection{Network Architectures}

\paragraph{Baseline Model.}                   
  The baseline model is a per-point MLP augmented with local and global set
  information as described in Section~\ref{subsec:main_text_implementation},
  i.e., \revise{each $d$-dimensional point feature} is concatenated with (i)~\revise{sorted pairwise distances to all other points and the correspondingly sorted points} and (ii)~the empirical mean and upper triangle of the covariance matrix of the set. Notably, the importance of pairwise distances for point cloud recovery has been empirically demonstrated in \citep{piening2025novel,sato2020fast}. \revise{A sinusoidal time embedding of dimension $32$ is concatenated to these features.}
  The per-point features are processed by a
  multi-layer MLP, followed by one self-attention block, and a final linear output
  projection to \revise{$\mathbb{R}^d$. The architecture is shared across three experiments, with different capacities:}
  \begin{itemize}       
  \item \textbf{Circles (Section~\ref{subsec:circular_toy_experiment}):} $3$-layer MLP with hidden dimension $64$, followed by $1$ self-attention block
  with $4$ heads \revise{on the hidden dimension}. Total parameters: \revise{$\approx 71\,000$. Training uses EMA with decay $0.999$.}
    \item \textbf{MNIST generation (Section~\ref{subsec:noise2mnist_64_experiment}):} 
    \revise{$5$}-layer MLP with hidden dimension $256$, followed by $1$ self-attention block
  with $4$ heads \revise{on the hidden dimension}. Total parameters: $\approx 1{,}179\,000$.         
    \item \revise{\textbf{QM9 (Section~\ref{subsec:qm9_experiment}):} a $4$-layer MLP with hidden dimension $192$, followed by one self-attention block with $4$ heads on the hidden dimension. Total parameters: $\approx 623\,000$.}
    \end{itemize}

  \paragraph{Point Transformer\revise{.}}
  The standard transformer processes each point cloud as an unordered set of tokens.
  An input linear projection maps each point $x_i \in \mathbb{R}^d$ ($d \in \{2, 3\}$) to a token of
  dimension $128$. A sinusoidal time embedding of dimension $32$ is projected to $128$
  and broadcast-added to all tokens. The model consists of $4$ standard transformer
  blocks, each with $4$ attention heads, LayerNorm, and an MLP with expansion
  ratio~$4$. A final linear layer projects back to $\mathbb{R}^d$. No positional
  embedding is used (the model is permutation-equivariant by design).
  \revise{We employ the same architecture with $\approx 798\,000$ parameters for ShapeNet (Section~\ref{subsec:shapenet_experiment}) and MNIST$\to$USPS (Section~\ref{subsec:mnist2usps}).}

\subsection{Kernel Density Estimation for Visualization}
\label{subsec:kde_viz_supp}
To visualize 2D empirical point cloud distributions in Experiment~\ref{subsec:mnist2usps}, we estimate a continuous density
from each point cloud $\{x_i\}_{i=1}^N \subset [0,1]^2$ via Gaussian kernel density
estimation~\citep{parzen1962estimation_kde}, using the \texttt{scipy.stats.gaussian\_kde}
implementation. The bandwidth follows the $d=2$ variant of
Scott's rule~\citep{Scott2015}, scaled by $0.9$, giving $h = 0.9 \cdot N^{-1/6}$.
The density is evaluated on a $64\times 64$ grid over $[0,1]^2$. To prevent attenuation
near the boundary, the grid is extended by $15\%$ padding on each side and cropped
back to $[0,1]^2$ before rendering.     

\subsection{\revise{Practical Guide to Choosing Couplings}}
\label{subsec:param_choice}
\revise{
Exact WoW is preferable when the lowest prescribed transport energy and straightest interpolation justify its cost. Its $B^2$ inner OT solves limit scalability, particularly because minibatch OT generally benefits from large batches \citep{tong2023improving_minibatch}. Gaussian Bures--Wasserstein outer costs are inexpensive when first and second moments describe the clouds well, but discard non-Gaussian shape information.

Sliced Wasserstein is a generally applicable choice for low-dimensional, non-Gaussian clouds. It is a metric, supports arbitrary source and target metameasures, and is substantially cheaper than exact W in $d=2$ or $3$ \citep{nguyen2021pointcloud_sliced}. Its projection cost grows in higher dimension \citep{nadjahi2020statistical}.

Lazy linear Wasserstein is preferable when scalable training and accurate few-step generation are priorities and a barycentric source is appropriate. Its validity as a coupling does not depend on exactness, but its fidelity to W depends on source concentration: too little noise can make learning difficult, while too much breaks the local linearization. We therefore recommend comparing LLW and exact Wasserstein costs and rankings over several noise levels on a small subset before training, as in Appendix~\ref{subsec:llw_locality}.
}
\newpage
\section{Visualization of Additional Samples}

\begin{figure}[h]
    \centering
    \begin{subfigure}[t]{0.49\textwidth}
        \centering
        \fbox{\includegraphics[width=\textwidth]{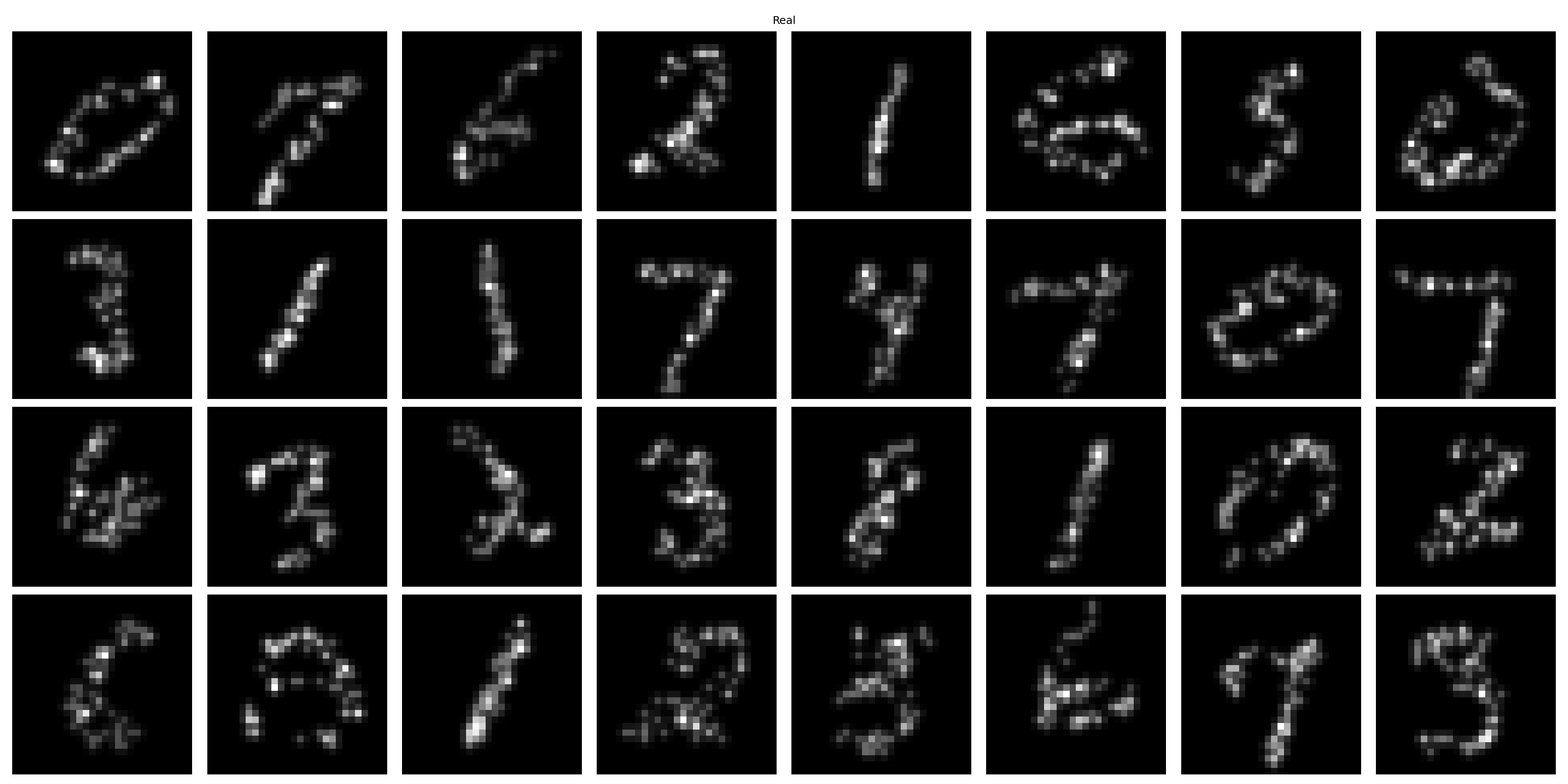}}
        \caption{True MNIST digits.}
        \label{subfig:mnis_true}
    \end{subfigure}
    \hfill
    \begin{subfigure}[t]{0.49\textwidth}
        \centering
        \fbox{\includegraphics[width=\textwidth]{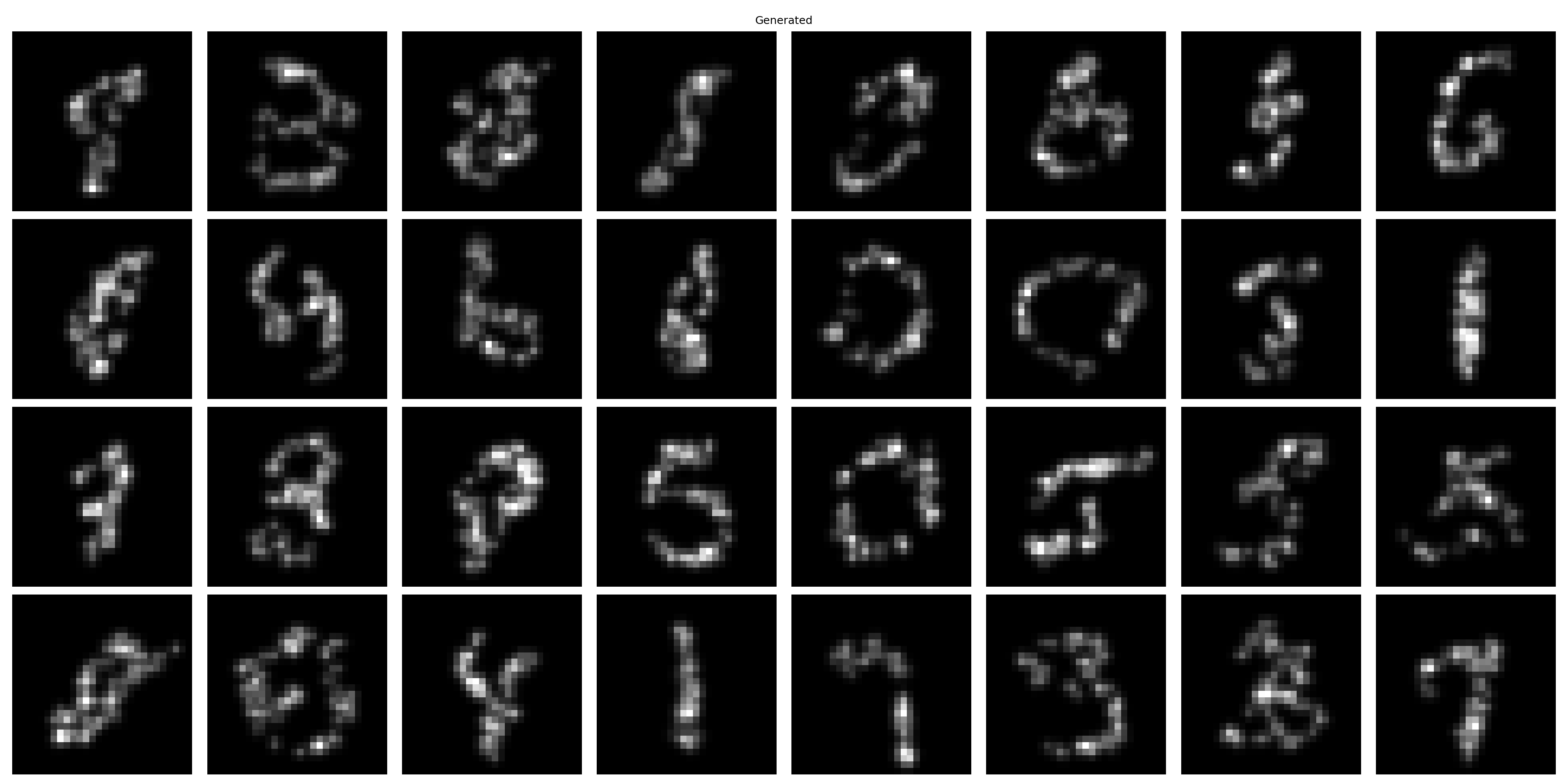}}
        \caption{Generated MNIST digits.}
        \label{subfig:mnist_gen}
    \end{subfigure}
    \caption{2D histogram visualization based on true and generated MNIST point clouds ($N=64$) based on the lazy linear flow learned via $(\widehat{\operatorname{OP}}_{\mathbf{LLW}},\,\widehat{\operatorname{IP}}_{\text{LLW}})$ for the barycentric noise$\to$MNIST generation experiment in Section~\ref{subsec:noise2mnist_64_experiment}. Pixel intensities correspond to number of aggregated point counts.}
    \label{fig:llw_mnist_supp_viz_histograms}
\end{figure}
\begin{figure}[h]
    \centering
\includegraphics[width=.2\textwidth]{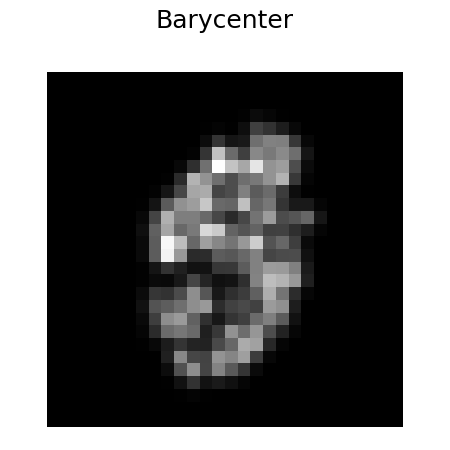}
    \caption{2D histogram visualization of the computed MNIST barycenter employed as the center of the barycentric noise for the experiment in Section~\ref{subsec:noise2mnist_64_experiment}\revise{.}}
    \label{fig:mnist_bary}
\end{figure}
\begin{figure}[h]
    \centering
    \includegraphics[width=0.95\linewidth]{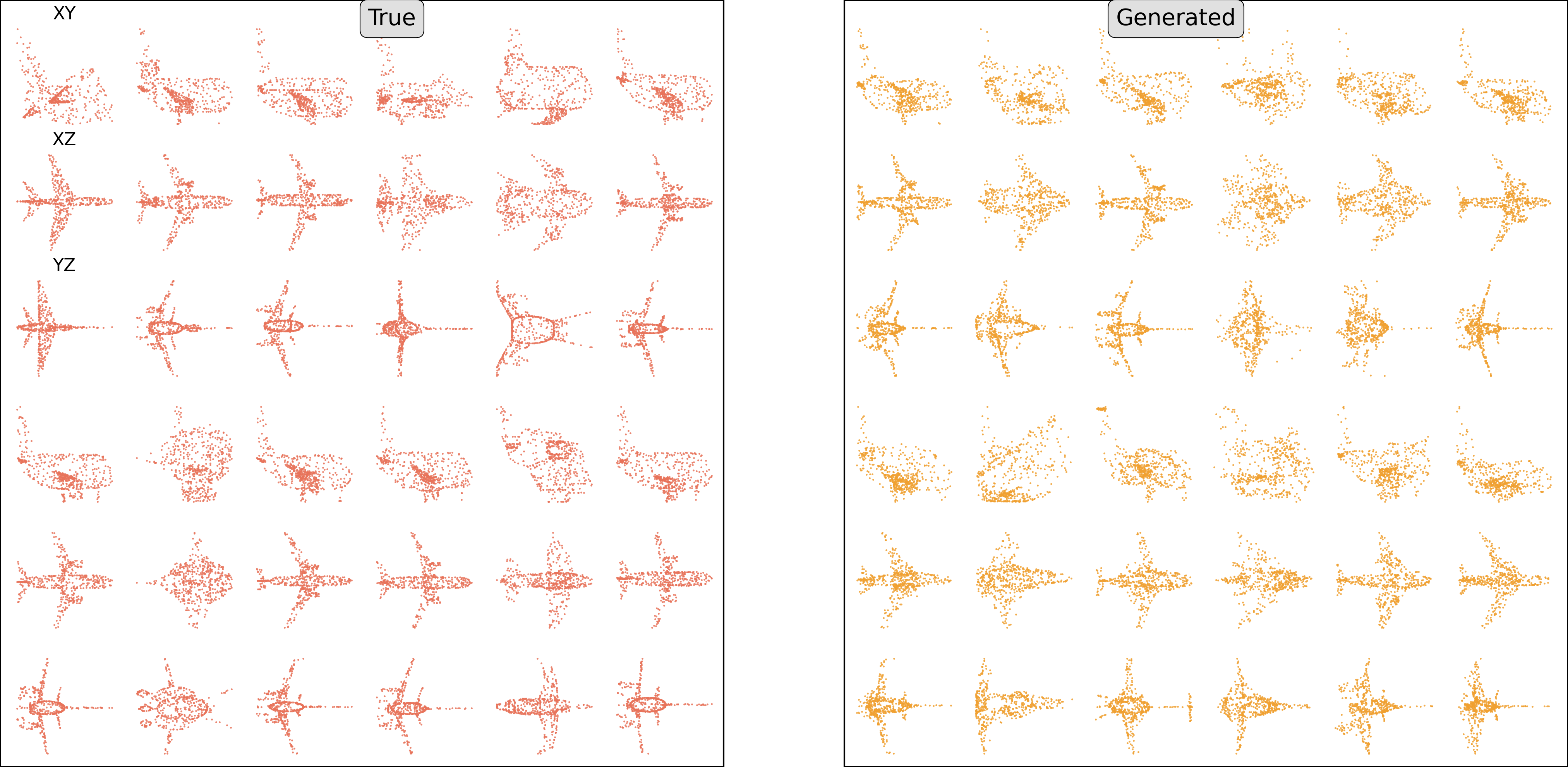}
    \caption{True and generated ShapeNet airplanes for \revise{discretization} $N=512$ for the experiment in Section~\ref{subsec:shapenet_experiment}, where we show generated samples for the model trained with $\widehat{\operatorname{OP}}_{\bSW}$. We display the point cloud as seen from different viewpoints.}
\end{figure}
\begin{figure}[h]
    \centering
    \begin{subfigure}[t]{0.315\textwidth}
        \centering
        \fbox{\includegraphics[width=\textwidth]{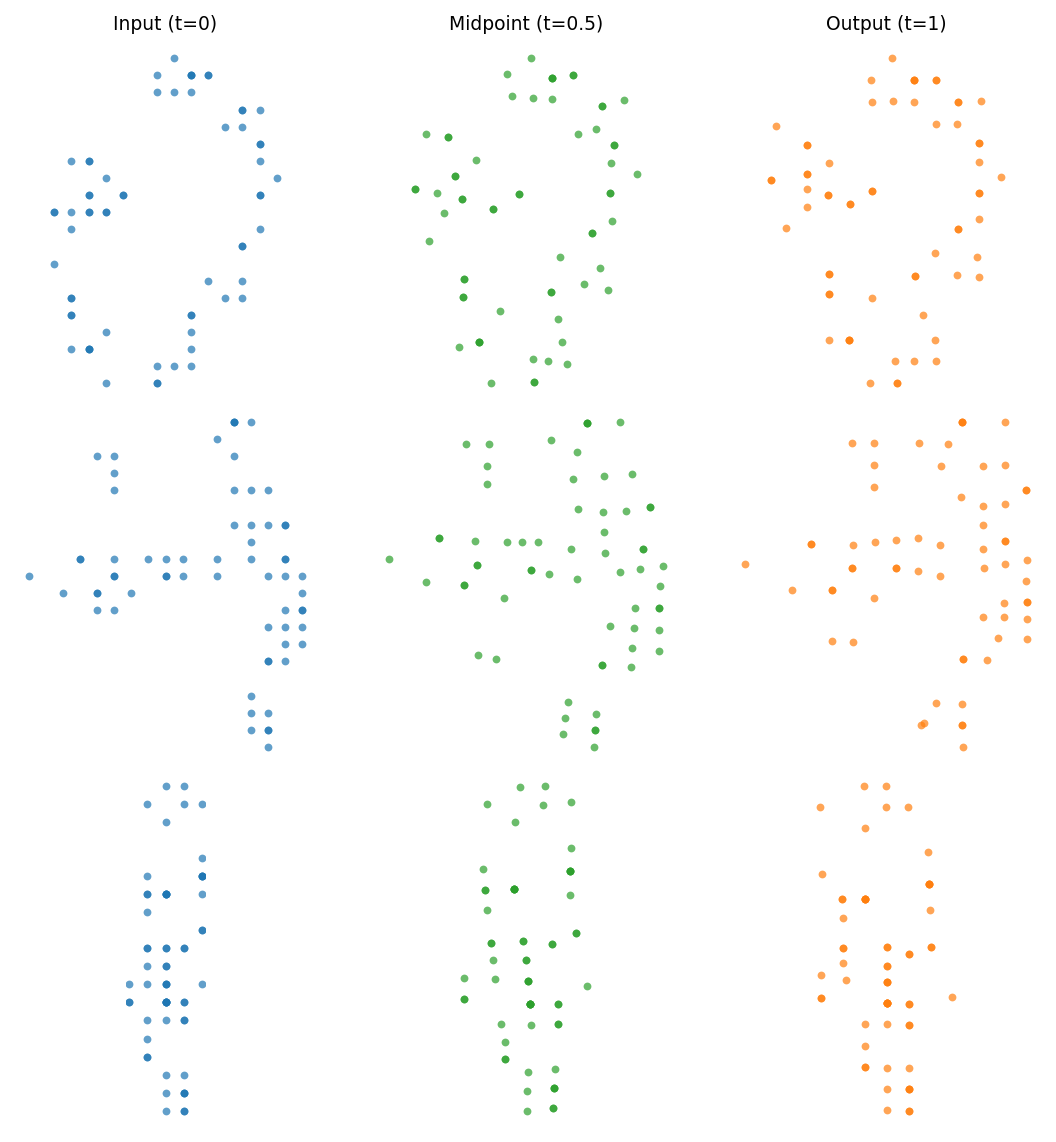}}
        \caption{$N=64$}
        \label{subfig:mnist2usps_supp_64}
    \end{subfigure}
    \hfill
    \begin{subfigure}[t]{0.315\textwidth}
        \centering
        \fbox{\includegraphics[width=\textwidth]{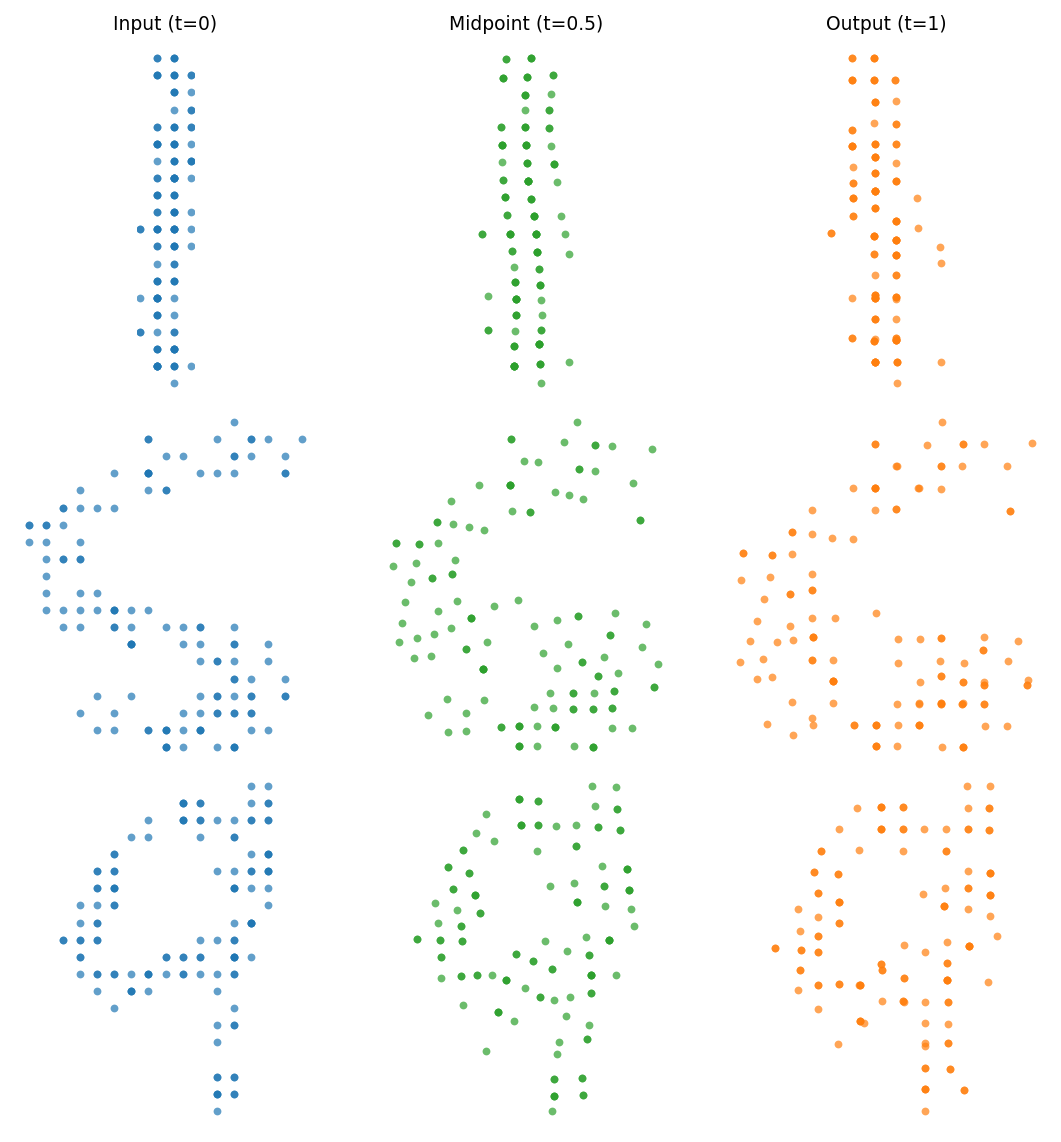}}
        \caption{$N=128$}
        \label{subfig:mnist2usps_supp_128}
    \end{subfigure}
    \hfill
    \begin{subfigure}[t]{0.315\textwidth}
        \centering
        \fbox{\includegraphics[width=\textwidth]{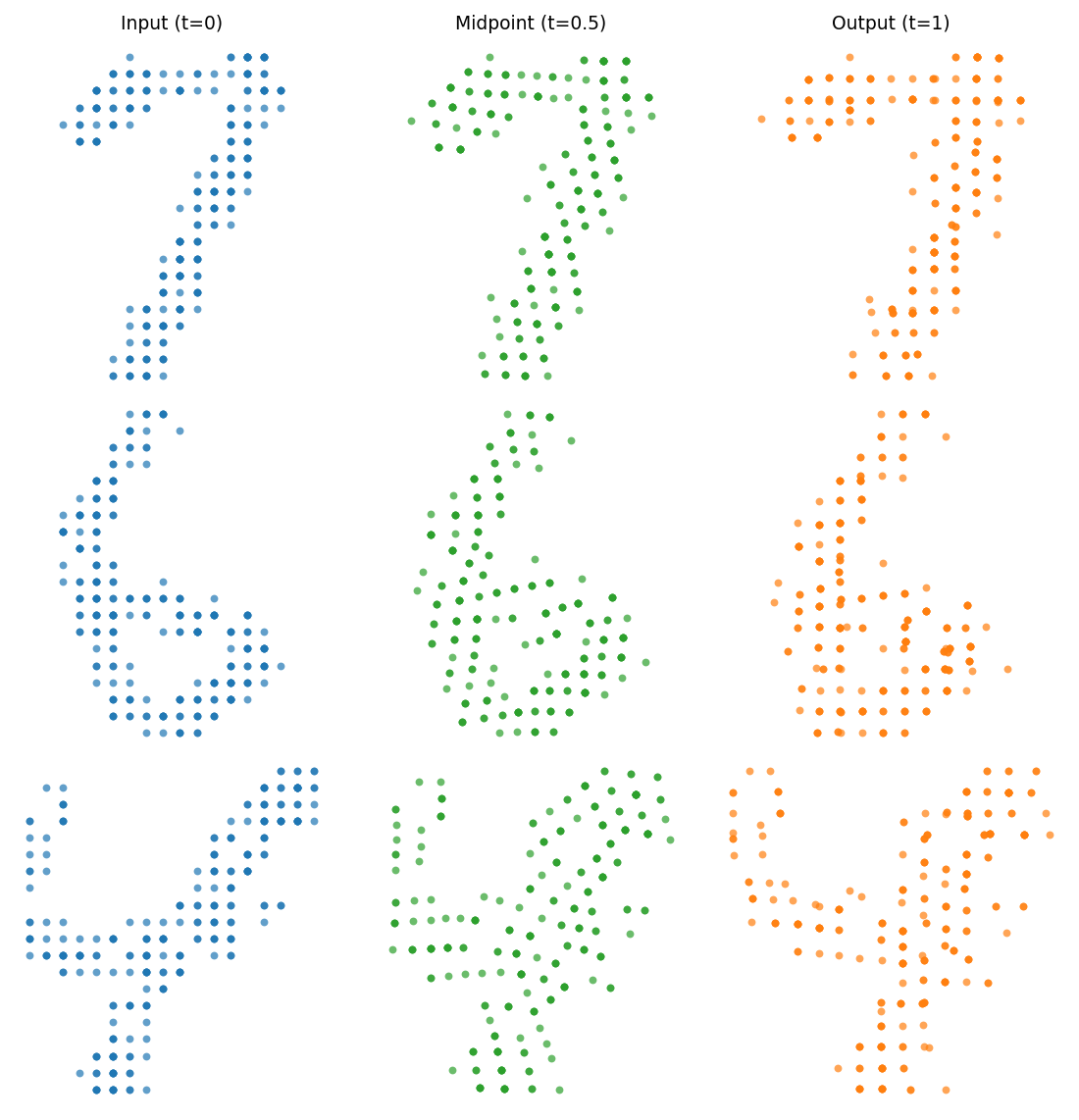}}
        \caption{$N=256$}
        \label{subfig:mnist2usps_supp_256}
    \end{subfigure}
    \caption{2D scatter plot visualization of the point cloud flow learned using the WoW setup $(\widehat{\operatorname{OP}}_{\bW},\,\widehat{\operatorname{IP}}_{\W})$ for $N=64, 128, 256$ for the MNIST$\to$USPS experiment in Section~\ref{subsec:mnist2usps}. Left is starting point ($t=0$, blue), middle is midpoint ($t=0.5$, green), and right is endpoint ($t=1$, orange).}
    \label{fig:wow_usps2mnist_supp_viz}
\end{figure}
\begin{figure}[h]
    \centering
    \begin{subfigure}[t]{0.49\textwidth}
        \centering
        \fbox{\includegraphics[width=\textwidth]{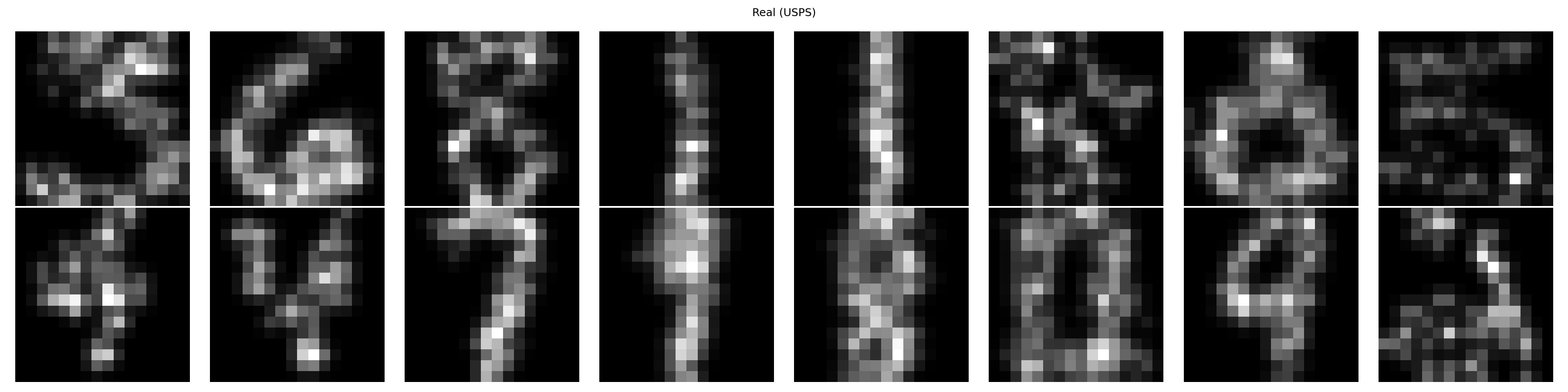}}
        \caption{True USPS digits.}
        \label{subfig:mnist2usps_true}
    \end{subfigure}
    \hfill
    \begin{subfigure}[t]{0.49\textwidth}
        \centering
        \fbox{\includegraphics[width=\textwidth]{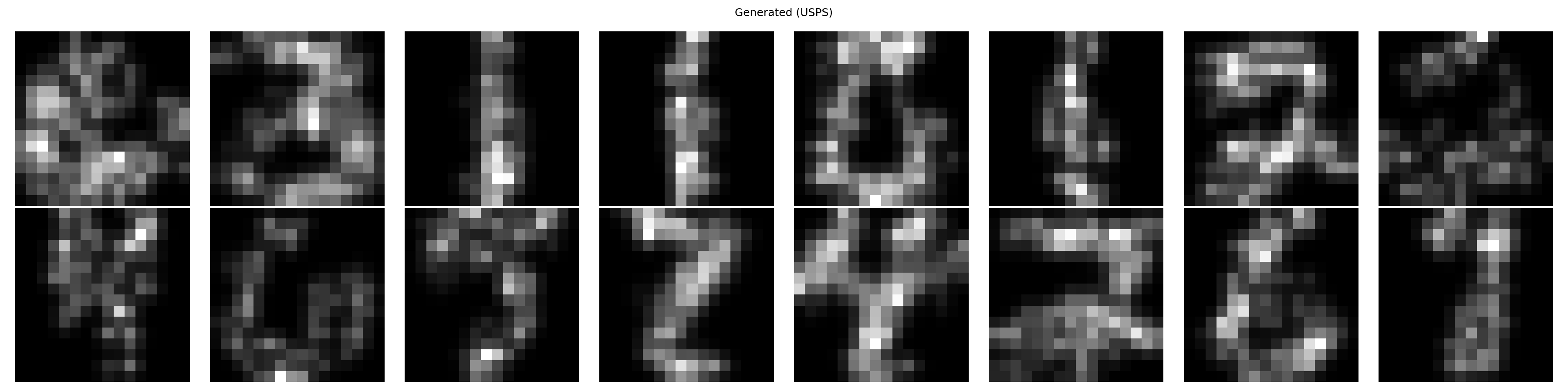}}
        \caption{Generated USPS digits.}
        \label{subfig:mnist2usps_gen}
    \end{subfigure}
    \caption{2D histogram visualization based on true and generated USPS point clouds for $N=256$ based on the WoW flow learned via $(\widehat{\operatorname{OP}}_{\bW},\,\widehat{\operatorname{IP}}_{\W})$ for the MNIST$\to$USPS experiment in Section~\ref{subsec:mnist2usps}. Pixel intensities correspond to number of aggregated point counts.}
    \label{fig:wow_usps2mnist_supp_viz_histograms}
\end{figure}
\begin{figure}[h]
\includegraphics[width=\textwidth]{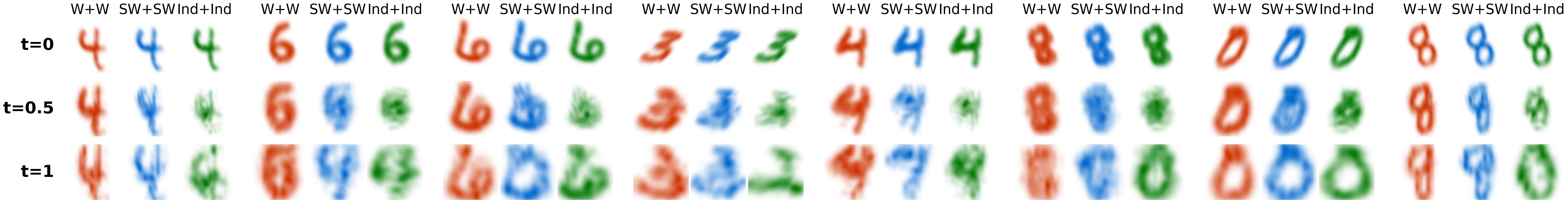}
    \caption{Same visualization as in  Figure~\ref{fig:mnist_usps_interp} with more samples for $N=4096$.
    }
    \label{fig:mnist_usps_interp_supp}
\end{figure}

\clearpage

\end{document}